\documentclass[nohyperref]{article}

\usepackage{microtype}
\usepackage{graphicx}
\usepackage{subfigure}
\usepackage{booktabs} 

\usepackage{hyperref}

\usepackage[accepted]{icml2022}

\usepackage{amsmath}
\usepackage{amssymb}
\usepackage{mathtools}
\usepackage{amsthm,fancybox}

\usepackage[capitalize,noabbrev]{cleveref}

\usepackage[textsize=tiny]{todonotes}

\newtheorem{theorem}{Theorem}
\newtheorem{cor}{Corollary}

\newtheorem{lemma}{Lemma}
\newtheorem{ass}{Assumption}
\newtheorem{definition}[cor]{Definition}
\usepackage{multirow}
\usepackage[T1]{fontenc}
\newcommand{\Norm}[1]{\left\|#1\right\|}
\usepackage{paralist}

\def \E {\mathbb{E}}
\def \R {\mathbb{R}}

\def \d {\mathbf{d}}

\def \h {\mathbf{h}}

\def \u {\mathbf{u}}
\def \v {\mathbf{v}}
\def \w {\mathbf{w}}
\def \x {\mathbf{x}}
\def \y {\mathbf{y}}
\def \A {\mathcal{A}}

\def \LL {\mathcal{L}}

\icmltitlerunning{Optimal Algorithms for Stochastic Multi-Level Compositional Optimization}

\begin{document}
\twocolumn[
\icmltitle{Optimal Algorithms for Stochastic Multi-Level Compositional Optimization}




\begin{icmlauthorlist}
\icmlauthor{Wei Jiang}{yyy}
\icmlauthor{Bokun Wang }{xxx}
\icmlauthor{Yibo Wang}{yyy}
\icmlauthor{Lijun Zhang}{yyy}
\icmlauthor{Tianbao Yang}{xxx}
\end{icmlauthorlist}
\icmlaffiliation{yyy}{National Key Laboratory for Novel Software Technology, Nanjing University, Nanjing, China}
\icmlaffiliation{xxx}{Department of
Computer Science, The University of Iowa, Iowa City, USA}
\icmlcorrespondingauthor{Lijun Zhang}{zhanglj@lamda.nju.edu.cn}
\icmlcorrespondingauthor{Tianbao Yang}{tianbao-yang@uiowa.edu}

\icmlkeywords{Machine Learning, ICML}

\vskip 0.3in
]



\printAffiliationsAndNotice{}  

\begin{abstract}
In this paper, we investigate the problem of stochastic multi-level compositional optimization, where the objective function is a composition of multiple smooth but possibly non-convex functions. Existing methods for solving this problem either suffer from sub-optimal sample complexities or need a huge batch size. To address these limitations, we propose a Stochastic Multi-level Variance Reduction method (SMVR), which achieves the optimal sample complexity of $\mathcal{O}\left(1 / \epsilon^{3}\right)$ to find an $\epsilon$-stationary point for non-convex objectives. Furthermore, when the objective function satisfies the convexity or Polyak-Łojasiewicz (PL) condition, we propose a stage-wise variant of SMVR and improve the sample complexity to $\mathcal{O}\left(1 / \epsilon^{2}\right)$ for convex functions or $\mathcal{O}\left(1 /\left(\mu\epsilon\right)\right)$ for non-convex functions satisfying the $\mu$-PL condition. The latter result implies the same complexity for $\mu$-strongly convex functions. To make use of adaptive learning rates, we also develop Adaptive SMVR, which achieves the same complexities but converges faster in practice. All our complexities match the lower bounds not only in terms of $\epsilon$ but also in terms of $\mu$ (for PL or strongly convex functions), without using a large batch size in each iteration.
\end{abstract}

\section{Introduction}
We consider the stochastic multi-level compositional optimization problem:
\begin{equation} \label{prob:1}
	\min_{\w\in \R^d} F(\w) = f_K\circ f_{K-1} \circ \cdots \circ f_1(\w),
\end{equation}
where $f_i:\R^{d_{i-1}}\mapsto \R^{d_i}$, $i=1,\dotsc,K$ (with $d_K = 1$ and $d_0 = d$). Only noisy evaluations of each layer function $f_i(\cdot;\xi)$ and its gradient $\nabla f_i(\cdot;\xi)$ can be accessed, where $\xi$ denotes a sample drawn from the oracle such that:
\begin{align*}
    \E_{\xi} \left[f_i(\cdot;\xi)\right] = f_i(\cdot), \quad \E_{\xi}\left[\nabla f_i(\cdot;\xi)\right] = \nabla f_i(\cdot).
\end{align*}
In machine learning, $\w$ often represents the parameter of a predictive model, and $F$ denotes the loss of that model, with $\xi$ representing a training sample (or a batch of samples). Problem~(\ref{prob:1}) finds its application in many tasks, such as reinforcement learning \citep{Dann2014PolicyEW}, robust learning~\cite{li2021tilted},  multi-step model-agnostic meta-learning~\citep{Ji2020MultiStepMM}, risk-averse portfolio optimization  \citep{Bruno2016RiskNA,Shapiro2009LecturesOS} and risk management~\citep{Cole2013HowDR, Dentcheva2015StatisticalEO}. 

Our goal is to solve problem~(\ref{prob:1}) with the optimal sample complexity, which is a commonly-used measure in stochastic optimization. Sample complexity characterizes the number of samples needed to find an $\epsilon$-stationary point for  non-convex functions, i.e., $\Norm{\nabla F(\w)} \leq \epsilon$\footnote{In some literature, the measure $\Norm{\nabla F(\w)}^{2} \leq \epsilon$ is used instead. Note that the complexity $\mathcal{O}\left(1 / \epsilon^{\alpha}\right)$ for $\Norm{\nabla F(\w)}^{2} \leq \epsilon$ implies a complexity $\mathcal{O}\left(1 / \epsilon^{2\alpha}\right)$ for $\Norm{\nabla F(\w)} \leq \epsilon$  \citep{Zhang2021MultiLevelCS}. For fair comparison, we would align the complexity measure of each method under the same criteria $\Norm{\nabla F(\w)} \leq \epsilon$, when discussing related algorithms.}, or an $\epsilon$-optimal point for convex functions, i.e., $F(\w)-\inf_{\w} F(\w) \leq \epsilon$. Problem~(\ref{prob:1}) reduces to the classic one-level stochastic optimization problem when $K=1$, and it is also known as the two-level compositional optimization for $K=2$.

\begin{table*}[t]
\caption{Summary of results for finding an $\epsilon$-stationary or $\epsilon$-optimal point. Here, CVX means convex, Mono. \& CVX means each layer function is monotone and convex, SC means $\mu$-strongly convex and PL means the $\mu$-PL condition (weaker than $\mu$-strongly convex).}
\label{sample}
\vskip 0.15in
\begin{center}
\begin{small}
\begin{sc}
\resizebox{0.95\textwidth}{!}{
\begin{tabular}{lccr}
\toprule
Method & Assumptions &   Complexity & Batch size  \\
\midrule
A-TSCGD \citep{Yang2019MultilevelSG}    & Smooth& $\mathcal{O}\left(1 / \epsilon^{(7+K)/2}\right)$ &  $\mathcal{O}\left(1 \right)$ \\
A-TSCGD \citep{Yang2019MultilevelSG}    & Smooth + SC & $\mathcal{O}\left(1 / \epsilon^{(3+K)/4}\right)$ &  $\mathcal{O}\left(1 \right)$ \\
NLASG \citep{balasubramanian2020stochastic} & Smooth&   $\mathcal{O}\left(1 / \epsilon^{4}\right)$& $\mathcal{O}\left(1 \right)$ \\
SCSC \citep{chen2021solving} & Smooth&  $\mathcal{O}\left(1 / \epsilon^{4}\right)$&$\mathcal{O}\left(1 \right) $ \\
Nested-SPIDER  \citep{Zhang2021MultiLevelCS}  & Smooth&   $\mathcal{O}\left(1 / \epsilon^{3}\right)$ & $\mathcal{O}\left(1 / \epsilon\right)$ \\
SSD \citep{Zhang2020OptimalAF}     & Smooth + Mono. \& CVX &  $\mathcal{O}\left(1 / \epsilon^{2}\right)$& $\mathcal{O}\left(1 \right)$  \\
SSD \citep{Zhang2020OptimalAF} & Smooth + SC &  $\mathcal{O}\left(1 / (\mu^2\epsilon) \right)$& $\mathcal{O}\left(1 \right)$ \\
\midrule
\textbf{SMVR (This work)}   & Smooth &  $\mathcal{O}\left(1 / \epsilon^{3}\right)$& $\mathcal{O}\left(1 \right)$ \\
\textbf{Stage-wise SMVR (This work)}   & Smooth + CVX &  $\mathcal{O}\left(1 / \epsilon^{2}\right)$& $\mathcal{O}\left(1 \right)$  \\
\textbf{Stage-wise SMVR (This work)}  & Smooth + PL &  $\mathcal{O}\left(1 / (\mu\epsilon) \right)$& $\mathcal{O}\left(1 \right)$ \\
\bottomrule
\end{tabular}}
\end{sc}
\end{small}
\end{center}
\vskip -0.1in
\end{table*}

For one-level and two-level non-convex problems, there exist single-loop algorithms, such as STORM \citep{Cutkosky2019MomentumBasedVR} and RECOVER \citep{qi2021online}, that can achieve the optimal $\mathcal{O}\left(1 / \epsilon^{3}\right)$ sample complexity for finding an $\epsilon$-stationary solution without using large batches. However, for multi-level problems, the errors of Jacobian and function value estimators accumulate with the level becoming deeper, making the problem much harder. Existing multi-level methods either suffer from sub-optimal complexities~\citep{Yang2019MultilevelSG, balasubramanian2020stochastic, chen2021solving} or require a huge and increasing batch size~\citep{Zhang2021MultiLevelCS}. When the objective function is convex or strongly convex, \citet{Zhang2020OptimalAF} prove a sample complexity of $\mathcal{O}\left(1 / \epsilon^{2}\right)$ or $\mathcal{O}(1/(\mu^2\epsilon))$. However, their analysis requires that each layer function $f_i$ is monotone and convex, and their complexity for $\mu$-strongly convex function is non-optimal in terms of $\mu$~\cite{Agarwal2012InformationTheoreticLB}. 

Hence, a fundamental question to be addressed is:

\shadowbox{\begin{minipage}[t]{0.95\columnwidth}%
\it Can we solve stochastic multi-level compositional problems with optimal complexities for non-convex, convex and strongly convex functions without using a large batch size?
\end{minipage}}

We give an affirmative answer to this question by proposing an optimal algorithm named Stochastic Multi-level Variance Reduction (SMVR). By using the variance reduction technique to estimate the Jacobians and function values in each level, our algorithm achieves the optimal sample complexity of $\mathcal{O}\left(1 / \epsilon^{3}\right)$ for non-convex functions~\citep{Arjevani2019LowerBF}. Central to the algorithmic design and analysis are: (i) the variance reduction is applied to both Jacobians and function values, which is different from most existing works~\cite{Yang2019MultilevelSG, balasubramanian2020stochastic, chen2021solving}; (ii) the Jacobian estimators are updated with a projection to ensure that errors of gradient estimators can be bounded regardless of the depth of the problem. When the objective function is convex or satisfies the $\mu$-PL condition (weaker than strong convexity), we develop a stage-wise version of our method and improve the complexity to $\mathcal{O}\left(1 / \epsilon^{2}\right)$ or $\mathcal{O}\left(1 / \left(\mu\epsilon\right)\right)$, which matches the corresponding lower bounds \citep{Agarwal2012InformationTheoreticLB}. The key in the analysis is to prove that the errors of Jacobian and function value estimators decrease in a stage-wise manner. Finally, to take advantage of adaptive learning rates, we also design an adaptive version of the SMVR method and prove the same rates. Adaptive SMVR performs better in practice and avoids tuning the learning rate manually. 

Compared with existing multi-level methods, this paper enjoys the following advantages:
\begin{compactenum}
\item We obtain the optimal sample complexity of $\mathcal{O}\left(1 / \epsilon^{3}\right)$ for non-convex objectives, which is better than existing multi-level methods~\cite{Yang2019MultilevelSG, balasubramanian2020stochastic, chen2021solving}. Although \citet{Zhang2021MultiLevelCS} achieve the same rate, their method uses a large and increasing batch size at the order of $\mathcal{O}\left(1 / \epsilon\right)$, which is impractical to use.
\item We achieve the optimal complexity of $\mathcal{O}\left(1 / \epsilon^{2}\right)$ and $\mathcal{O}\left(1 / \left(\mu\epsilon\right)\right)$ for convex and strongly convex functions respectively. Compared with ~\citet{Zhang2020OptimalAF}, we do not require each layer function $f_i$ is monotone and convex, and have a better dependence on $\mu$ for $\mu$-strongly convex functions.
\item We develop Adaptive SMVR method to make use of adaptive learning rates, which enjoys the same complexity but converges faster in practice.
\end{compactenum}
A comparison between our results and existing multi-level methods is shown in \cref{sample} and empirical results on practical problems demonstrate the effectiveness of our method.

\section{Related Work}\label{Sec:2}
This section briefly reviews related work on stochastic two-level and multi-level compositional optimization problems. 
\subsection{Two-Level Compositional Optimization}
\citet{wang2017stochastic} first introduce stochastic compositional gradient descent (SCGD) to minimize a composition of two-level expected-value functions. The SCGD method adopts two step size sequences in different time scales to update the decision variable and inner function separately. When the inner function is smooth, their method achieves a complexity of $\mathcal{O}\left(1 / \epsilon^{7}\right)$ for non-convex objectives, $\mathcal{O}\left(1 / \epsilon^{3.5}\right)$ for convex functions, and $\mathcal{O}\left(1 /\left(\mu^{14/4} \epsilon^{5/4}\right)\right)$ for $\mu$-strongly convex functions. In a subsequent work \citep{DBLP:journals/jmlr/WangLF17}, the accelerated stochastic compositional proximal gradient (ASC-PG) is proposed to improve the complexity to $\mathcal{O}\left(1 / \epsilon^{4.5}\right)$, $\mathcal{O}\left(1 / \epsilon^{2}\right)$, and $\mathcal{O}\left(1 / \epsilon\right)$ for non-convex, convex and strongly convex functions, respectively.

Instead of using two-timescale step sizes, a single-timescale method called Nested Averaged Stochastic Approximation (NASA) has been developed by \citet{Ghadimi2020AST} which achieves the complexity of  $\mathcal{O}\left(1 / \epsilon^{4}\right)$ for non-convex objectives. With the increasing popularity of variance reduction techniques for one-level stochastic optimization such as SARAH~\citep{Nguyen2017SARAHAN}, SPIDER  \citep{Fang2018SPIDERNN}, SpiderBoost  \citep{Wang2018SpiderBoostAC} and STORM  \citep{Cutkosky2019MomentumBasedVR}, variance reduced algorithms are also developed for two-level compositional problems with improved rates under a slightly stronger smoothness assumption~\citep{Huo2018AcceleratedMF, Yuan2019EfficientSN, Zhang2019ACR, Zhang2019ASC, Liu2021VarianceRM}. The optimal $\mathcal{O}\left(1 / \epsilon^{3}\right)$ sample complexity is achieved by \citet{Yuan2019EfficientSN} based on SARAH and by \citet{Zhang2019ASC} based on SPIDER with a large batch size. Later, \citet{Yuan2020StochasticRM} develop an algorithm named  STORM-Compositional, which leads to an $\mathcal{O}\left(1 / \epsilon^{3}\right)$ complexity using mini-batches. To avoid using batches, \citet{qi2021online} propose a method based on STORM and obtain the same rate. However, these two-level methods can not be extended to multi-level optimization problems directly. 

\subsection{Multi-Level Compositional Optimization}
\citet{Yang2019MultilevelSG} first investigate the problem of multi-level optimization and develop accelerated $T$-level stochastic compositional gradient descent (A-TSCGD). By using an extrapolation-interpolation scheme, their method achieves a sample complexity of $\mathcal{O}\left(1 / \epsilon^{(7+K)/2}\right)$ for $K$-level problems. This rate is improved to $\mathcal{O}\left(1 / \epsilon^{(3+K)/4}\right)$ for strongly convex functions. Later, \citet{balasubramanian2020stochastic} propose a nested linearized averaging stochastic gradient method (NLASG), which extends the NASA \citep{Ghadimi2020AST} algorithm to the general $K \geq 1$ setting and achieves a sample complexity of $\mathcal{O}\left(1 / \epsilon^{4}\right)$. In a concurrent work, \citet{chen2021solving} come up with a stochastically corrected stochastic compositional gradient method (SCSC), which uses a technique similar to STORM to estimate the function values at each level, and also has a sample complexity of $\mathcal{O}\left(1 / \epsilon^{4}\right)$. 

Recently, \citet{Zhang2021MultiLevelCS} present the Nested-SPIDER method, which uses nested variance reduction to approximate the gradient and improves the sample complexity to  $\mathcal{O}\left(1 / \epsilon^{3}\right)$. However, their method requires a large and increasing batch size at the order of  $\mathcal{O}\left(1 / \epsilon\right)$. In the first iteration of each stage, the batch size even has to be set as large as $\mathcal{O}\left(1 / \epsilon^{2}\right)$. No complexities are provided for convex and strongly convex functions. Later, \citet{Zhang2020OptimalAF} prove that the sample complexity can be improved to $\mathcal{O}\left(1 / \epsilon^{2}\right)$ when every layer function $f_i$ is monotone and convex, using a general stochastic sequential dual method (SSD). The complexity is further reduced to $\mathcal{O}\left(1 /(\mu^2\epsilon)\right)$ for $\mu$-strongly convex functions. However, their method requires strong assumptions, i.e., layer-wise convexity and monotonicity. In contrast, our method only requires the overall objective function to be convex or strongly convex to achieve the same complexity for convex functions and an even better complexity for strongly convex functions.

\section{The Proposed Method} \label{Sec:3}
We first discuss the main challenge in solving multi-level compositional optimization problems. Then, we develop an optimal method for non-convex objectives. Finally, we explore additional conditions to further improve the sample complexity. 

\subsection{Notations and Assumptions}
Let $\xi$ represent some random variable (in practice, it represents a training sample or a batch of samples drawn from the oracle) and $\Norm{\cdot}$ denote the Euclidean norm of a vector. For simplicity, we use $\Pi_{L_f}$ to denote the projection onto the ball with radius $L_f$, i.e., 
\begin{align*}
    \Pi_{L_f}(\x) = \underset{\|\w\| \leqslant L_{f}}{\operatorname{argmin}} \|\w-\x\|^{2}.
\end{align*}
We give the definition of sample complexity below.
\begin{definition}
The sample complexity is the number of samples needed to find a point satisfying $\E \left[\Norm{\nabla F(\w)}\right] \leq \epsilon$ ($\epsilon$-stationary) or $\E \left[ F(\w)-\inf_{\w} F(\w)\right] \leq \epsilon$ ($\epsilon$-optimal).
\end{definition}
Moreover, we make the following assumptions throughout the paper, which are commonly used in the studies of stochastic compositional optimization  \citep{DBLP:journals/jmlr/WangLF17,wang2017stochastic,Yuan2019EfficientSN,Zhang2019ASC,Zhang2021MultiLevelCS}. We treat the parameters $L_f$, $L_J$, $\sigma_f$, $\sigma_J$,  $\LL_f$, $\LL_J$ below as global constants.

\begin{ass} \label{asm:stochastic1} (Smoothness and Lipschitz continuity)
    	All functions $f_1, \dotsc, f_K$ are $L_f$-Lipschitz continuous and their Jacobians $\nabla f_1, \dotsc, \nabla f_K$ are $L_J$-Lipschitz continuous. Note that this implies the objective function $F(\w)$ is $L_F$ smooth, where $L_F = L_f^{2K-1}L_J\sum_{i=1}^K\frac{1}{L_f^i}$.
\end{ass}

\begin{ass}\label{asm:stochastic2} (Bounded variance)\ \ For $1\leq i\leq K$:
\begin{align*}
    \E_{\xi_t^i}\left[f_i(\x;\xi_t^i)\right] = f_i(\x),\ & \\
	\E_{\xi_t^i}\left[\nabla f_i(\x;\xi_t^i)\right] = \nabla f_i(\x&),\\
	\E_{\xi_t^i}\left[\Norm{f_i(\x;\xi_t^i) - f_i(\x)}^2 \right]\leq & \ \sigma_f^2,\\
	\E_{\xi_t^i}\left[\Norm{\nabla f_i(\x;\xi_t^i) - \nabla f_i(\x) }^2 \right] \leq & \ \sigma_J^2,
\end{align*}
where $\{\xi_t^i\}_{i=1}^K$ are mutually independent.
\end{ass}

\begin{ass}\label{asm:stoc_smooth3}
	(Mean-squared smoothness)
	\begin{align*}
			\E_{\xi_t^i}\left[\Norm{f_i(\x;\xi_t^i)-  f_i(\y;\xi_t^i)}^2\right] &\leq \LL_f^2 \Norm{\x - \y}^2,\\
		\E_{\xi_t^i}\left[\Norm{\nabla f_i(\x;\xi_t^i)- \nabla f_i(\y;\xi_t^i)}^2\right] &\leq  \LL_J^2 \Norm{\x - \y}^2.
	\end{align*}
\end{ass}

\begin{ass}\label{asm:stochastic4} $F_{*}=\inf_{\w} F(\w) \geq-\infty$ and $F\left(\w_{1}\right)-F_{*} \leq \Delta_{F}$ for the initial solution $\w_{1}$.
\end{ass}

\subsection{The Challenge}
Compared with one-level problems, the main dilemma in multi-level optimization is that we can not obtain the unbiased gradient of the function $F$. Take a two-level compositional problem as an example.  The objective function can be written as: $F(\w) = f\circ g(\w)$ and its gradient is:
\begin{align*}
	 \nabla F(\w) = \nabla g(\w)\cdot\nabla f(g(\w)).
\end{align*}
Although we can access to the unbiased estimation of each layer function and its gradient, i.e., $\E_{\xi_1}\left[g(\x;\xi_1)\right] = g(\x)$, $\E_{\xi_2}\left[f(\x;\xi_2)\right] = f(\x)$ and $\E_{\xi_2}\left[\nabla f(\x;\xi_2)\right] = \nabla f(\x)$, it is still challenging to obtain an unbiased estimation of the gradient $\nabla f(g(\w))$. This is because the expectation over $\xi_1$ cannot be moved inside of $\nabla f$ such that:
\begin{align*}
    \E_{\xi_1,\xi_2}\left[\nabla f(g(\w;\xi_1);\xi_2)\right] \neq \nabla f(g(\w)).
\end{align*} 
Due to the same reason, it is also difficult to get the unbiased estimation of the function value:
\begin{align*}
    \E_{\xi_1,\xi_2}\left[f(g(\w;\xi_1);\xi_2)\right] \neq f(g(\w)).
\end{align*}
The above challenge motivates us to use the variance reduced estimator to have a better evaluation of both function values and Jacobians of each level to ensure that the estimation errors reduce over time.

However, variance reduced estimators used in two-level optimization  problems~\cite{qi2021online} can not be applied to multi-level directly, because the error might blow up as the depth increases if the estimators of Jacobians are not bounded. To handle this issue, \citet{Zhang2021MultiLevelCS} use an extremely small step size and re-estimate the function values and Jacobians of all levels after several iterations with a large batch size. However, this method inevitably introduces large batches (as large as $\mathcal{O}\left(1 / \epsilon^{2}\right)$) at the beginning of each stage, and since they use SPIDER~\citep{Fang2018SPIDERNN} as their estimator, their method requires a batch size of $\mathcal{O}\left(1 / \epsilon\right)$ at other iterations. To avoid using large batches, our method uses STORM~\citep{Cutkosky2019MomentumBasedVR} estimator and projects gradients onto a ball to ensure the Jacobians can be well bounded so that the error of the gradient estimator does not blow up.

\subsection{Stochastic Multi-Level Variance Reduction}
Now, we present the proposed algorithm -- Stochastic Multi-level Variance Reduction method (SMVR) for solving problem~(\ref{prob:1}). The goal of our algorithm is to find an $\epsilon$-stationary point with low sample complexity. As mentioned before, the main difficulty is that we can not obtain an unbiased estimation of the gradients and inner function values in the multi-level setting. We note that, in the one-level problem,  STORM uses a momentum-based variance reduction method to evaluate the gradient: 
\begin{align*}
    \d_{t} &= (1-\beta_t)\d_{t-1} + \beta_t\nabla f\left(\x_{t};\xi_{t}\right) \\
    &\quad+(1-\beta_t)\left(\nabla f\left(\x_{t};\xi_{t}\right)-\nabla f\left(\x_{t-1}; \xi_{t}\right)\right).
\end{align*}
This technique reduces the variance of the estimated value and obtains the optimal rate. Inspired by STORM, we apply similar variance reduced estimators in each level to approximate the gradient more accurately.  
\begin{algorithm}[tb]
	\caption{SMVR}
	\label{alg:multi-storm}
	\begin{algorithmic}
	\STATE {\bfseries Input:} time step $T$, initial points $\left(\w_1,\u_1,\v_1\right)$,  \\ \quad\quad\quad parameter $c$ and learning rate sequence $\left\{\eta_{t}\right\}$ 
		\FOR{time step $t = 1$ {\bfseries to} $T$}
		\STATE Set $\u_t^0 = \w_t$, $\beta_t = c \eta_{t-1}^2$\\
		\FOR{level $i = 1$ {\bfseries to} $K$}
		\STATE Sample $\xi_t^i$
		\STATE Compute estimator $\u_t^i$ according to (\ref{eq:esti_u})
		\STATE Compute estimator $\v_t^i$ according to (\ref{eq:esti_v})
		\ENDFOR
		\STATE Update gradient estimation: $\v_t = \prod_{i=1}^K \v_t^i$
		\STATE Update the weight: $\w_{t+1} = \w_t - \eta_t \v_t$
		\ENDFOR
	\STATE Choose $\tau$ uniformly at random from $\{1, \ldots, T\}$
	\STATE Return $(\w_{\tau},\u_{\tau},\v_{\tau})$
	\end{algorithmic}
\end{algorithm}

The proposed method is described in \cref{alg:multi-storm}. In each time step $t$, we use two sequence $\u_t^i$ and $\v_t^i$ to estimate the function value and the gradient in level $i$. To estimate the function value, we use a nested STORM estimator, i.e.,  
\begin{align}\label{eq:esti_u}
     \u_t^i &= (1-\beta_t)\u_{t-1}^i + \beta_t f_i(\u_t^{i-1};\xi_t^i)\nonumber \\ &\quad + (1-\beta_t)\left(f_i(\u_t^{i-1};\xi_t^i) - f_i(\u_{t-1}^{i-1};\xi_t^i)\right).
\end{align}
This can be interpreted as that $\u_t^i$ is a STORM estimator of $f_i(\u_t^{i-1})$. For estimating the Jacobians, we use the nested STORM estimator followed by a projection, i.e.,
\begin{align}\label{eq:esti_v}
\v_t^i &= \Pi_{L_f}\left[(1-\beta_t)\v_{t-1}^i + \beta_t \nabla f_i(\u_t^{i-1};\xi_t^i) \right.\nonumber \\ &\quad \left.  + (1-\beta_t)\left(\nabla f_i(\u_t^{i-1};\xi_t^i) - \nabla f_i(\u_{t-1}^{i-1};\xi_t^i)\right)\right].
\end{align}
The projection operation is to ensure the error of the stochastic gradient estimator can be bounded; otherwise, they may blow up as the level becomes deeper. That is to say, on one hand, we want to enjoy variance reduction of the estimator (since true gradients are in the projected domain, projection does not hinder the analysis); on the other hand we do not want the variance of estimator accumulates over multiple levels ($\mathbf v_i$ is bounded after projection). Hence, projection on the Jacobian estimator is a perfect solution. After the gradient of each level is evaluated, we use the chain rule to calculate the estimated gradient of the objective function, i.e., $\v_t = \v_t^1 \v_t^2 \cdots \v_t^K$ and apply gradient descent to update the decision variable $\w_t$ at the end of each time step.

Note that in the first iteration, we evaluate the function value and gradient of each level simply as $\u_1^i = f(\u_1^{i-1};\xi_{1}^i)$ and $\v_1^i = \nabla f_i(\u_1^{i-1};\xi_{1}^i)$. Our algorithm does not need to use batches in any iterations. Of course, it supports mini-batches, and $\xi_t^i$ in the algorithm can represent a training sample or a batch of samples. Next, we show the sample complexity of the proposed method. Due to space limitations, all the proofs are deferred to the appendix. We define the constant $L_{1}=\max \left\{1, K L_{f}^{2(K-1)}, K L_{F}^{2}, 2 K\left(\sigma_{J}^{2}+\sigma_{f}^{2}\right),\right.\\ \left. 2\left(L_{J}^{2}+L_{f}^{2}\right)\left(2 K+2 K \sigma_{f}^{2}\right) \sum_{i=1}^{K}\left(2 \LL_{f}^{2}\right)^{i-1}\right\}$.
	   
\begin{theorem}\label{thm:main}
If we set $c=10L_{1}^2$, $\eta_t = \left(a+t \right)^{-1/3}$ and $a={\left( 20L_1^3\right)}^{3/2}$,  our algorithm finds an $\epsilon$ stationary point in $\mathcal{O}(1/\epsilon^{3})$ iterations.
\end{theorem}

\textbf{Remark:} The complexity is at the order of $\mathcal{O}\left(1 / \epsilon^{3}\right)$, which matches the lower bound in one-level setting \citep{Arjevani2019LowerBF}, implying multi-level setting does not make the problem much harder. Our SMVR method uses decreasing learning rates and avoids using batches in any iteration, which is more practical to implement compared with the existing method which requires huge batch size and changing the batch size over time \citep{Zhang2021MultiLevelCS}.

\subsection{Faster Convergence under Stronger Conditions}\label{Sec:3.3}
In this section, we explore whether other assumptions could be utilized to further improve the sample complexity. We develop a variant algorithm named Stage-wise SMVR, which achieves better complexity when the objective function satisfies the PL condition or convexity.

\begin{algorithm}[tb]
	\caption{Stage-wise SMVR}
	\label{alg:Stagewise multi-storm}
	\begin{algorithmic}
	\STATE {\bfseries Input:} initial points $\left(\w_0,\u_0,\v_0\right)$, parameter $c$
		\FOR{stage $s = 1$ {\bfseries to} $S$}
		\STATE Set $\eta_{s}$ and $T_{s}$ according to Lemma~\ref{lem:222}
		\STATE $\w_{s},\u_{s},\v_{s}$ = SMVR (with $T_{s}$, $\left(\w_{s-1},\u_{s-1},\v_{s-1}\right)$, $c$\\ \quad\quad\quad\quad\quad\quad\ and $\eta_{s}$)
		\ENDFOR
	\STATE Return $\w_{S}$
	\end{algorithmic}
\end{algorithm}

The new algorithm is a multi-stage version of the SMVR method, summarized in Algorithm~\ref{alg:Stagewise multi-storm}. Instead of decreasing the learning rate $\eta_t$ polynomially, we decrease $\eta$ and $\beta$ after each stage and increase the number of iterations per stage. At the end of each stage, the algorithm save the output $\w_{s},\u_{s},\v_{s}$, which are used for restarting in the next stage. With these modifications, we can obtain a better convergence guarantee under the PL condition or convexity.

First, we investigate the case that the objective function satisfies the PL condition, which is a commonly used condition in the literature \citep{Charles2018StabilityAG, Nouiehed2019SolvingAC, Xie2020LinearCO, Chewi2020GradientDA}. We first introduce the definition of the PL condition.
\begin{definition}
$F(\w)$ satisfies the $\mu$-PL condition if there exists $\mu > 0$ such that:
\begin{align*}
2 \mu\left(F(\mathbf{w})-F_{*}\right) \leq\|\nabla F(\mathbf{w})\|^{2}.    
\end{align*}
\end{definition}
Note that a function can be non-convex and still satisfy the PL condition. Also, the PL condition is weaker than strong convexity~\citep{Karimi2016LinearCO}. With this condition, we can prove that the error of function estimator $\u_s$ and gradient estimator $\v_s$ would decrease after each stage.
\begin{lemma}\label{lem:222}
    Define $\epsilon_{1} = \frac{8L_1}{\mu}$ and $\epsilon_{s} = \frac{\epsilon_{1}}{2^{s-1}}$, with $\beta_{1} =\frac{1}{2L_1}$, $T_{1} =  \max \left\{ 4L_1 K\left( \sigma_{f}^2+\sigma_{J}^2\right), 2\sqrt{2L_1}\Delta_{F}\right\}$, $\beta_{s}=\frac{\mu \epsilon_{s-1}}{L_2}$,$T_{s}=\max \left\{\frac{4L_2^{3/2}}{\mu \epsilon_{s-1}}, \frac{4L_2}{\mu^{3 / 2} \sqrt{\epsilon_{s-1}}}\right\}$, $c=16L_{1}^2$, $\eta_s = \sqrt{\beta_s/c}$ and $L_2 = 64 L_1^2$, the output of Algorithm~\ref{alg:Stagewise multi-storm} 
   at each stage satisfies:
	\begin{gather*}	        E\left[F\left(\w_{s}\right)-F_{*}\right] \leq \epsilon_{s};\\
           \sum_{i=1}^K\E\left[\Norm{f_i(\u_{s}^{i-1}) - \u_{s}^i}^2 + \Norm{\v_{s}^i - \nabla f_i(\u_{s}^{i-1})}^2\right] \leq \mu\epsilon_{s}.
    \end{gather*}
\end{lemma}
The above lemma shows that the objective gap $F\left(\w_{s}\right)-F_{*}$ is halved after each stage. So, after $S = \log_{2}\left(\frac{2\epsilon_{1}}{\epsilon}\right)$ stages, the output of our method  satisfies  $F\left(\w_{S}\right)-F_{*} \leq \epsilon$. Based on Lemma~\ref{lem:222}, we prove the convergence rate of our algorithm in the following theorem.
\begin{theorem}\label{thm:main2}
	Assume $F(\w)$ satisfies the $\mu$-PL condition. Stage-wise SMVR attains an $\epsilon$-optimal point with a sample complexity of $\mathcal{O}\left({1}/\left(\mu \epsilon\right)\right)$.
\end{theorem}

If the objective function satisfies the convexity rather than PL condition, our method can still utilize this property to improve the sample complexity, as stated below.
\begin{theorem}\label{thm:main3}
	Assume $F(\w)$ is convex and $\left\|x^{*}\right\| \leq D$, where $x^{*}$ denote an optimal solution. Our algorithm attains an $\epsilon$-optimal point with a complexity of $\mathcal{O}\left({1}/{\epsilon^2}\right)$.
\end{theorem}

\textbf{Remark:} 
Our Stage-wise SMVR method behaves optimally when the objective function enjoys the PL condition or convexity. For smooth and convex functions, our method matches the $\mathcal{O}\left({1}/{\epsilon^2}\right)$ lower bound for this problem \citep{Agarwal2012InformationTheoreticLB}. When it comes to the PL condition, we note that there exists $\mathcal{O}\left({1}/\left(\mu \epsilon\right)\right)$ lower bound for the $\mu$-strongly convex setting  \citep{Agarwal2012InformationTheoreticLB}, which is a special case of the PL condition, thus proving our method is optimal. Compared with existing results \citep{Zhang2020OptimalAF}, our analysis requires weaker assumptions and enjoys a better and optimal dependence in terms of $\mu$.

\section{SMVR with Adaptive Learning Rates}\label{Sec:4}
In this section, we show that the proposed method can be extended to adaptive learning rates and obtains the same sample complexity. Adaptive learning rates are widely used in stochastic optimization problems, and many successful methods have been proposed, such as AdaGrad \citep{COLT:Adaptive:Subgradient}, Adam \citep{kingma:adam}, AMSGrad \citep{j.2018on}, AdaBound \citep{luo2018adaptive}, etc. However, it remains less investigated in the stochastic multi-level literature.  Inspired by the above methods, we develop an adaptive version of our method, named Adaptive SMVR. To use adaptive learning rates, we revise the weight update step in \cref{alg:multi-storm} as follows:
\begin{align}\label{rule1}
    \w_{t+1} = \w_t - \frac{\eta_t}{\sqrt{\h_{t}}+\delta} \v_t,
\end{align}
where $\delta > 0$ is a parameter to avoid dividing zero and the parameter $\h_{t}$ can take following forms:
\begin{align}\label{rule2}
    \begin{aligned}
    &\text{AdaGrad-type:} \quad \,\;\mathbf{h}_{t}=\frac{1}{t} \sum_{i=1}^{t} {\v}_{i}^{2}\\
    &\text{Adam-type:} \quad\quad\;\;\, \mathbf{h}_{t}=\left(1-\beta_{t}^{\prime}\right) \mathbf{h}_{t-1}+\beta_{t}^{\prime} {\v}_{t}^{2}\\
    &\text{AMSGrad-type:} \quad \mathbf{h}_{t}^{\prime}=\left(1-\beta_{t}^{\prime}\right) \mathbf{h}_{t-1}^{\prime}+\beta_{t}^{\prime} \v_{t}^{2},\\ &\quad\quad\quad\quad\quad\quad\quad \ \ \mathbf{h}_{t}=\max \left(\mathbf{h}_{t-1}, \mathbf{h}_{t}^{\prime}\right) \\
    &\text{AdaBound-type:} \ \ \ \, \mathbf{h}_{t}^{\prime}=\left(1-\beta_{t}^{\prime}\right) \mathbf{h}_{t-1}^{\prime}+\beta_{t}^{\prime} \v_{t}^{2}, \\ &\quad\quad\quad\quad\quad\quad\quad \ \ \, \mathbf{h}_{t}=\Pi_{\left[1 / c_{u}^{2}, 1 / c_{l}^{2}\right]}\left[\mathbf{h}_{t}^{\prime}\right]
    \end{aligned}
\end{align}
where $c_{l} \leq c_{u}$ and $\Pi_{[a, b]}$ projects the input into the range $[a, b]$. Inspired by the recent study of Adam-style methods~\cite{guo2022stochastic}, we can give the sample complexity of the Adaptive SMVR in Theorem~\ref{thm:T3} using similar analysis. 
\begin{theorem}\label{thm:T3}
	If we choose $c=10L_{3}^2$, $\eta_t = \left(a+t \right)^{-1/3}$ and  $a={\left( 20L_3^3\right)}^{3/2}$, Adaptive SMVR with learning rate defined in (\ref{rule1}) and (\ref{rule2}), can obtain a stationary point in $\mathcal{O}(1/\epsilon^{3})$ iterations, where $L_3$ is a constant indicated in the proof.
\end{theorem}
\textbf{Remark:} The sample complexity is still at the order of $\mathcal{O}\left(1 / \epsilon^{3}\right)$, when the adaptive learning rate is used. Our adaptive SMVR changes the learning rate automatically, which reduces the need to tune hyper-parameters manually. When the objective function satisfies the convexity or PL condition, Theorem~\ref{thm:main2} and Theorem~\ref{thm:main3} can be easily extended to the adaptive version with the same sample complexity.
\begin{figure*}[ht]
\vskip 0.2in
	\begin{center}
		\subfigure{
			\includegraphics[width=0.24\textwidth]{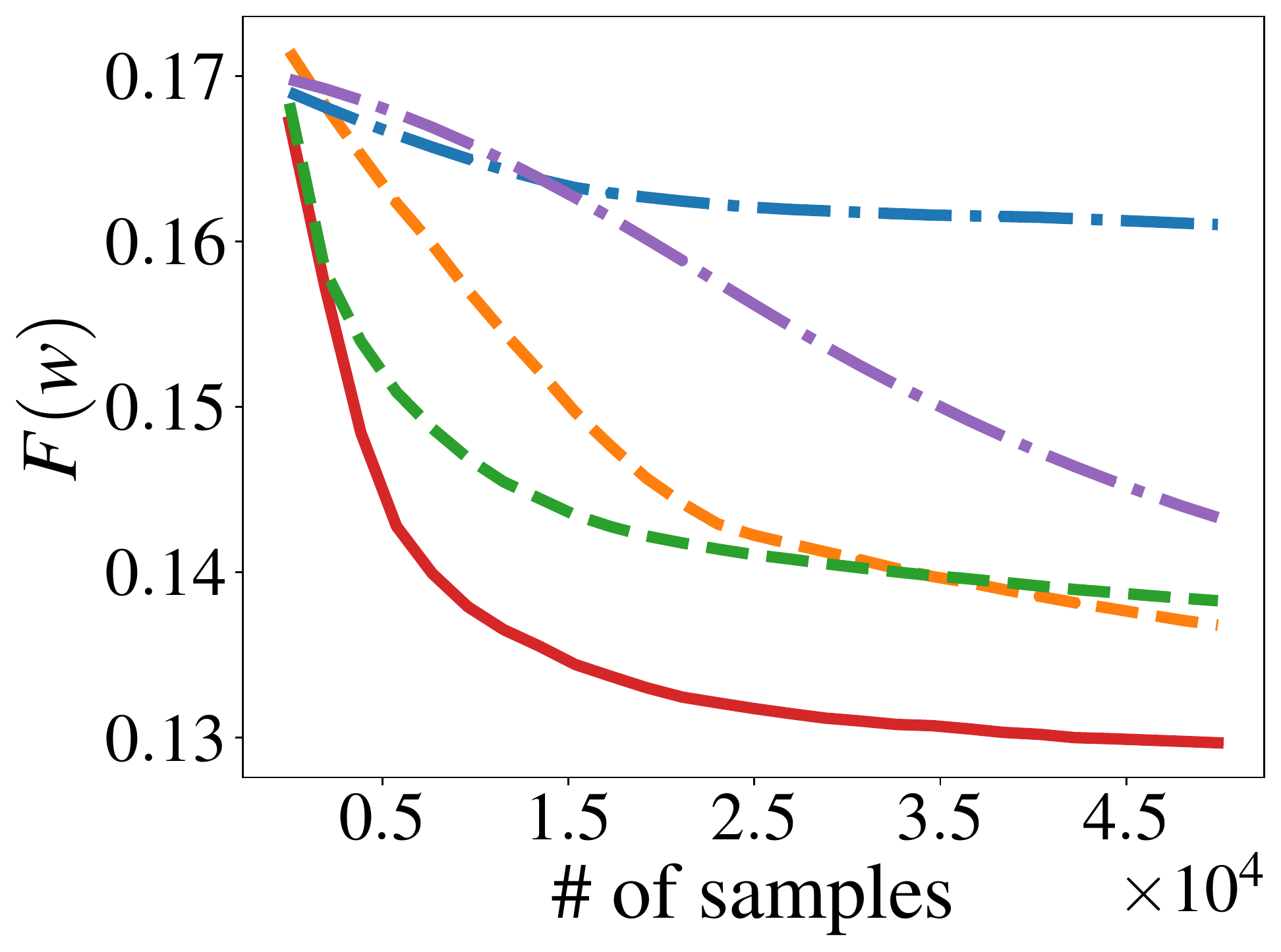}
			\includegraphics[width=0.24\textwidth]{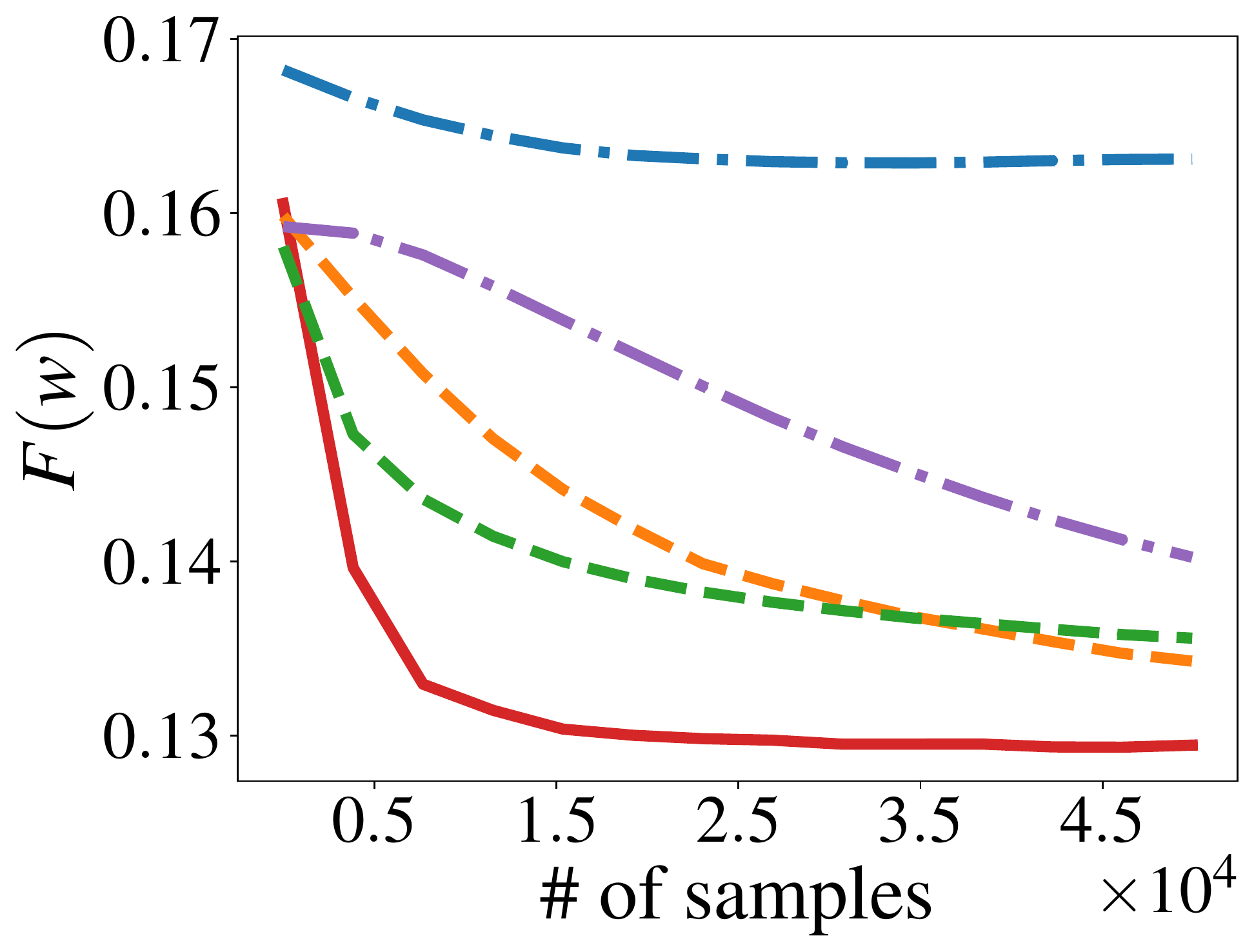}
			\includegraphics[width=0.24\textwidth]{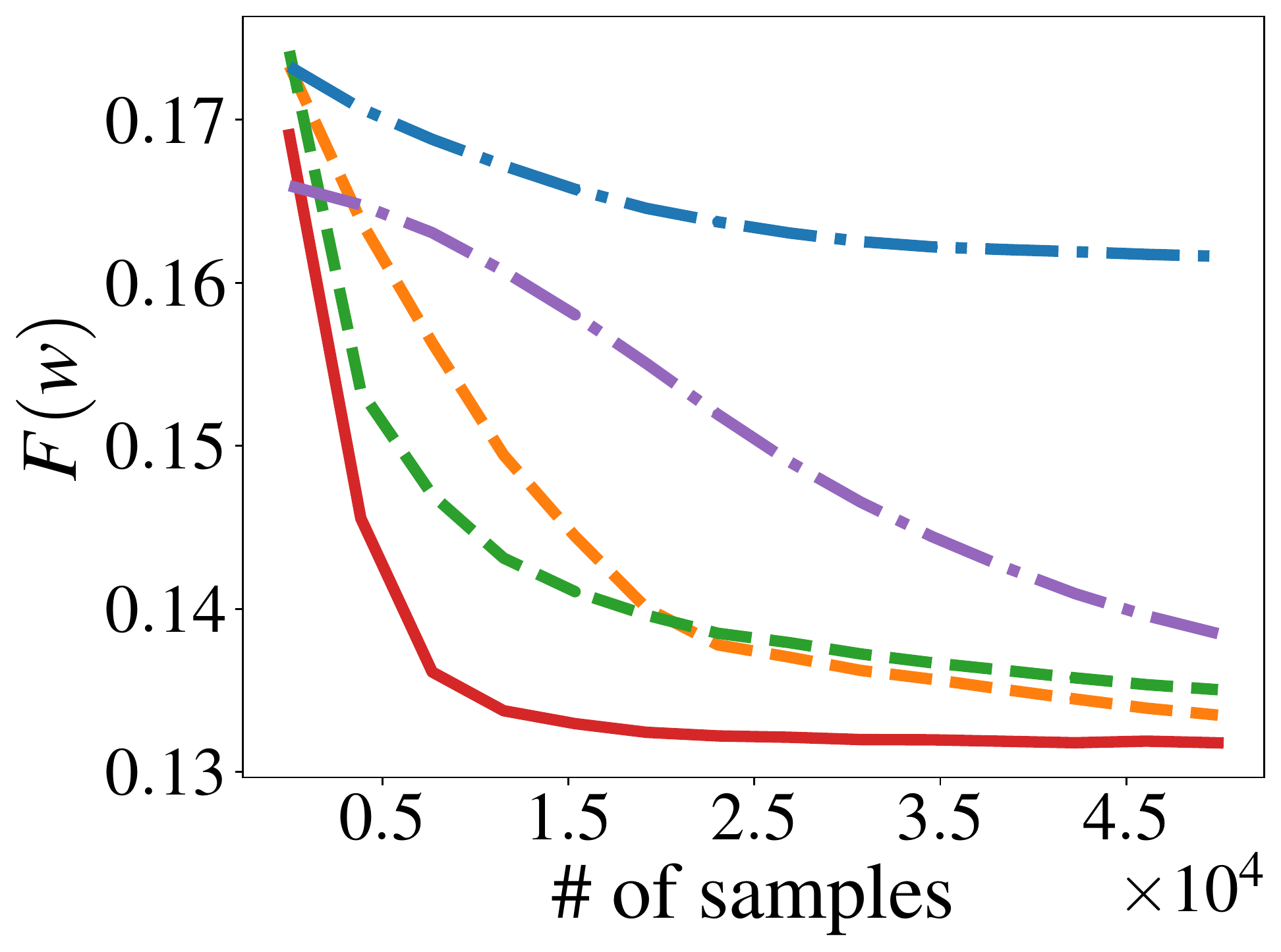}
			\includegraphics[width=0.24\textwidth]{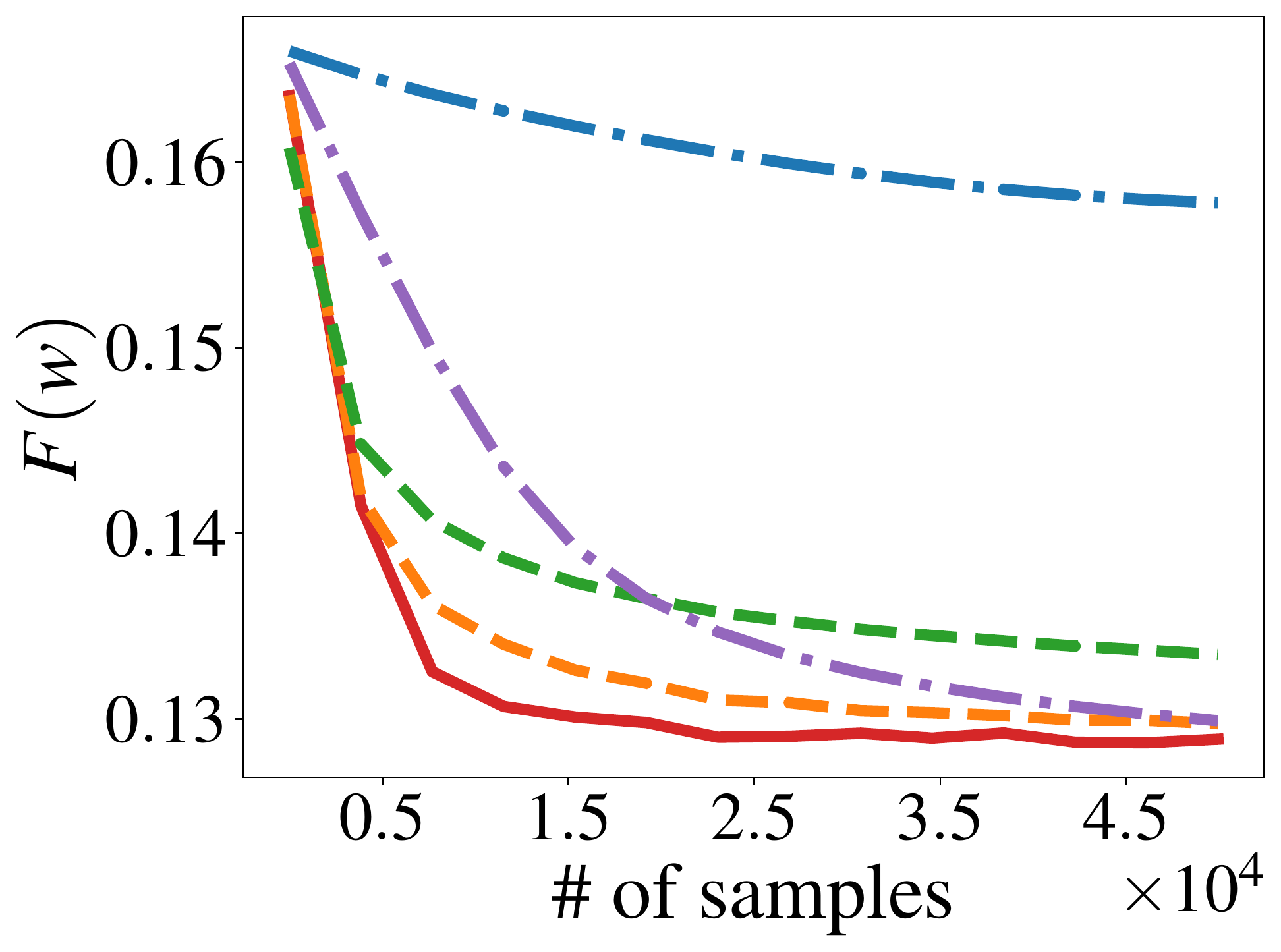}
		}%
		\setcounter{subfigure}{0}
		\subfigure[Industry-10]{
			\includegraphics[width=0.24\textwidth]{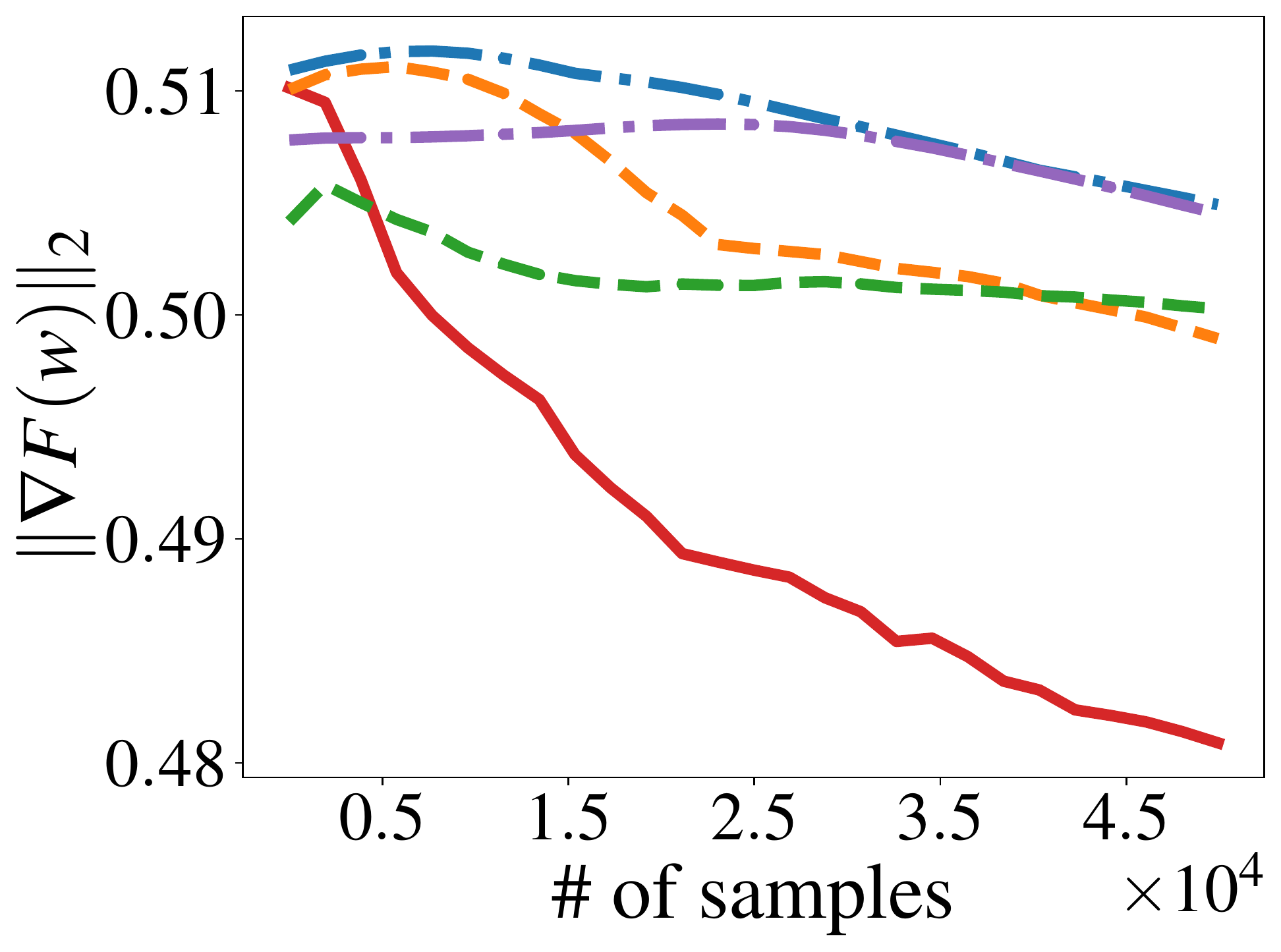}
		}
		\hspace{-3mm}
		\subfigure[Industry-12]{
			\includegraphics[width=0.24\textwidth]{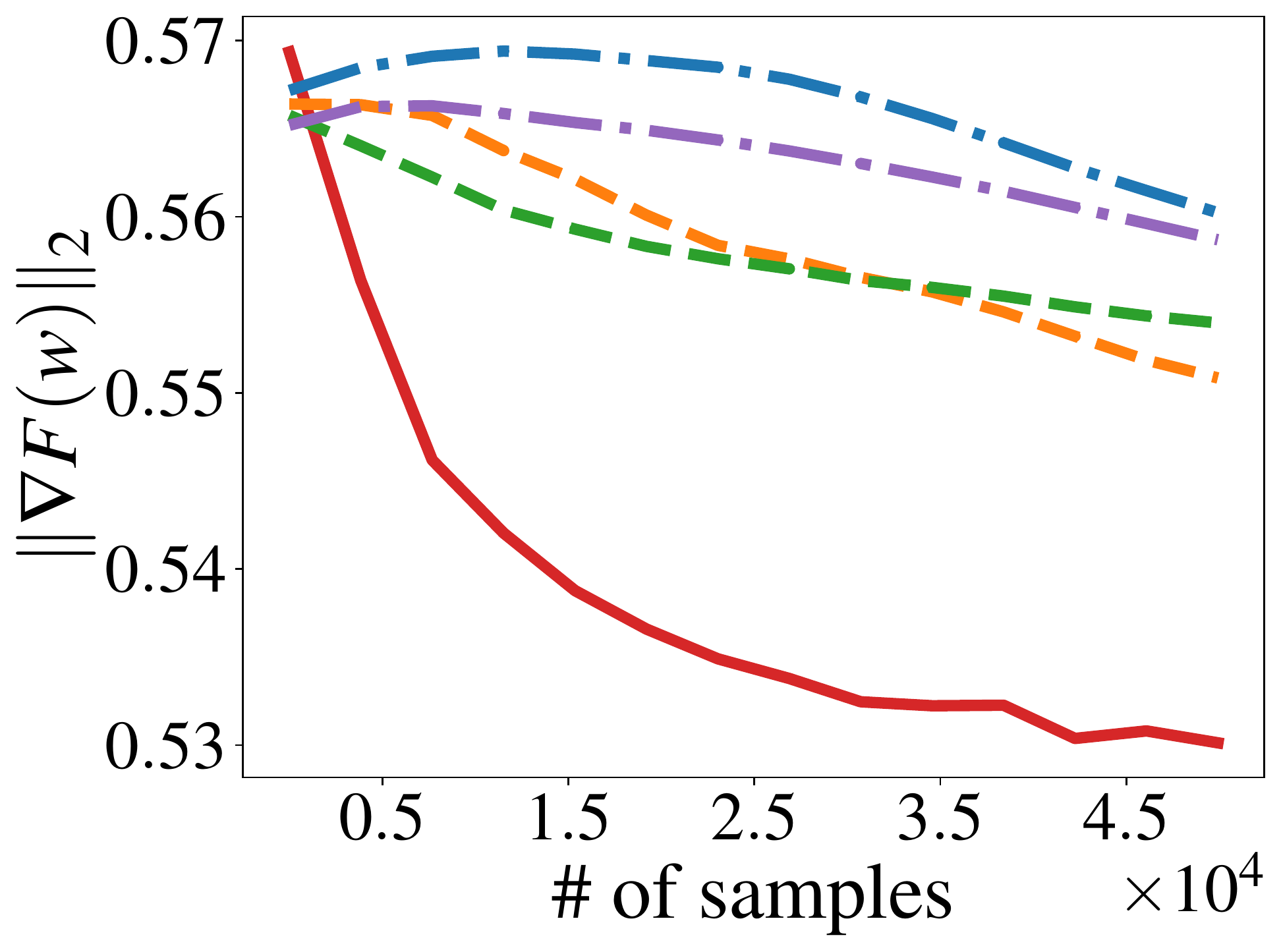}
		}
		\hspace{-3mm}
		\subfigure[Industry-17]{
			\includegraphics[width=0.24\textwidth]{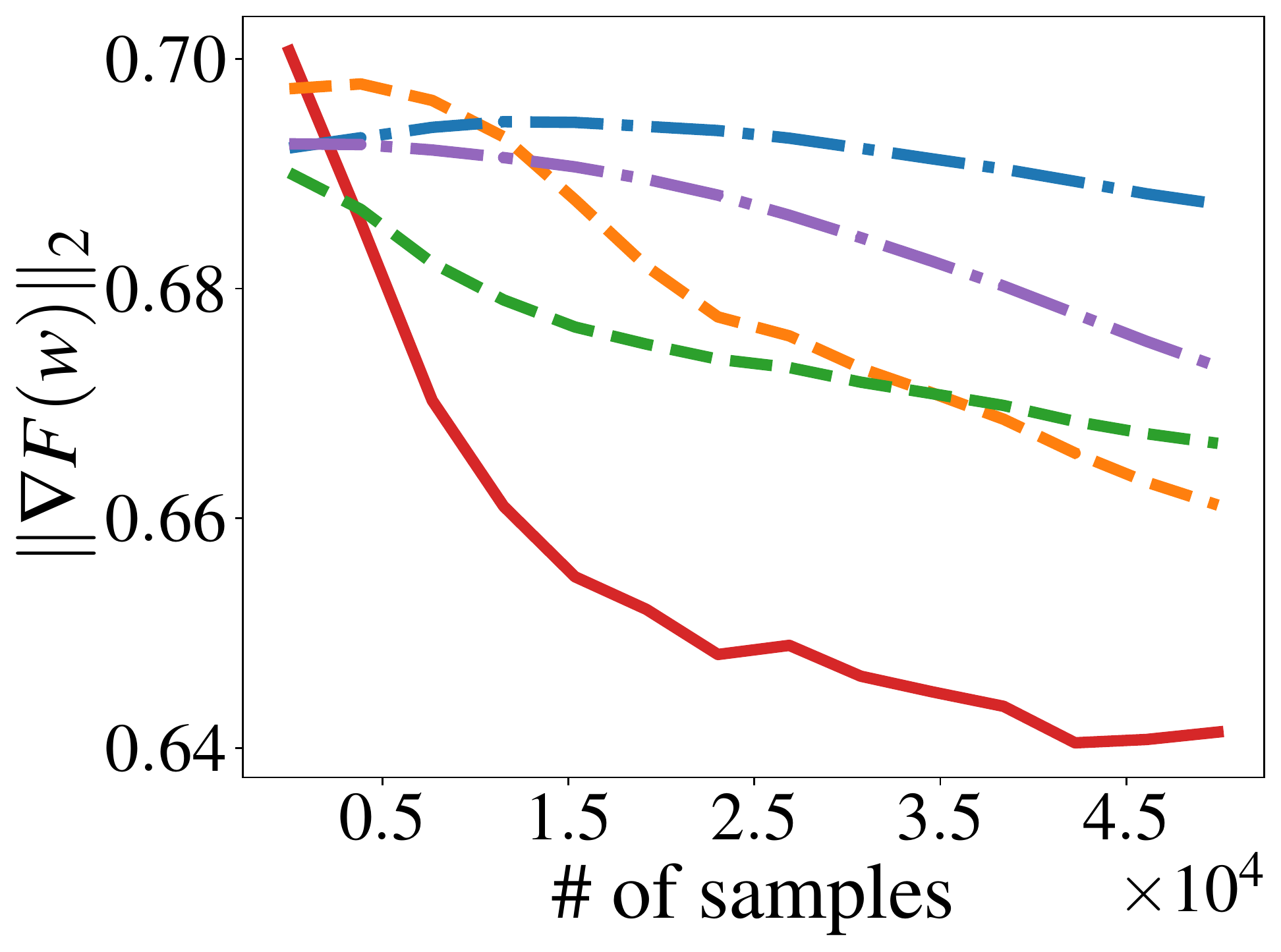}
		}
		\hspace{-3mm}
		\subfigure[Industry-30]{
			\includegraphics[width=0.24\textwidth]{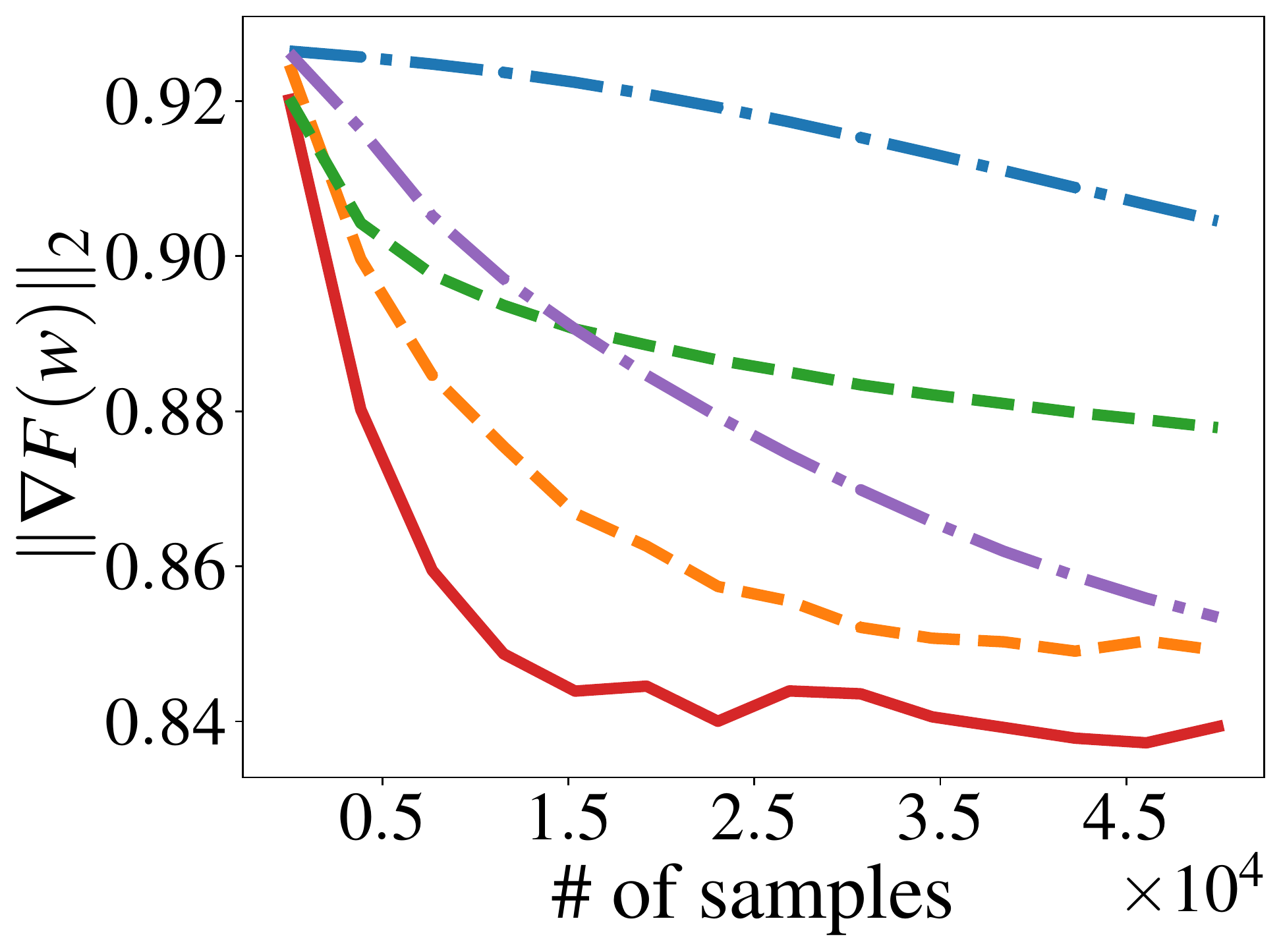}
		}%
		\vskip -0.05in
		\setcounter{subfigure}{1}
		\subfigure{
			\includegraphics[width=0.8\textwidth]{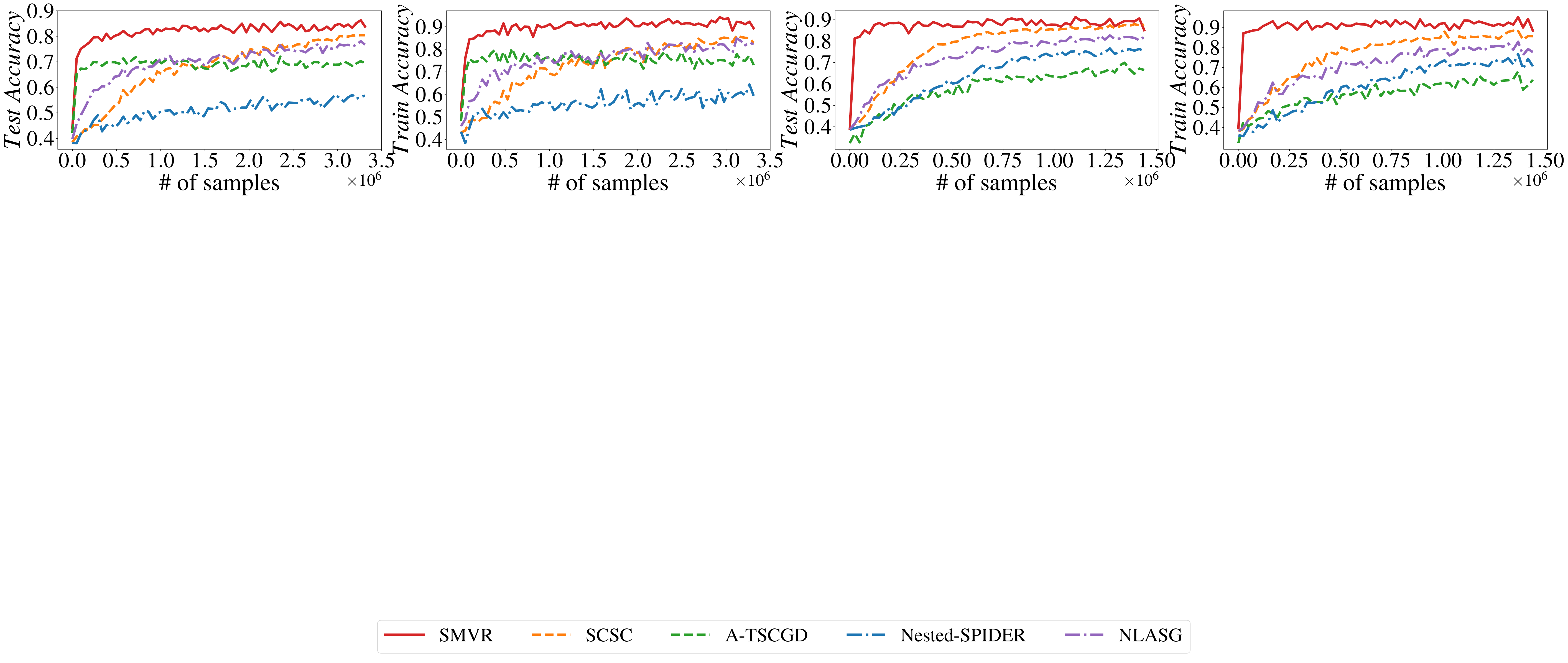}
		}%
		\vskip -0.05in
		\caption{Results for Risk-Averse Portfolio Optimization.}
		\label{fig:1}
	\end{center}
\vskip -0.2in
\end{figure*}
\begin{figure*}[!ht]
\vskip 0.2in
	\begin{center}
		\subfigure{
			\includegraphics[width=0.24\textwidth]{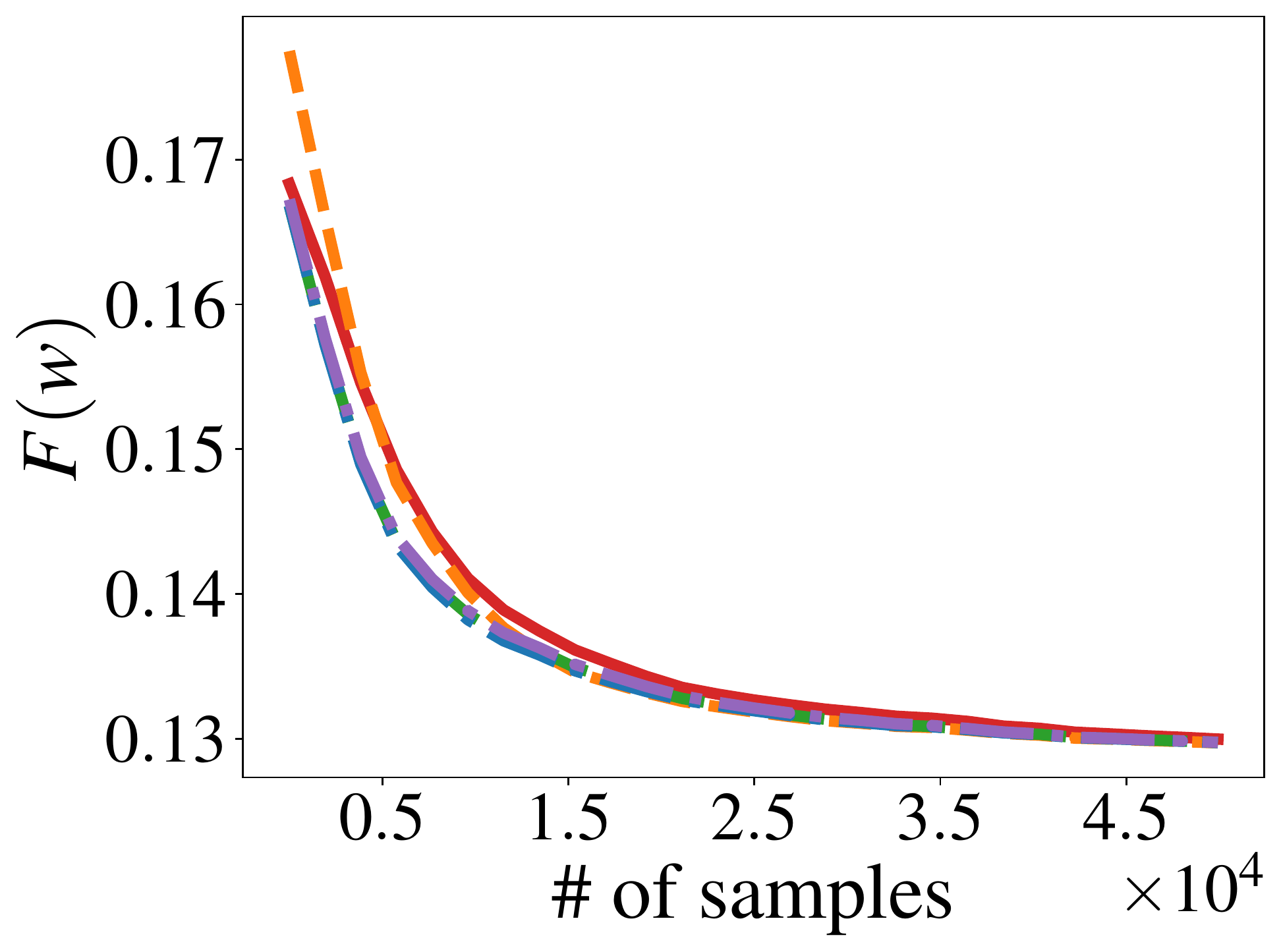}
			\includegraphics[width=0.24\textwidth]{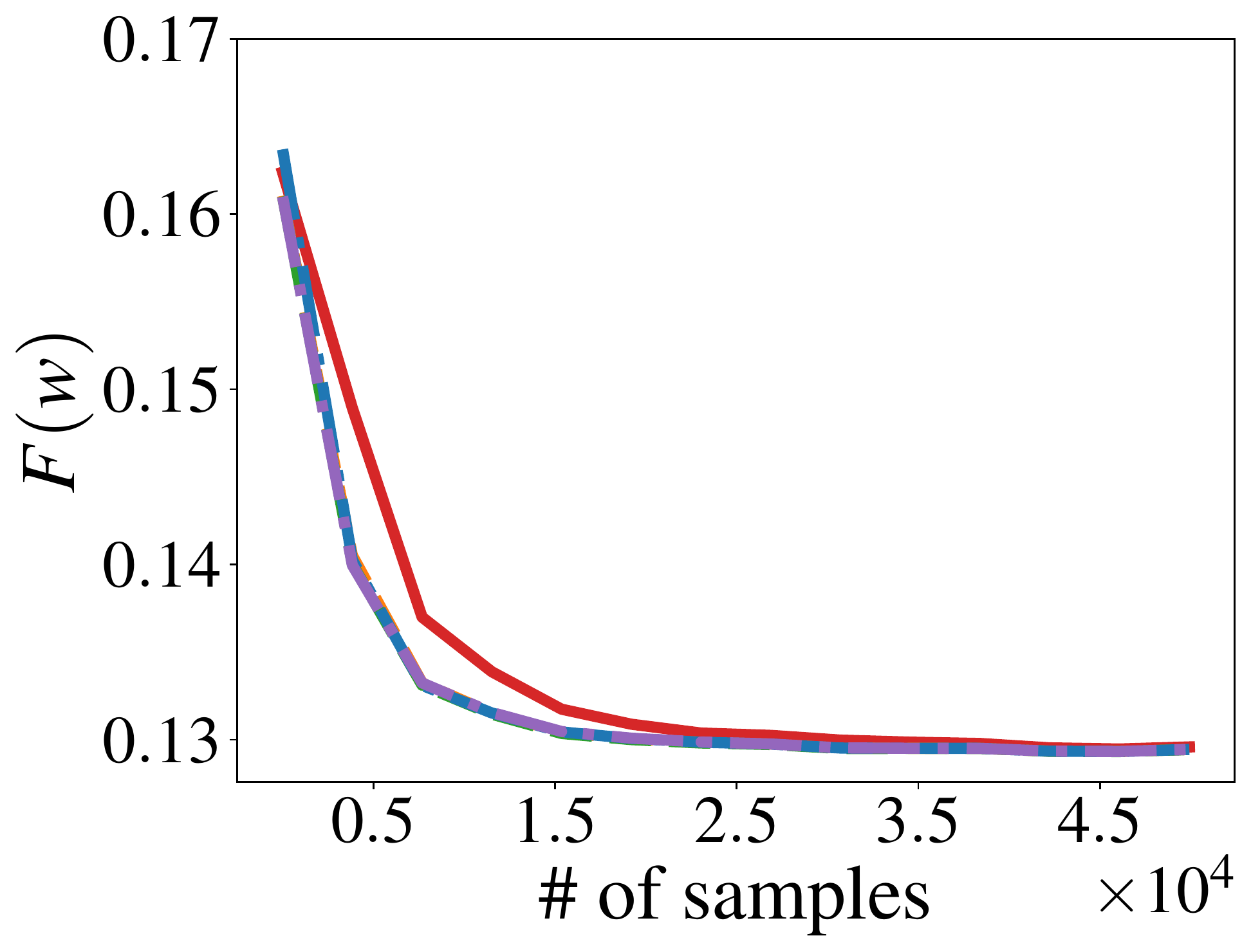}
			\includegraphics[width=0.24\textwidth]{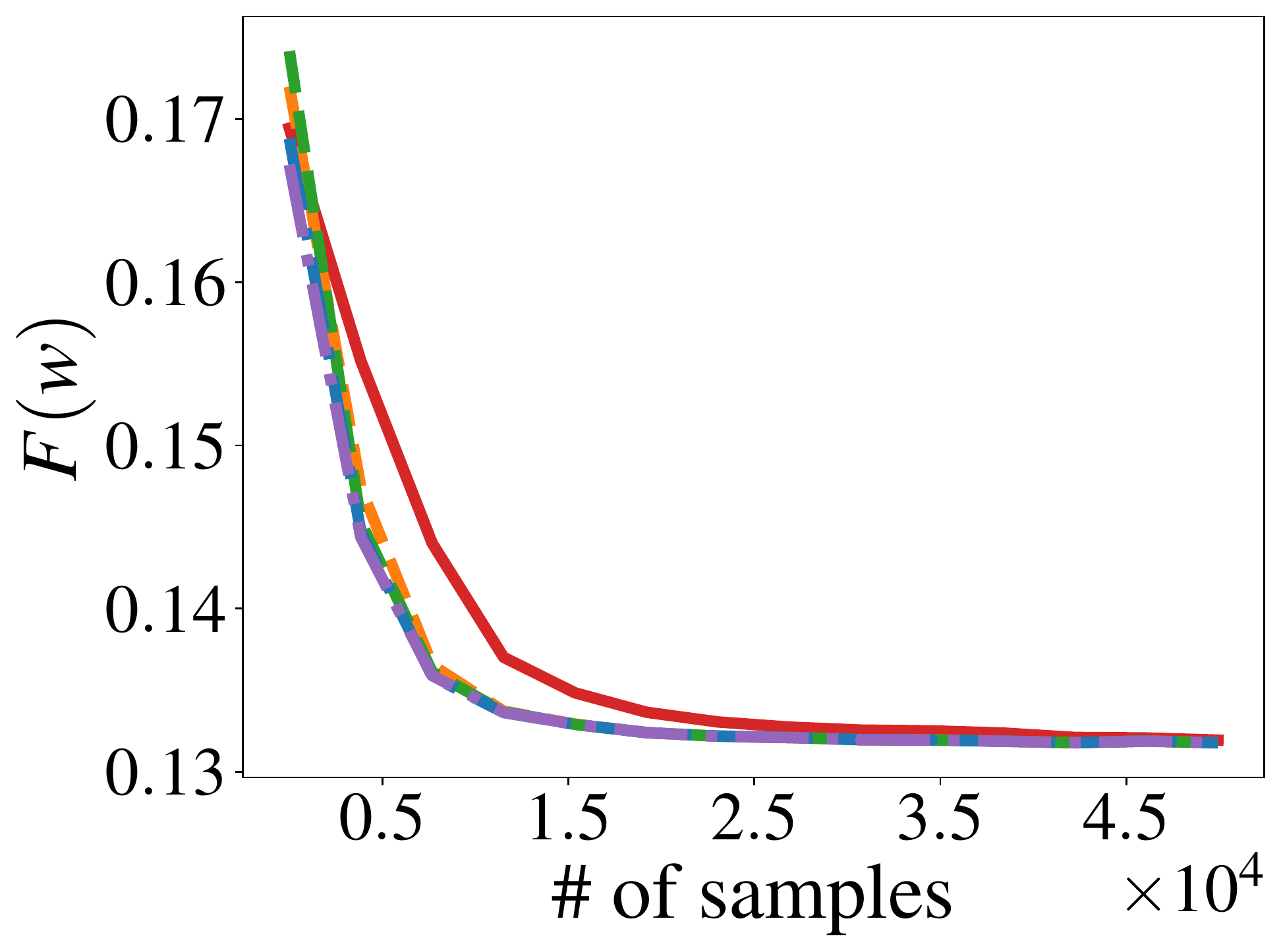}
			\includegraphics[width=0.24\textwidth]{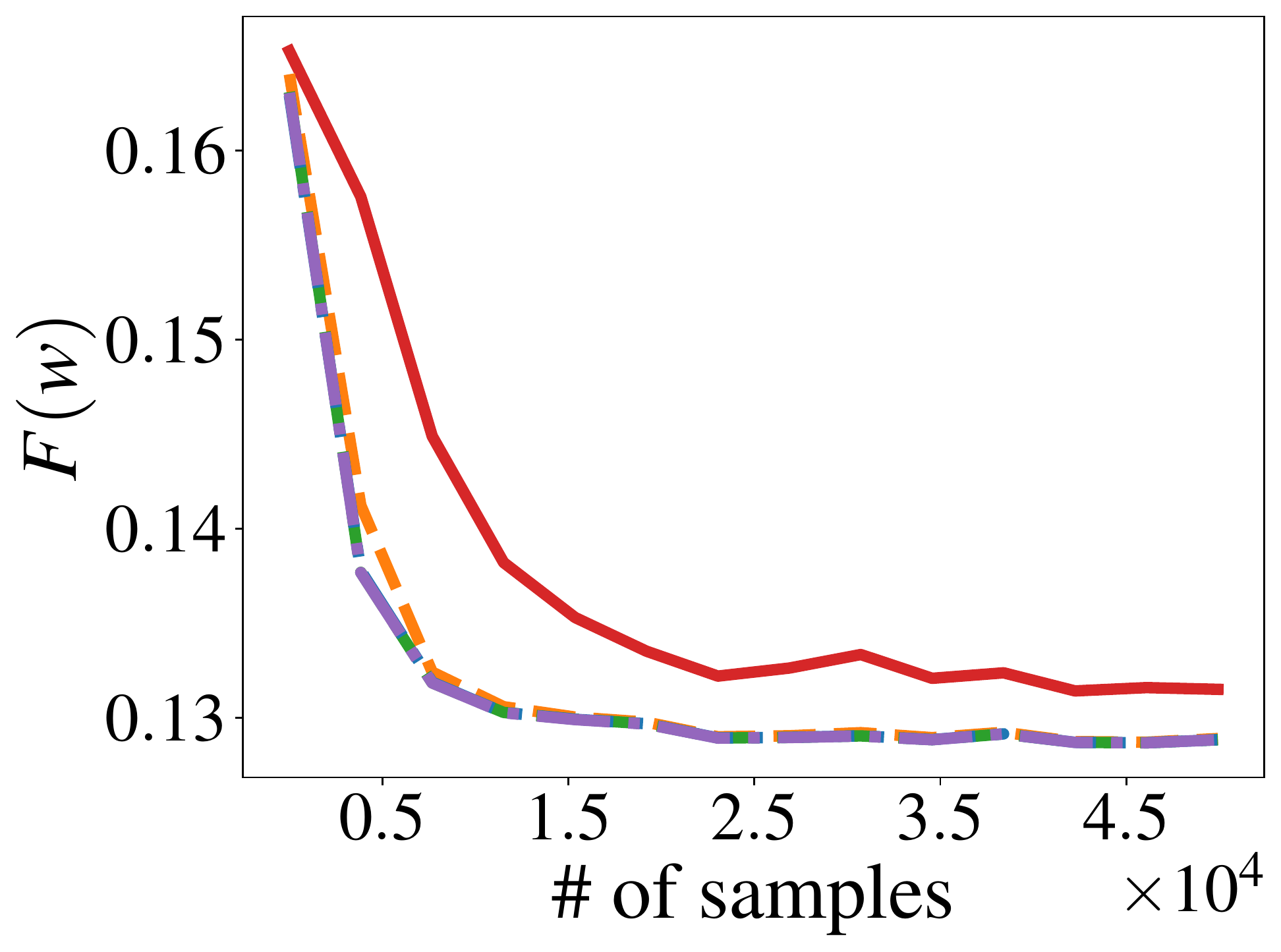}
		}%
		\setcounter{subfigure}{0}
		\subfigure[Industry-10]{
			\includegraphics[width=0.24\textwidth]{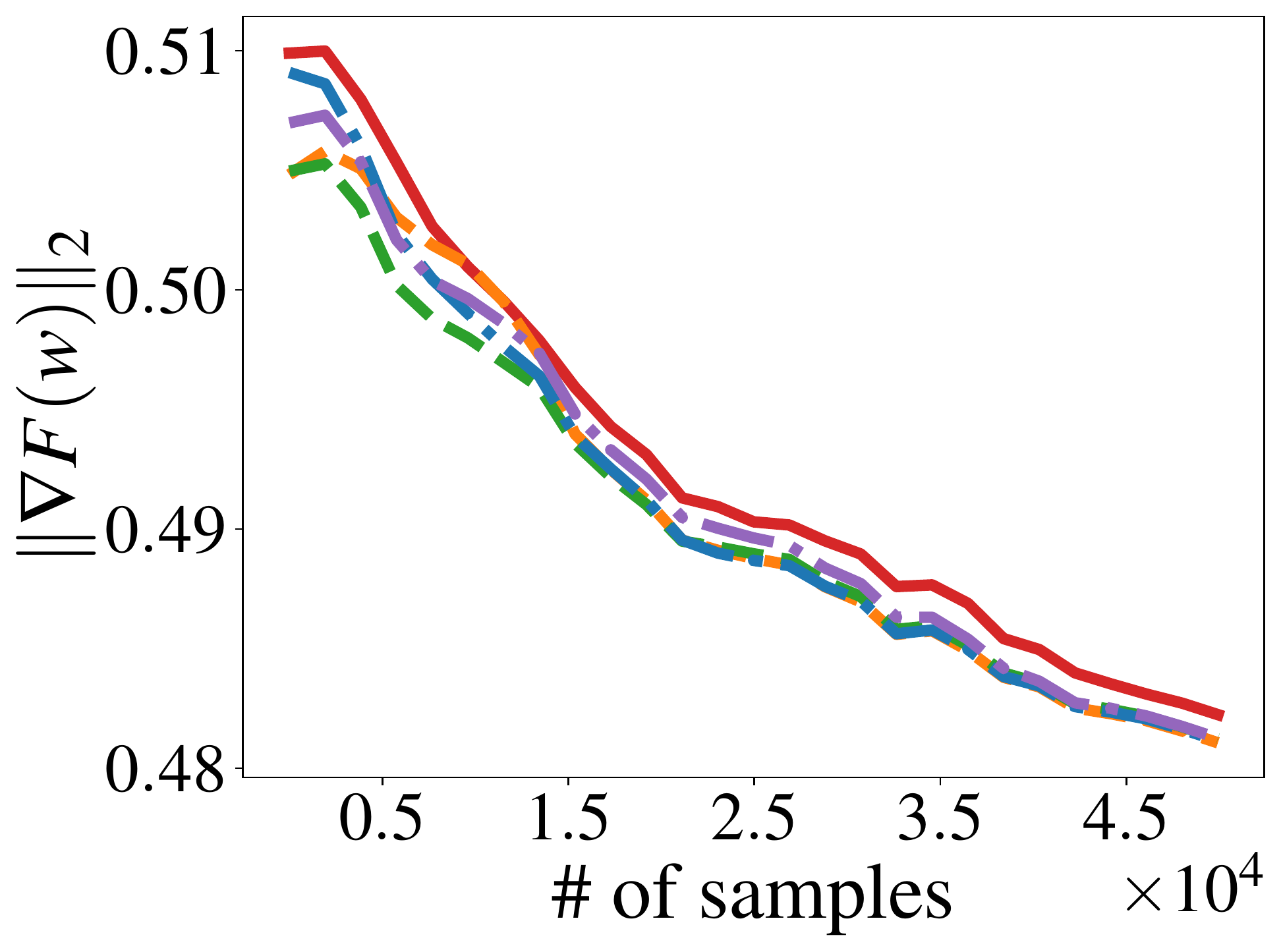}
		}
		\hspace{-3mm}
		\subfigure[Industry-12]{
			\includegraphics[width=0.24\textwidth]{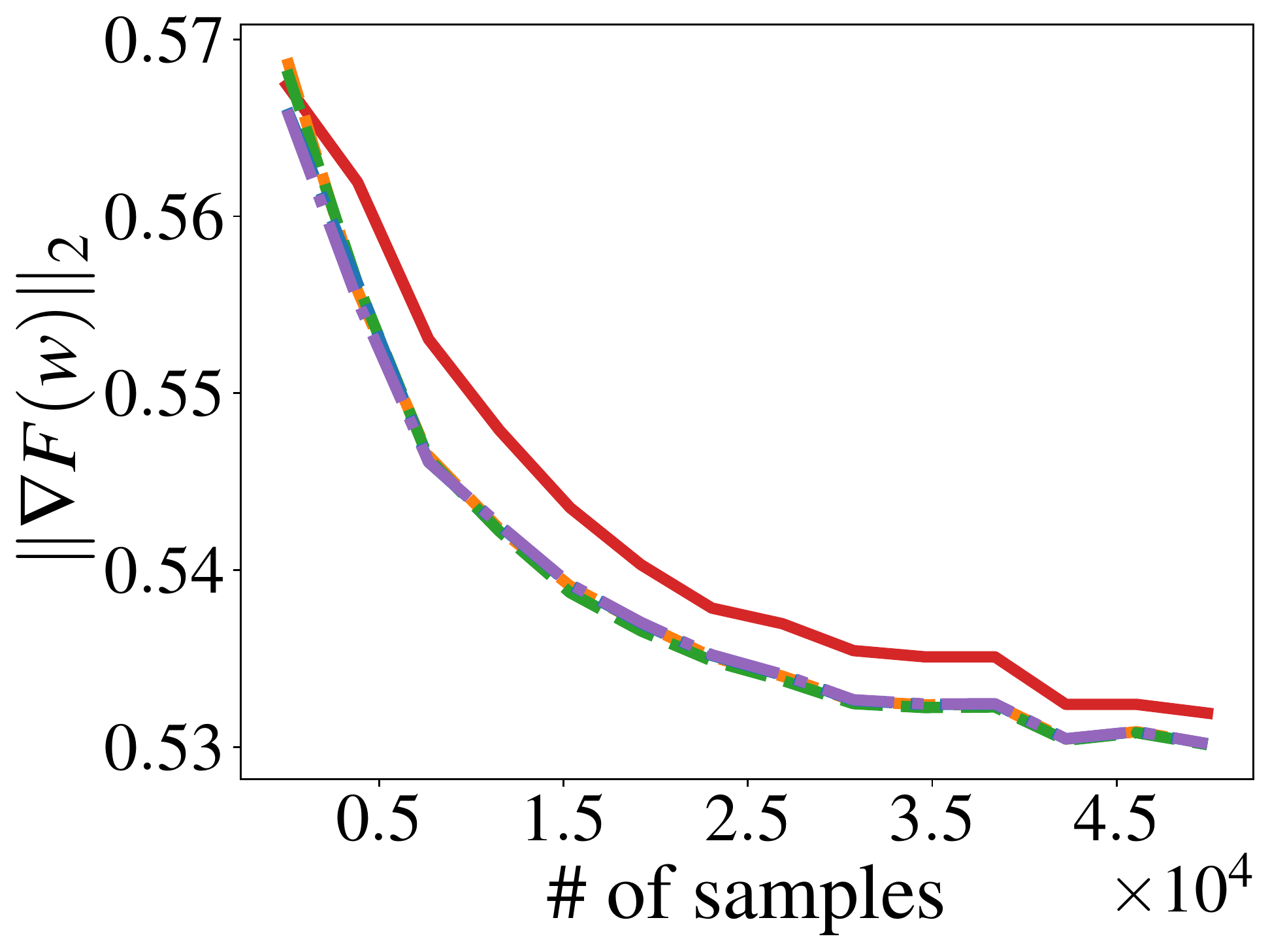}
		}
		\hspace{-3mm}
		\subfigure[Industry-17]{
			\includegraphics[width=0.24\textwidth]{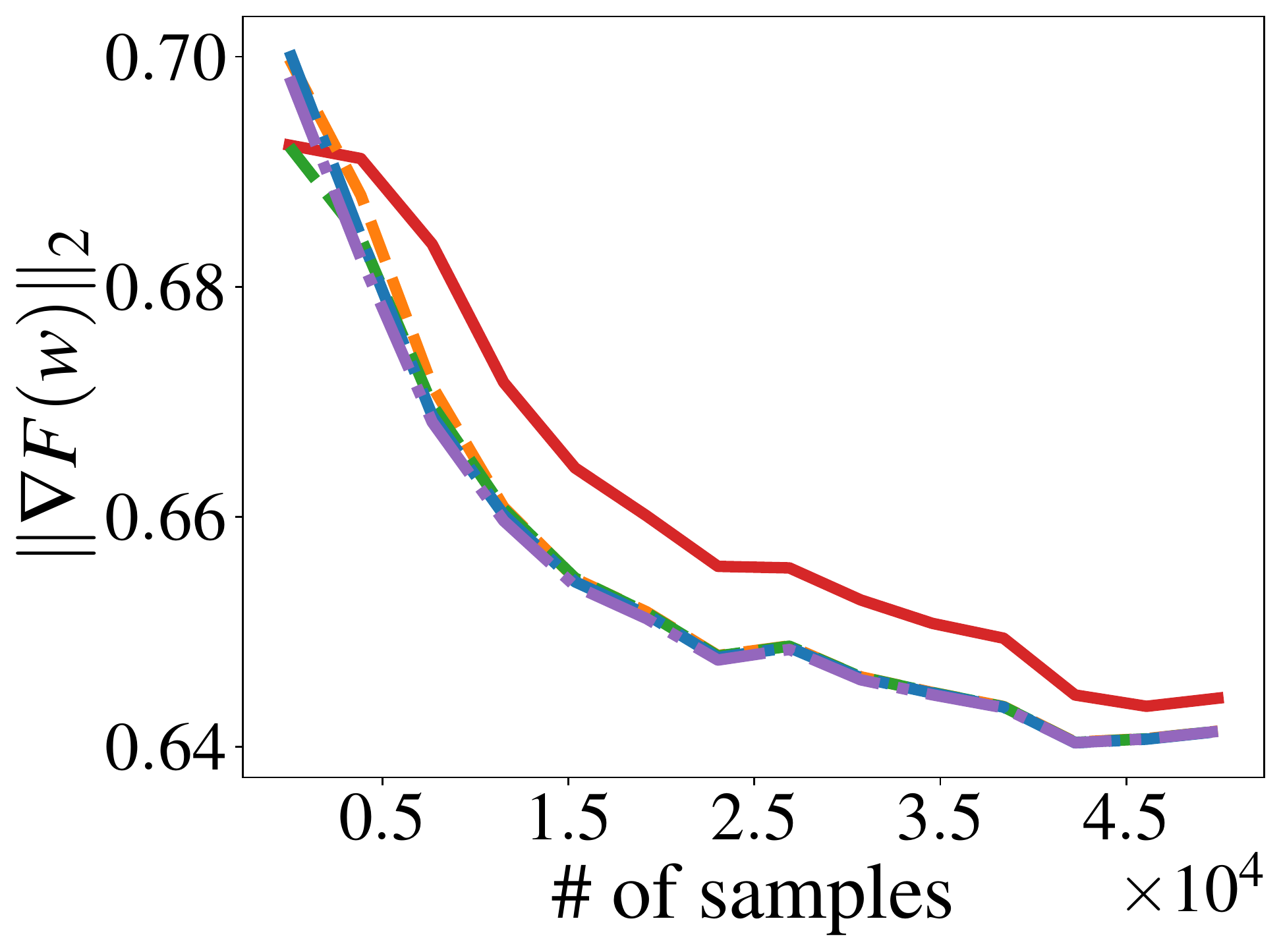}
		}
		\hspace{-3mm}
		\subfigure[Industry-30]{
			\includegraphics[width=0.24\textwidth]{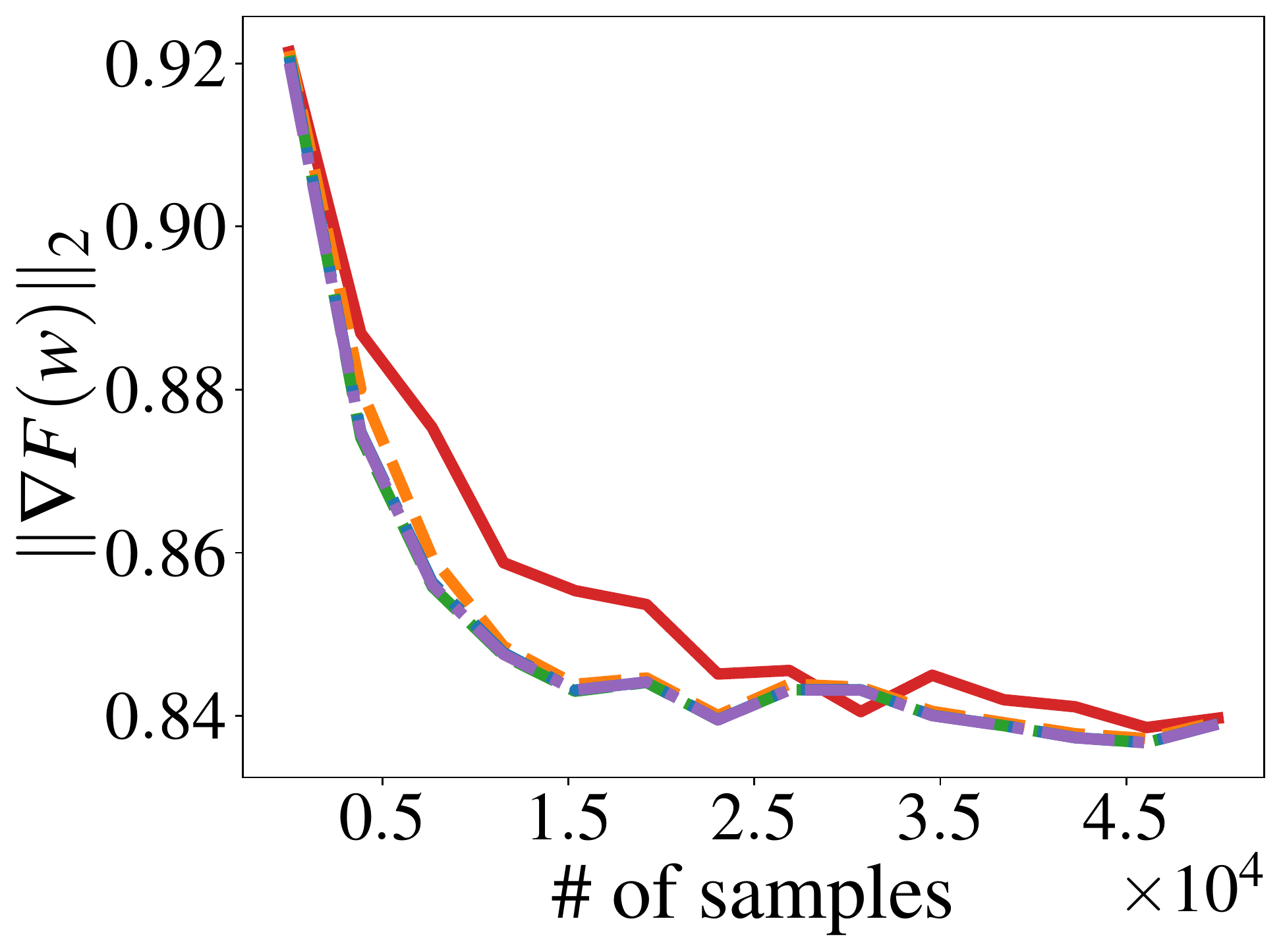}
		}%
		\vskip -0.05in
		\setcounter{subfigure}{1}
		\subfigure{
			\includegraphics[width=0.8\textwidth]{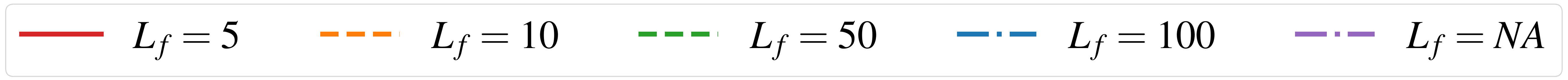}
		}%
		\vskip -0.05in
		\caption{Results for Risk-Averse Portfolio Optimization.}
		\label{fig:2}
	\end{center}
\vskip -0.2in
\end{figure*}	
\section{Experiments}\label{Sec:5}
In this section, we conduct numerical experiments to evaluate the performance of the proposed method over three different tasks. We compare our method with existing multi-level algorithms, including A-TSCGD \citep{Yang2019MultilevelSG}, NLASG \citep{balasubramanian2020stochastic}, Nested-SPIDER \citep{Zhang2021MultiLevelCS} and SCSC \citep{chen2021solving}. For our method, the parameter $\beta$ is searched from the set $\{0.1, 0.5, 0.9\}$. For other methods, we choose the hyper-parameters suggested in the original papers or use grid search to select the best hyper-parameters. When it comes to the learning rate, we tune it from the range $\{0.01,0.02,0.05,0.1,0.2,0.5,1\}$. As for the projection operation $\Pi_{L_f}$, we can simply set $L_f$ as a large value and we provide a sensitivity analysis in terms of tuning $L_f$ in the 
first experiment.

\subsection{Risk-Averse Portfolio Optimization}
We first consider the risk-averse portfolio optimization problem. Suppose we have $d$ assets to invest during each time step $\{1, \ldots, T\}$, and $r_{t} \in \mathbb{R}^{d}$ denotes the payoff of $d$ assets in the time step $t$. Our goal is to maximize the return of the investment and minimize the risk at the same time. One useful formulation is the mean-deviation risk-averse optimization problem~\citep{Shapiro2009LecturesOS}, where the risk is defined as the standard deviation. This mean–deviation model is widely used in practice and often used to conduct experiments in multi-level optimization~\citep{Yang2019MultilevelSG, Zhang2020OptimalAF}. The problem can be formulated as:
\begin{align*}
\max _{x \in \mathcal{X}} \frac{1}{T} \sum_{t=1}^{T}\left\langle r_{t}, x\right\rangle-\lambda \sqrt{\frac{1}{T} \sum_{t=1}^{T}\left(\left\langle r_{t}, x\right\rangle-\left\langle \bar{r}, x\right\rangle\right)^{2}},
\end{align*}
where $\bar{r}= \sum_{t=1}^{T} r_{t}$ and decision variable $x$ denotes the investment quantity vector in $d$ assets. This problem is a three-level stochastic compositional optimization problem, and each layer can be represented as:
\begin{align*}
     &f_{1}(x)=\left(\frac{1}{T} \sum_{t=1}^{T}\langle r_{t}, x\rangle, x\right), \\
     &f_{2}(y, x)=\left(y, \frac{1}{T} \sum_{t=1}^{T}\left(\left\langle r_{t}, x\right\rangle-y\right)^{2}\right), \\ 
     &f_{3}\left(z_{1}, z_{2}\right)=-z_{1}+\lambda \sqrt{z_{2}}.
\end{align*}
In the experiment, we test different methods on real-world datasets Industry-10, Industry-12, Industry-17 and Industry-30 from Keneth R. French Data Library\footnote{https://mba.tuck.dartmouth.edu/pages/faculty/ken.french/}. These datasets contain 10, 12, 17 and 30 industrial assets payoff over 25105 consecutive periods, respectively. Following \citet{Zhang2021MultiLevelCS}, we set parameter $\lambda=0.2$.

\cref{fig:1} shows the loss value and the norm of the gradient against the number of samples drawn by each method, and all curves are averaged over 20 runs. We can find that our method converges much faster than other algorithms in all tasks. Specifically, both the loss and the gradient of SMVR decrease more quickly, demonstrating the low sample complexity of the proposed method.

We also conduct experiments on tuning the parameter $L_f$ for the projection operation $\Pi_{L_f}$. For theoretical analysis, if $L_f$ is set bigger than the actual upper bound of the gradient, the convergence rate remains the same order, just with a bigger constant. Here, we tune the $L_f$ form the set $\{5, 10, 50, 100\}$, and the results are shown in Figure~\ref{fig:2}, where $L_f=$ NA  means the projection is not used (it equals to setting $L_f$ as an extremely large number, such as $1e7$). We find that the method performs very closely as long as $L_f$ is set as a large number and would perform worse when $L_f$ is small. In practical use, we can simply set $L_f$ as a large number.

\begin{figure*}[ht]
\vskip 0.2in
\begin{center}
	\subfigure{
    	\includegraphics[width=0.24\textwidth]{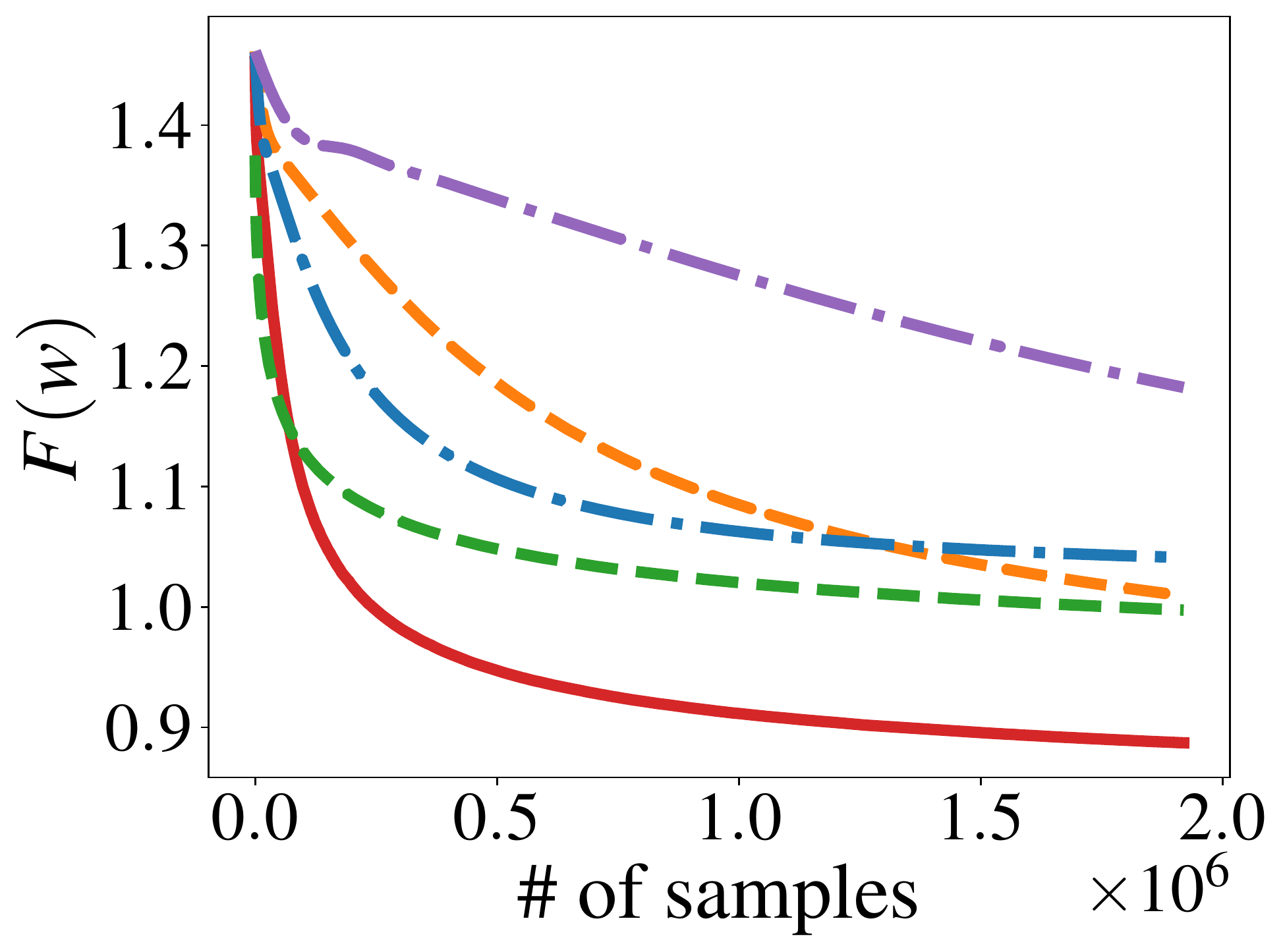}
		\includegraphics[width=0.24\textwidth]{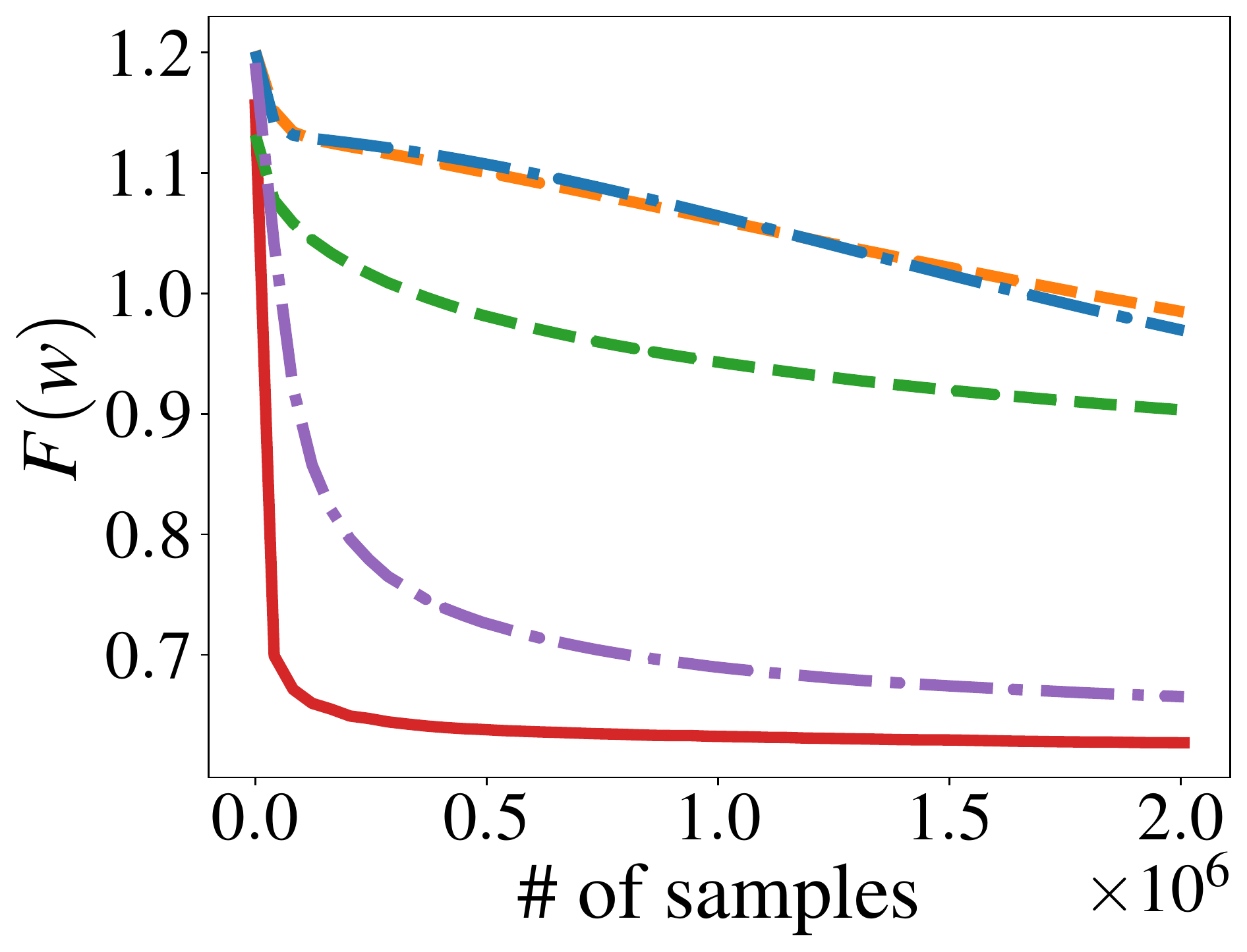}
		\includegraphics[width=0.24\textwidth]{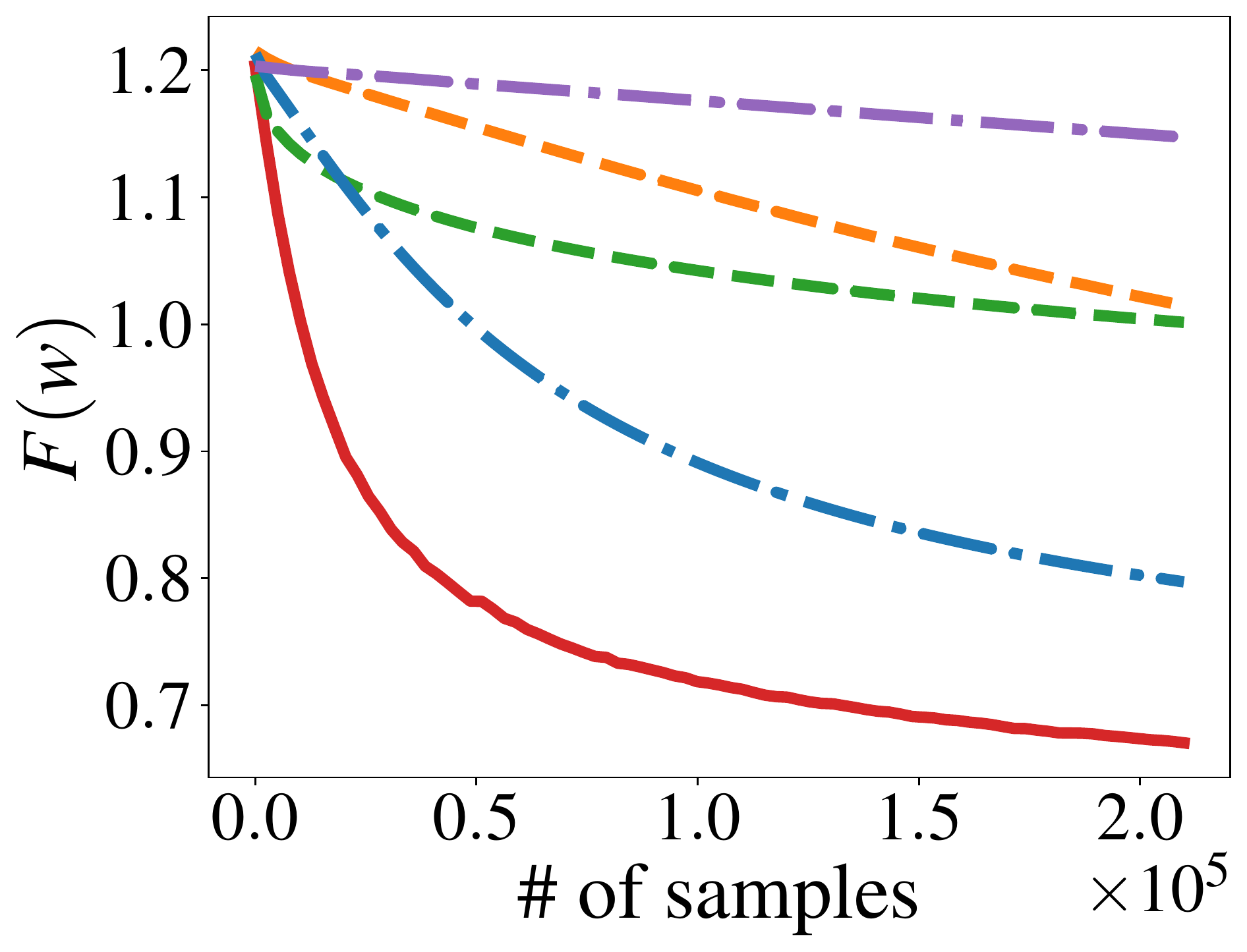}
		\includegraphics[width=0.24\textwidth]{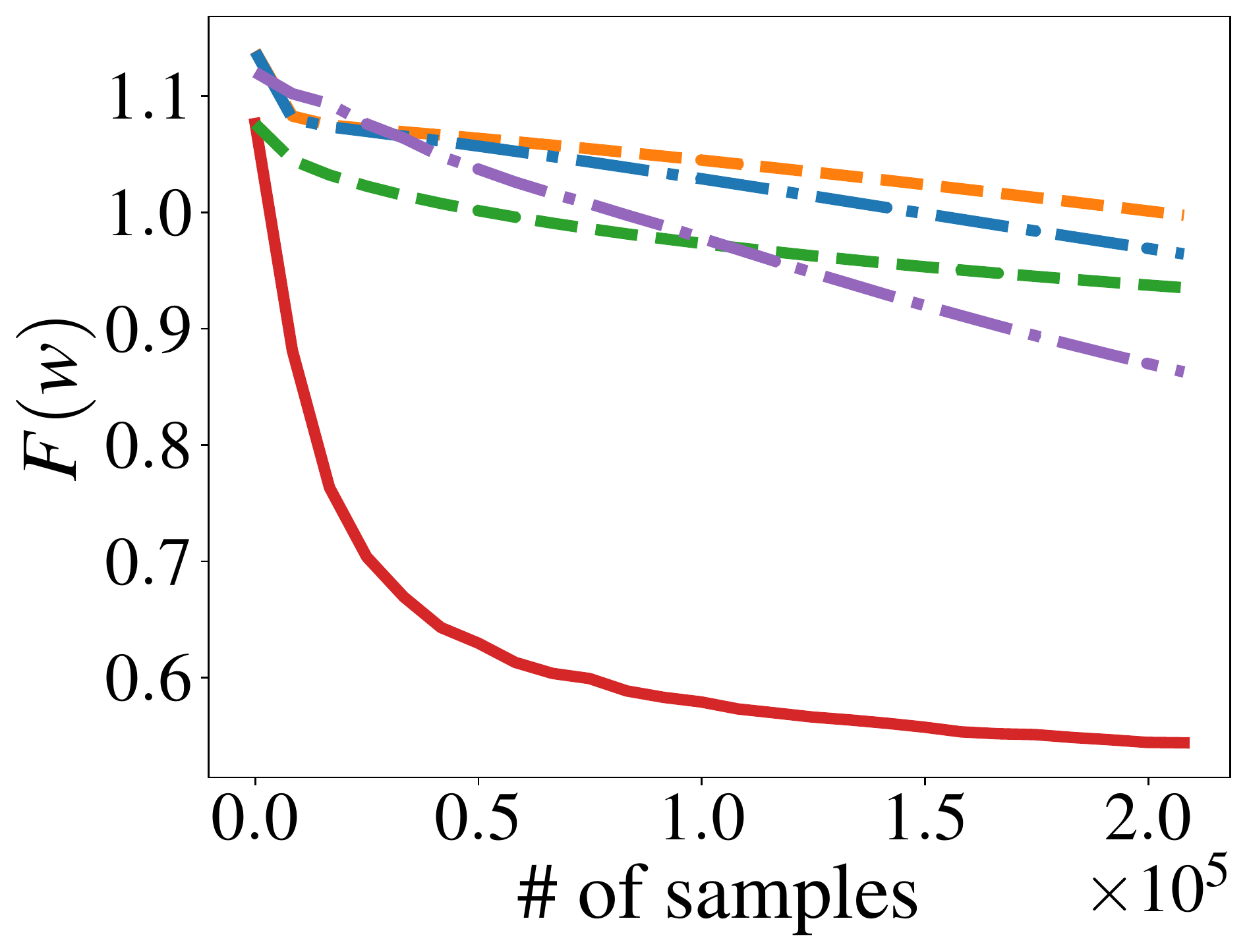}
	}%
	\setcounter{subfigure}{0}
	\subfigure[HIV-1]{
		\includegraphics[width=0.24\textwidth]{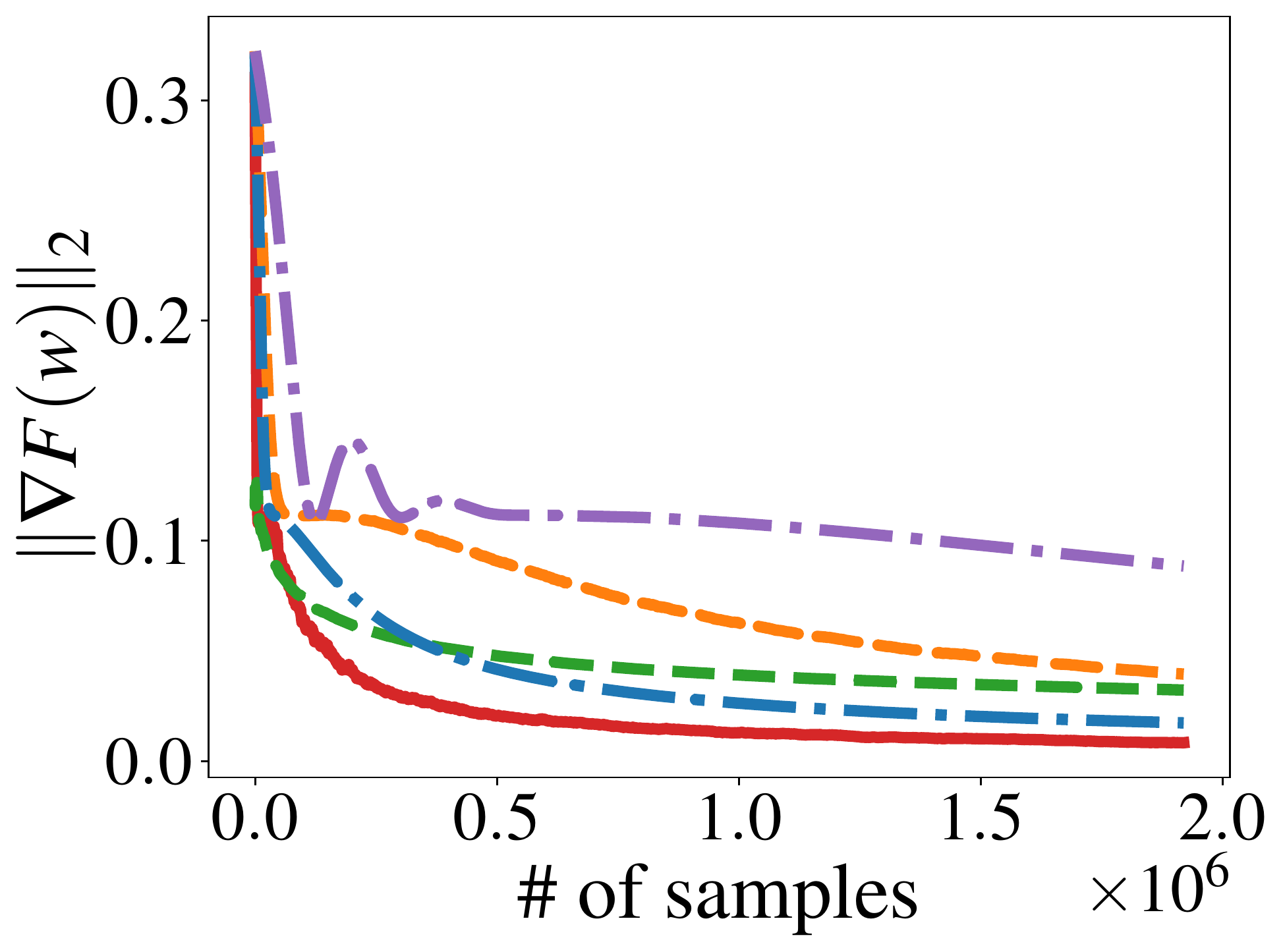}
	}
	\hspace{-3mm}
	\subfigure[Australian]{
		\includegraphics[width=0.24\textwidth]{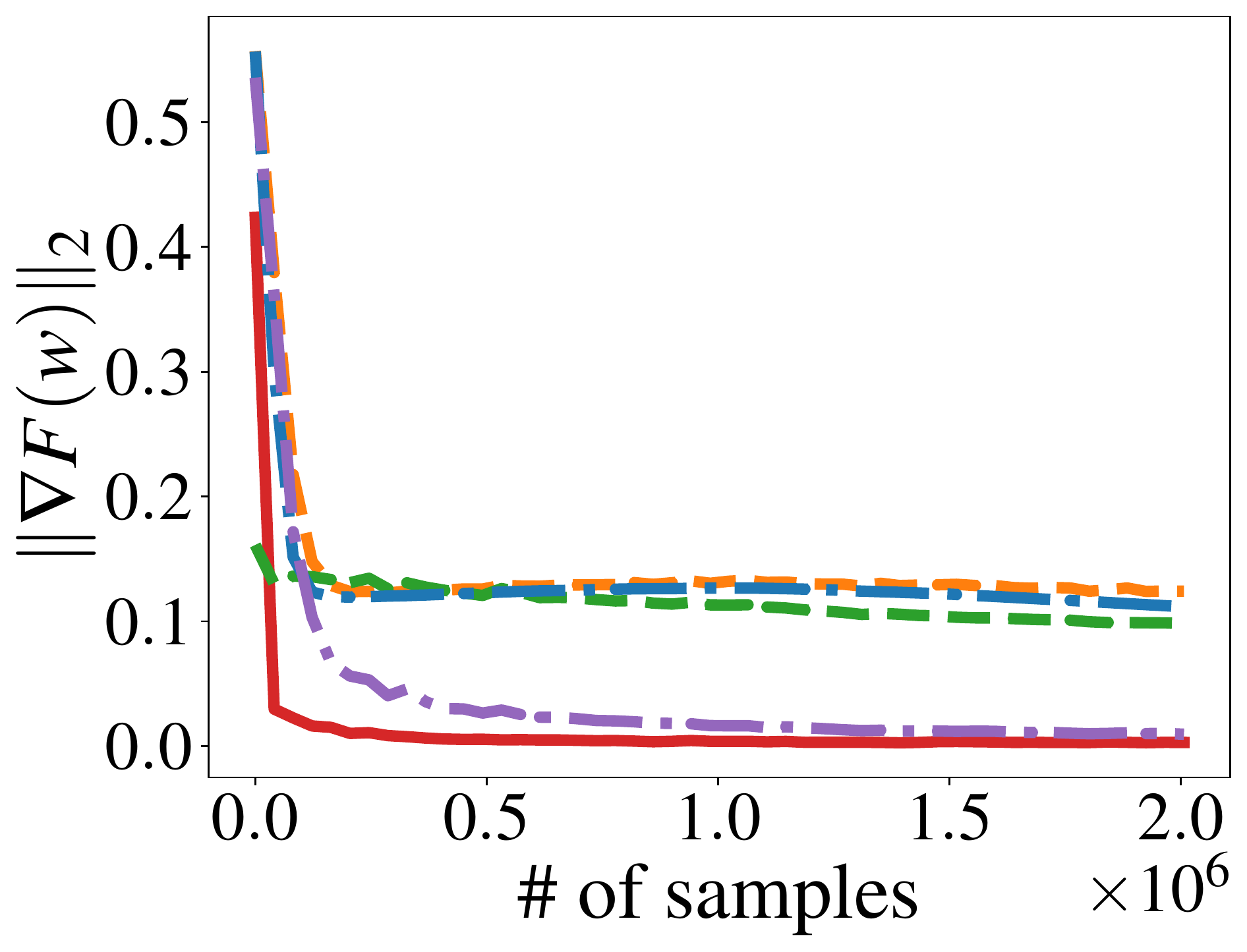}
	}
	\hspace{-3mm}
	\subfigure[Breast-cancer]{
		\includegraphics[width=0.24\textwidth]{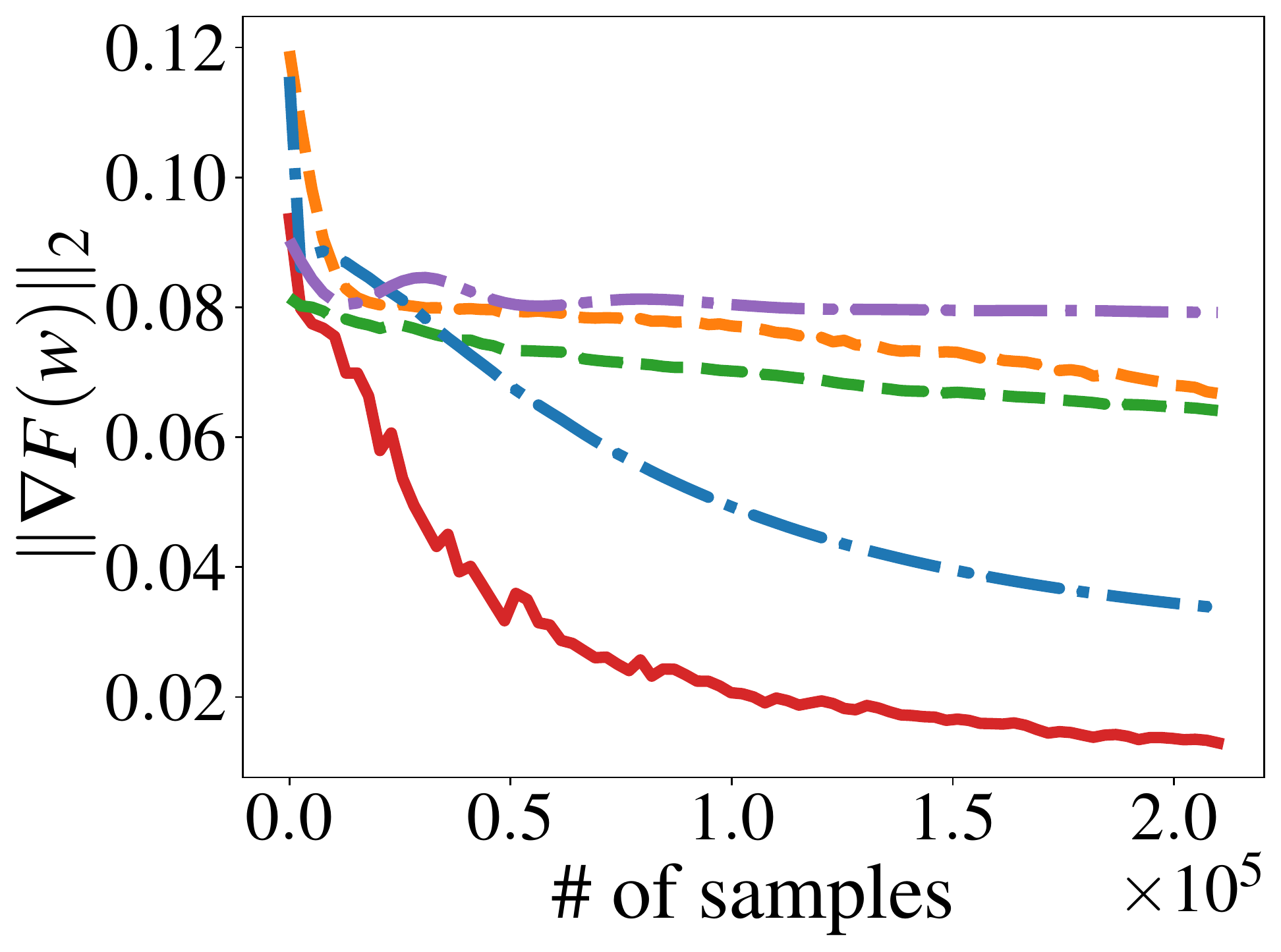}
	}
	\hspace{-3mm}
	\subfigure[Svmguide1]{
		\includegraphics[width=0.24\textwidth]{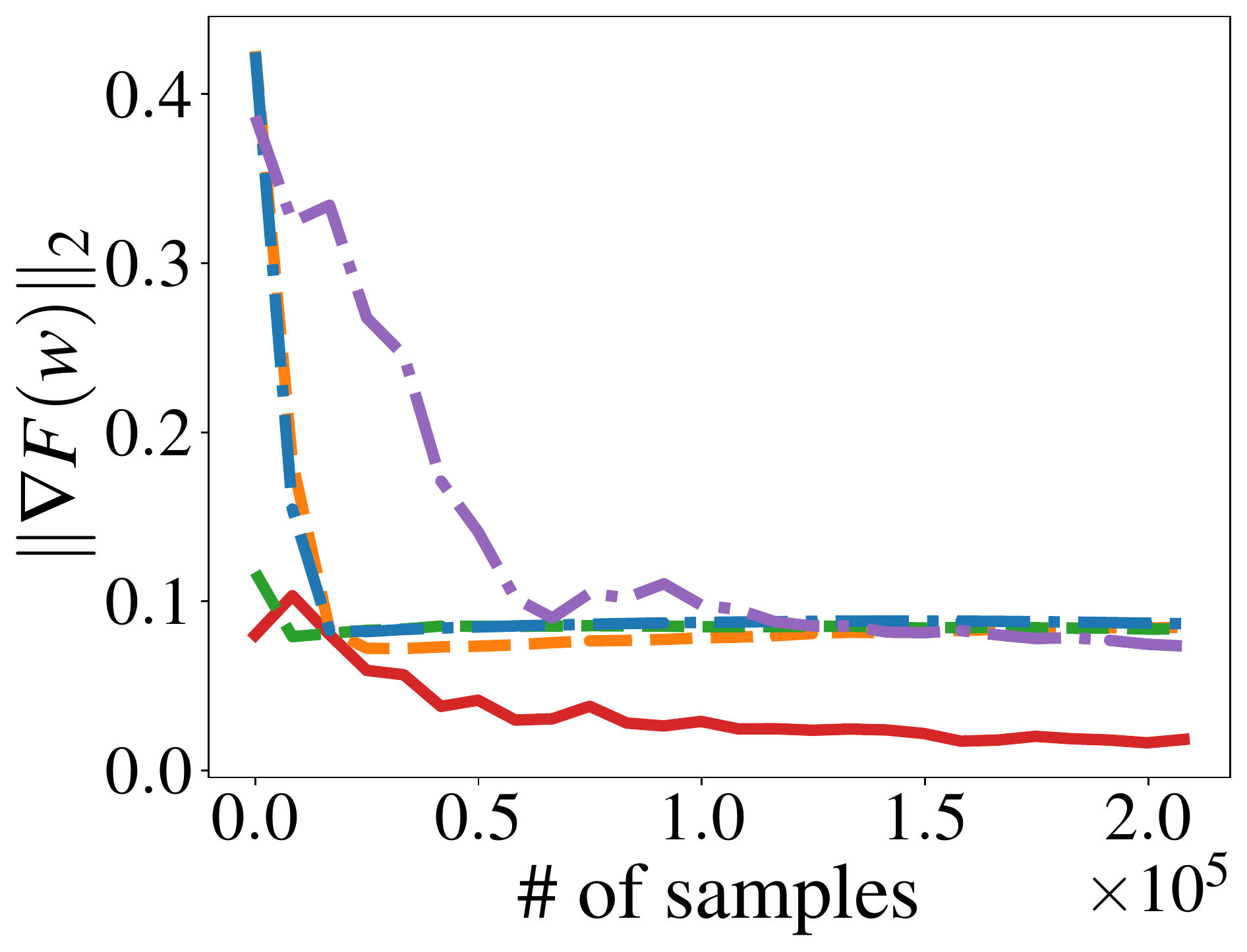}
	}
	\vskip -0.05in
	\setcounter{subfigure}{1}
	\subfigure{
		\includegraphics[width=0.8\textwidth]{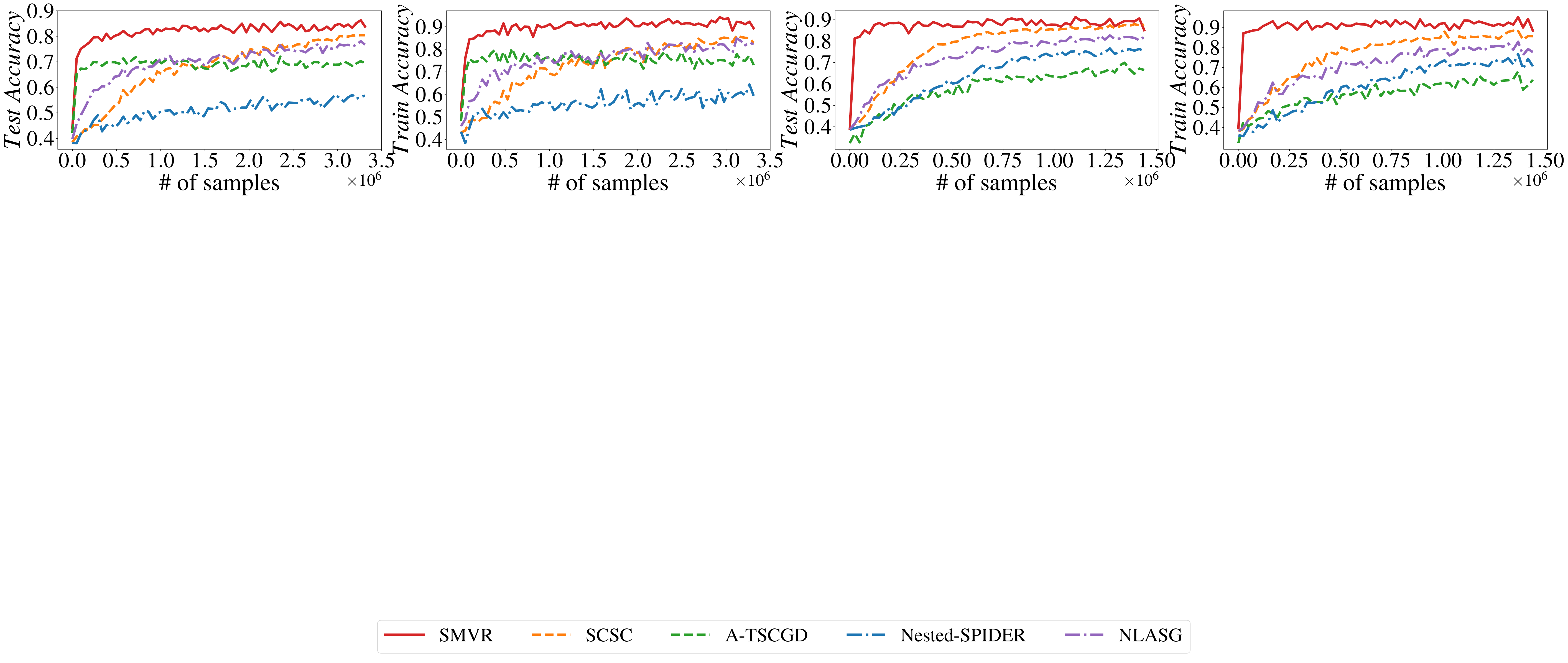}
	}%
	\vskip -0.05in
	\caption{Results for Hierarchical Tilted Empirical Risk Minimization.}
  	\label{fig:3}
\end{center}
\vskip -0.2in
\end{figure*}
\begin{table*}[t]
	\caption{Classification accuracies (\%) for Hierarchical Tilted Empirical Risk Minimization.}
	\label{sample-table}
	\vskip 0.15in
	\begin{center}
			\begin{sc}
			\resizebox{\textwidth}{!}{
				\begin{tabular}{l|cc|cc|cc|cr}
					\toprule
					\multirow{2}*{Method} & \multicolumn{2}{c|}{Hiv-1} & \multicolumn{2}{c|}{Australian scale} &\multicolumn{2}{c|}{Breast-cancer}  &\multicolumn{2}{c}{Svmguide1}\\
					\cline{2-9}				 &\text{rare}&\text{overall}&\text{rare}&\text{overall}&\text{rare}&\text{overall}&\text{rare}&\text{overall}\\
					\hline
					A-TSCGD    	& 71.3$\pm$ 2.9& 88.1$\pm$ 1.2& 62.7$\pm$ 8.9&  72.5$\pm$ 5.2 & 55.1$\pm$ 8.5 & 80.6$\pm$ 4.4 & 76.5$\pm$ 4.6 & 85.8$\pm$ 2.1\\
					SCSC 		& 76.7$\pm$ 2.4& 88.7$\pm$ 1.2& 61.8$\pm$ 8.4&  72.5$\pm$ 5.2 & 55.6$\pm$ 8.0 & 80.7$\pm$ 4.1 & 75.8$\pm$ 4.7 & 85.5$\pm$ 2.1\\
					Nested-SPIDER    	& 47.3$\pm$ 9.2& 63.0$\pm$ 6.7& 68.5$\pm$ 9.7&  72.9$\pm$ 4.8 & 61.8$\pm$ 8.6 & 83.3$\pm$ 6.6 & 74.3$\pm$ 5.4 & 84.8$\pm$ 2.4\\
					NLASG    	& 69.9$\pm$ 3.0& 88.1$\pm$ 1.3& 79.4$\pm$ 8.4&  78.7$\pm$ 5.8 & 42.0$\pm$ 8.5 & 76.7$\pm$ 5.2 & 73.3$\pm$ 5.2 & 85.0$\pm$ 2.2\\
					\hline
					\textbf{SMVR}    & \textbf{79.0 $\pm$ 2.5}& \textbf{90.0 $\pm$ 1.2}& \textbf{82.9 $\pm$ 8.2}& \textbf{82.6 $\pm$ 4.6}& \textbf{74.0 $\pm$ 8.5}& \textbf{87.0 $\pm$ 4.3} & \textbf{80.9 $\pm$ 3.0} & \textbf{87.3 $\pm$ 2.7}\\
					\bottomrule
				\end{tabular}}
			\end{sc}
	\end{center}
	\vskip -0.1in
\end{table*}

\subsection{Hierarchical Tilted Empirical Risk Minimization}
Hierarchical Tilted Empirical Risk Minimization (TERM) is a method proposed by \citet{li2020tilted, li2021tilted}, which can deal with noisy and imbalanced machine learning problems simultaneously. TERM objective is given by $\widetilde{R}(w):=\frac{1}{t} \log \left(\frac{1}{N} \sum_{i \in[N]} e^{t l\left(w;z_{i} \right)}\right)$, where $l\left(w;z_{i} \right)$ is the loss on the sample $z_{i}$ from data $\left\{z_{1}, \ldots, z_{N}\right\}$. It can mitigate outliers with parameter $t<0$ and handle class imbalance when $t>0$. When the task involves outliers and class imbalance at the same time, Hierarchical TERM can be used:
\begin{align*}
\tilde{J}(w):=\frac{1}{t} \log \left(\frac{1}{|\mathcal{D}|} \sum_{\mathcal{G} \subseteq \mathcal{D}}|\mathcal{G}| e^{t \tilde{R}_{\mathcal{G}}(w)}\right),\\
\text{with}  \ \  \widetilde{R}_{\mathcal{G}}(w):=\frac{1}{\tau} \log \left(\frac{1}{|\mathcal{G}|} \sum_{z \in \mathcal{G}} e^{\tau \ell(w ; z)}\right),
\end{align*}
where $\mathcal{D}$ represents all training samples and $\mathcal{G}$ denotes samples for one specific class. Parameter $t$ and $\tau$ are constants dealing with different goals (i.e., outliers and class imbalance). It is a four-level stochastic compositional optimization problem, with each layer represented as:
\begin{align*}
     f_{1}(w)=\frac{1}{|\mathcal{G}|} \sum_{z \in \mathcal{G}} e^{\tau \ell(w ; z)},\  f_{2}\left(x\right)=\frac{1}{\tau}\log(x), \\ \
     f_{3}\left(y\right)=\frac{1}{|\mathcal{D}|} \sum_{\mathcal{G}  \subseteq\mathcal{D}}|\mathcal{G}| e^{t y}, \ f_{4}\left(z\right)=\frac{1}{t}\log(z).
\end{align*}
In the experiment, we use the "HIV-1"\footnote{https://archive.ics.uci.edu/ml/datasets.php},  "Australian"\footnote{https://www.csie.ntu.edu.tw/\textasciitilde cjlin/libsvmtools/datasets/ \label{libsvm_dataset}}, "Breast-cancer"\textsuperscript{\ref {libsvm_dataset}} and "svmguide1"\textsuperscript{\ref {libsvm_dataset}} dataset, and make the training data noisy and imbalanced, where nearly $30\%$ of the labels are reshuffled and the number of rare class versus common class is 1:20. We set $\tau=-2$, $t=10$ according the origin paper and repeat each experiment 20 times. 
\begin{figure*}[!ht]
    \vskip 0.2in
  	\centering
	\subfigure[5-way 1-shot]{
			\includegraphics[width=0.24\textwidth]{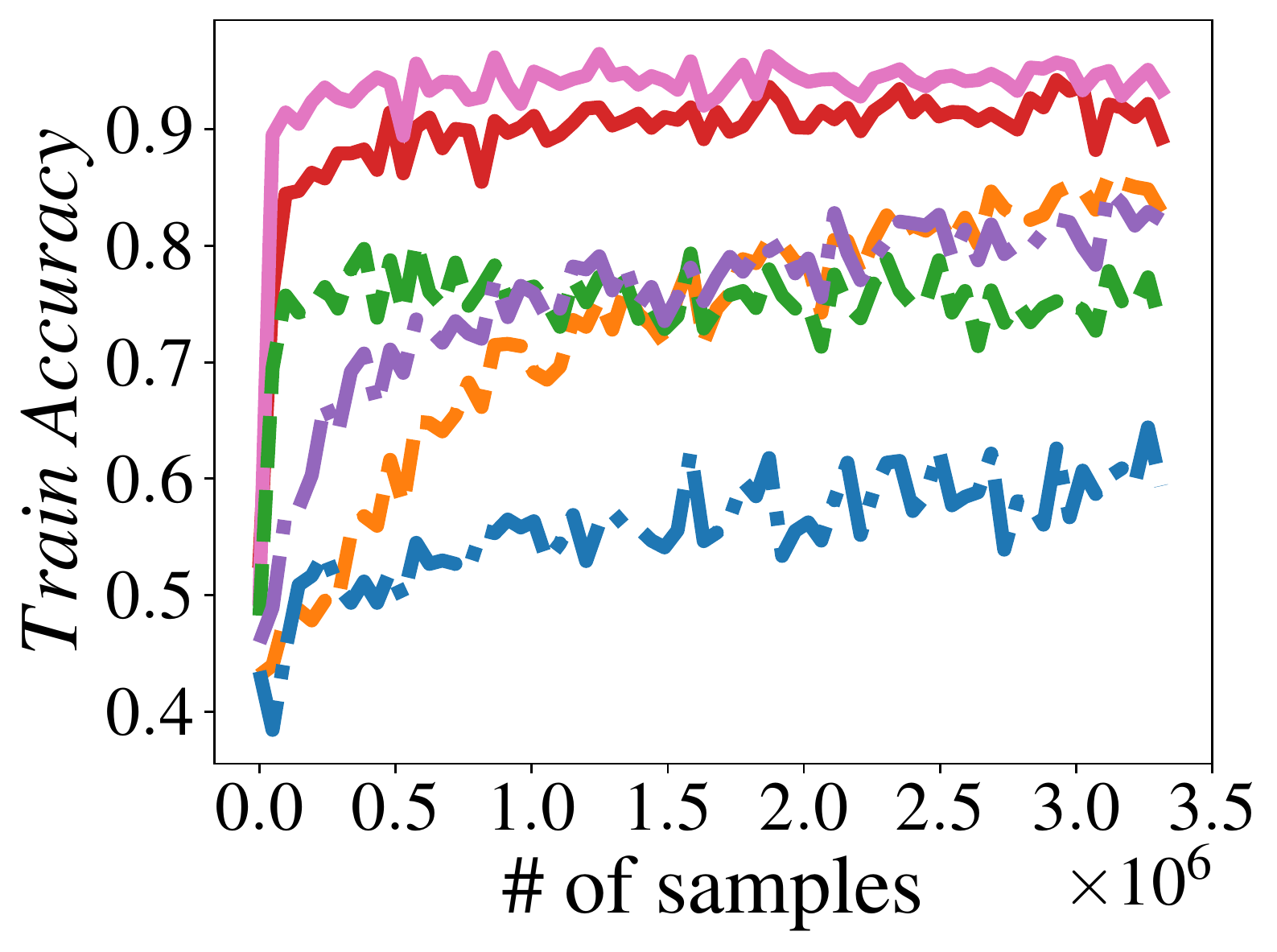}
			\includegraphics[width=0.24\textwidth]{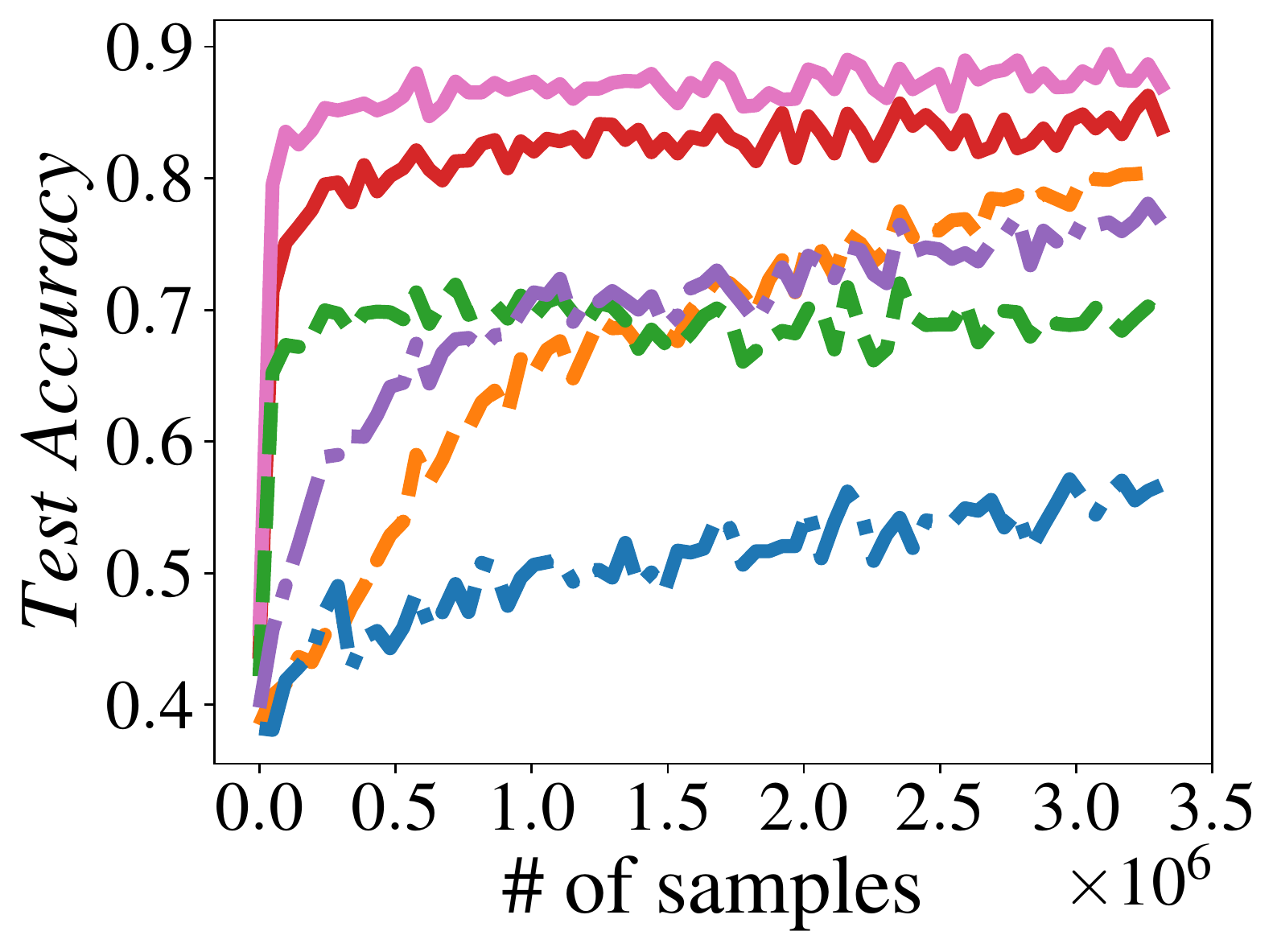}
	}%
	\subfigure[5-way 5-shot]{
		 	\includegraphics[width=0.24\textwidth]{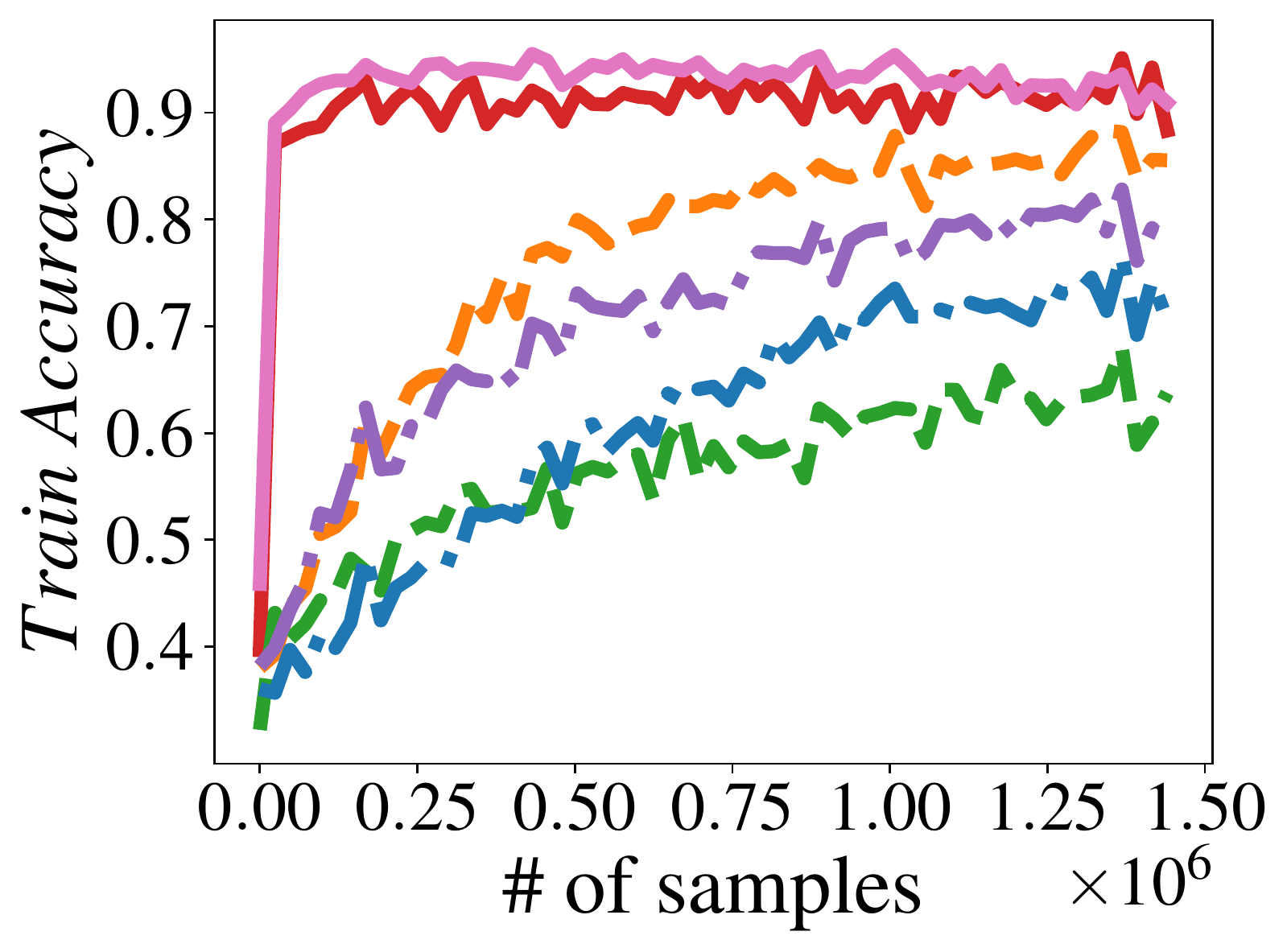}
		 	\includegraphics[width=0.24\textwidth]{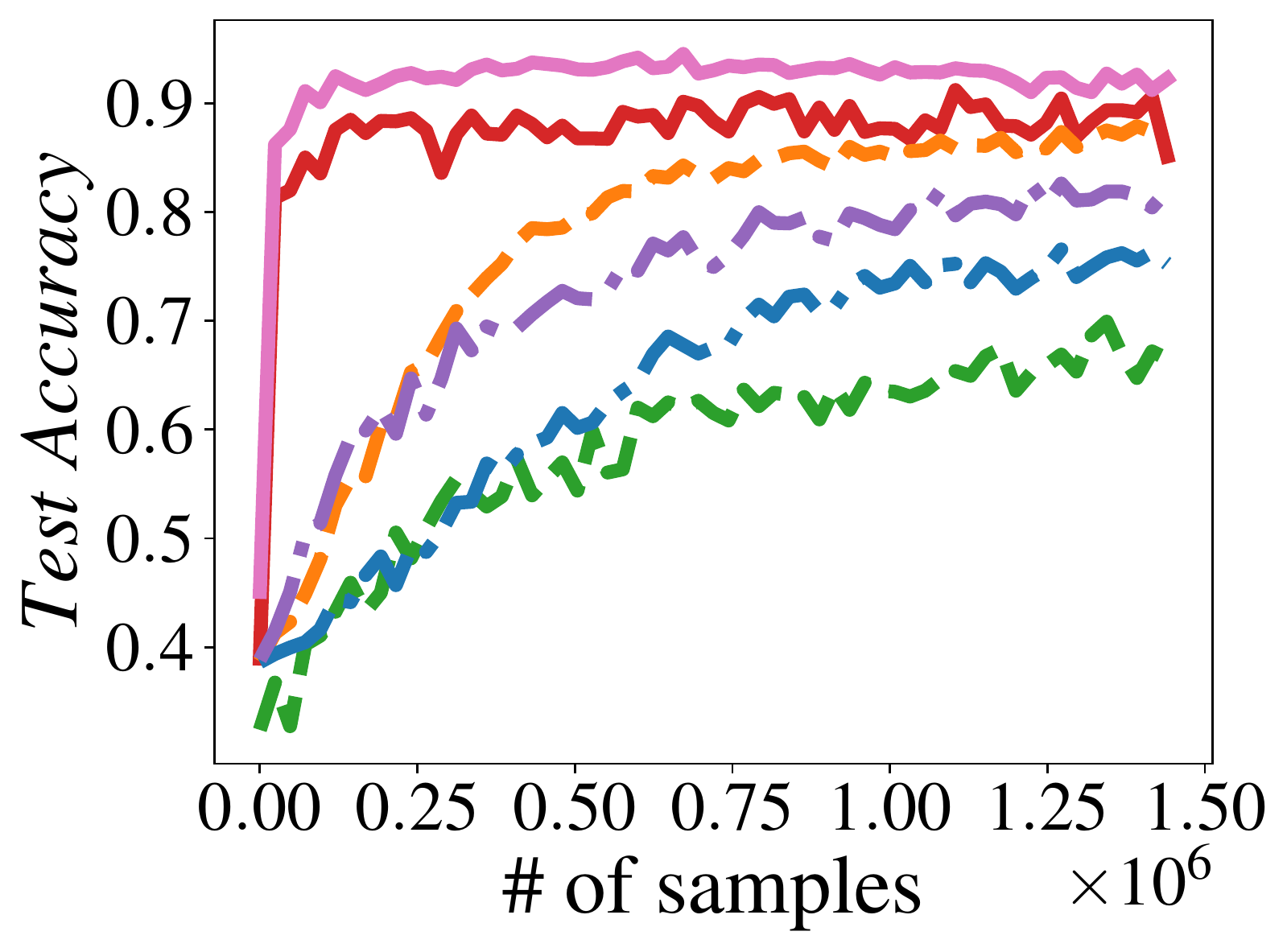}
 	}%
 	\vskip -0.05in
 	\setcounter{subfigure}{1}
 	\subfigure{
 			\includegraphics[width=0.95\textwidth]{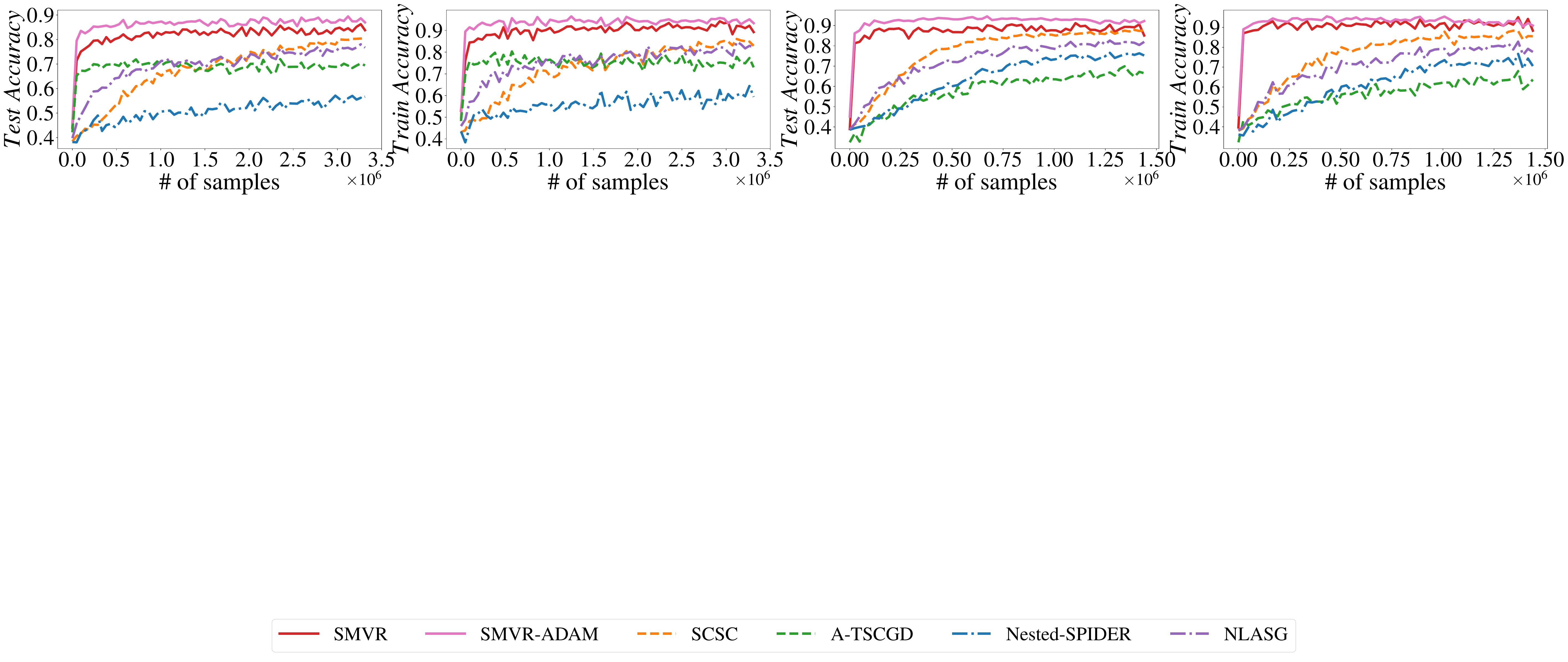}
 	}%
 	\vskip -0.05in
	\caption{Results for Multi-step Model-Agnostic Meta-Learning.}
	\label{fig:MAML}
	\vskip -0.2in
\end{figure*}
As shown in \cref{fig:3}, our SMVR performs best among all algorithms. The loss value and the norm of the gradient converge to a small value more quickly than other methods. We also report the classification accuracy in Table~\ref{sample-table}. It shows that SMVR achieves the highest accuracy on the rare class and the overall task simultaneously, indicating the effectiveness of our method.

\subsection{ Multi-Step Model-Agnostic Meta-Learning}
At last, we conduct experiments on Multi-step Model-Agnostic Meta-Learning (MAML). Multi-step MAML aims to find a good initialization point that performs well in different tasks after taking a few steps of gradient descent. Classical one-step MAML can be formed as:
\begin{align*}
    \min _{\boldsymbol{\x} } F(\boldsymbol{\theta})&:=\frac{1}{M} \sum_{m=1}^{M} F_{m}\left(\boldsymbol{\x}-\alpha \nabla F_{m}(\boldsymbol{\x})\right),\\
    \text{with} \ \ F_{m}(\boldsymbol{\theta})&:=\mathbb{E}_{\xi_{m}}\left[f\left(\boldsymbol{\theta} ; \xi_{m}\right)\right]
\end{align*}
where $\alpha$ is the learning rate, $F_{m}$ denotes the loss on task $m$ and $\xi_{m}$ represents the training samples for task $m$. One-step MAML is a two-level problem since it updates the initial point once and then evaluates on different tasks. In practice, we usually update the initial point more times to achieve better results. For example, \citet{Finn2017ModelAgnosticMF} use five-step updates, which is a six-level compositional problem.

Following \citet{Finn2017ModelAgnosticMF}, we conduct experiments on 5-way 1-shot and 5-shot task on Omniglot dataset \citep{Lake2011OneSL}. Each task is a 5-class classification problem, with only 1 or 5 training samples for each class. We conduct 5-step MAML and repeat each experiment 3 times.

We report the accuracy of different methods against the number of training samples in \cref{fig:MAML}. Since adaptive learning rates are widely used in neural networks, which are also applied in Multi-step MAML, we implement Adaptive SMVR methods in this task, denoted as SMVR-ADAM. We use the adaptive learning rate defined in (\ref{rule1}) and (\ref{rule2}) and choose the commonly used Adam-type. As can be seen, the accuracy of SMVR and SMVR-ADAM increases rapidly both in training sets and testing sets, and outperforms other methods dramatically. Although SMVR and SMVR-ADAM enjoy the same sample complexity, the latter converges faster in practice due to the adaptive learning rate used.

\section{Conclusion}\label{Sec:6}
In this paper, we propose an optimal algorithm named SMVR for stochastic multi-level composition optimization. We prove that the proposed algorithm, by using variance reduced estimator of function values and Jacobians, can achieve the sample complexity of $\mathcal{O}\left(1 / \epsilon^{3}\right)$ for finding an $\epsilon$-stationary point. This complexity matches the lower bound even in the one-level setting, and our method avoids using batches in any iterations. When the objective function further satisfies the convexity or PL condition, we develop a stage-wise version of SMVR to obtain the optimal complexity of $\mathcal{O}\left({1}/{\epsilon^2}\right)$ or $\mathcal{O}\left({1}/{\epsilon}\right)$. To take advantage of adaptive learning rates, we also propose Adaptive SMVR, which can achieve the same complexity with the learning rate changing adaptively. Experiments on three real-world tasks demonstrate the superiority of the proposed method.

\section*{Acknowledgements}
W. Jiang, Y. Wang and L. Zhang were partially supported by NSFC (62122037, 61921006). B. Wang and T. Yang were partially supported by NSF Grant 2110545, NSF Career Award 1844403. The authors would like to thank the anonymous reviewers for their helpful comments.
\bibliography{ref}
\bibliographystyle{icml2022}

\newpage
\appendix
\onecolumn

\section{Proof of \cref{thm:main}}
We first provide some supporting lemmas and then conclude to show the sample complexity of the proposed method.
\begin{lemma}\label{lem:2} The objective function F is $L_F$-smooth, where $L_F \coloneqq L_f^{2K-1}L_J\sum_{i=1}^K\frac{1}{L_f^i}$.
\end{lemma}
\begin{proof}
Denote $F_{i}(x)=f_{i} \circ f_{i-1} \circ \cdots \circ f_{1}(x)$. We know $F_{i}(x)$ is $L_f^{i}$ Lipschitz continuous, since when $i \geq 2$:
\begin{align*}
\left\|F_{i}(x)-F_{i}(y)\right\| &=\left\|f_{i}\left(F_{i-1}(x)\right)-f_{i}\left(F_{i-1}(y)\right)\right\| \\ & \leq L_{f}\left\|F_{i-1}(x)-F_{i-1}(y)\right\| \\&\leq L_{f}^{i-1}\left\|F_{1}(x)-F_{1}(y)\right\| \\ &\leq L_{f}^{i}\|x-y\| \quad 
\end{align*}
Then, we have:
\begin{align*}
&\left\|\nabla F_{K}(x)-\nabla F_{K}(y)\right\| \\
=&\left\|\nabla f_{K}\left(F_{K-1}(x)\right) \nabla F_{K-1}(x)-\nabla f_{K}\left(F_{K-1}(y)\right) \nabla F_{K-1}(y)\right\| \\ =&\left\|\nabla f_{K}\left(F_{K-1}(x)\right)\left(\nabla F_{K-1}(x)-\nabla F_{K-1}(y)\right)+\nabla F_{K-1}(y)\left(\nabla f_{K}\left(F_{K-1}(x)\right)-\nabla f_{K}\left(F_{K-1}(y)\right)\right)\right\| \\ 
\leq&\left\|\nabla f_{K}\left(F_{K-1}(x)\right)\right\|\left\|\nabla F_{K-1}(x)-\nabla F_{K-1}(y)\right\|+\left\|\nabla F_{K-1}(y)\right\|\left\|\nabla f_{K}\left(F_{K-1}(x)\right)-\nabla f_{K}\left(F_{K-1}(y)\right)\right\| \\ 
\leq& L_{f}\left\|\nabla F_{K-1}(x)-\nabla F_{K-1}(y)\right\|+L_{f}^{K-1} L_{J}\left\|F_{K-1}(x)-F_{K-1}(y)\right\| \\ 
\leq& L_{f}\left\|\nabla F_{K-1}(x)-\nabla F_{K-1}(y)\right\|+L_{f}^{2 K-2} L_{J}\Norm{x-y} \\  \leq& L_{f}^{2}\left\|\nabla F_{K-2}(x)-\nabla F_{K-2}(y)\right\|+\left(L_{f}^{2 K-3} L_{J}+L_{f}^{2 K-2} L_{J}\right)\Norm{x-y} \\ 
\leq& L_{f}^{K-1}\left\|\nabla F_{1}(x)-\nabla F_{1}(y)\right\|+ L_{f}^{2 K-1} L_{J} \sum_{i=1}^{K-1} \frac{1}{L_{f}^{i}}\Norm{x-y}\\ 
\leq& L_{f}^{2 K-1} \frac{1}{L_{f}^{K}} L_{J}\Norm{x-y}+L_{f}^{2 K-1} L_{J} \sum_{i=1}^{K-1} \frac{1}{L_{f}^{i}}\Norm{x-y} \\ 
=&L_{f}^{2 K-1} L_{J} \sum_{i=1}^{K} \frac{1}{L_{f}^{i}}\Norm{x-y}
\end{align*}
\end{proof}
A similar result can also be found in Lemma 2.1 of  \citep{balasubramanian2020stochastic}. This property is unsurprising since the composition of two smooth and Lipschitz functions is still smooth and Lipschitz.

\begin{lemma}\label{lem:3}
	Let $ \eta_{t} \leq \frac{1}{2L_F}$, we have following guarantee:
	\begin{align*}
		F(\w_{t+1}) \leq F(\w_t) + \frac{\eta_{t} }{2}\Norm{\nabla F(\w_t) - \v_t}^2 - \frac{\eta_{t}}{2}\Norm{\nabla F(\w_t)}^2 - \frac{\eta_{t}}{4}\Norm{\v_t}^2.
	\end{align*}
\end{lemma}	
\begin{proof}
Due to the smoothness of function $F$ and the definition of $\w_{t+1}$, we have:
\begin{align*}
F\left(\w_{t+1}\right) & \leq F\left(\w_{t}\right)+\left\langle\nabla F(\w_{t}), \w_{t+1}-\w_{t}\right\rangle+\frac{L_{F}}{2}\left\|\w_{t+1}-\w_{t}\right\|^{2} \\ &=F(\w_{t})-\eta_t\left\langle\nabla F(\w_{t}), \v_{t}\right\rangle+\frac{\eta_t^{2} L_{F}}{2}\|\v_{t}\|^{2} \\ & \leq F(\w_{t})-\eta_t\langle\nabla F(\w_{t}), \v_{t}\rangle+\frac{\eta_t}{2}\|\nabla F(\w_{t})\|^{2}+\frac{\eta_t}{2}\left\|\v_{t}\right\|^{2}-\frac{\eta_t}{2}\|\nabla F(\w_{t})\|^{2}-\frac{\eta_t}{2}\left\|\v_{t}\right\|^{2} +\frac{\eta_t}{4}\left\|\v_{t}\right\|^{2}\\ &=F(\w_{t})+\frac{\eta_t}{2}\|\nabla F(\w_{t})-\v_{t}\|^{2}-\frac{\eta_t}{2}\|\nabla F(\w_{t})\|^{2}-\frac{\eta_t}{4}\left\|\v_{t}\right\|^{2} 
\end{align*}
\end{proof}
We can derive the term $\Norm{\nabla F(\w_t)}^2$ from the right side. Since $F(\w_{t+1})$ and $F(\w_{t})$ can be eliminated when summing up, the only thing need to show is that the estimated gradient $\v_t$ is not too far away from the true gradient $\nabla F(\w_t)$. We prove this in the following lemmas.

\begin{lemma}\label{lem:9} 
Denote that $\nabla \widehat{F}_K(\w_t)\coloneqq \prod_{i=1}^K\nabla f_i(\u_t^{i-1})$ and $C_i \coloneqq L_f^{K-1} L_J\sum_{j=1}^{K-i-1}{L_f^j}$. For $K\geq 2$, we have:
	\begin{align*}
        \Norm{\nabla F(\w_t) - \nabla \widehat{F}_K(\w_t)}^2  \leq  K \sum_{i=1}^{K-1}C_i^2\Norm{f_i(\u^{i-1}_t) - \u_t^i}^2.
	\end{align*}
\end{lemma}
\begin{proof}
	First, we define that $\y_t^i \coloneqq f_i\circ f_{i-1}\circ \dotsc \circ f_1(\w_t)$, $\nabla \widehat{F}_i(\w_t) \coloneqq \nabla f_1(\w_t) \cdots \nabla f_i(\u^{i-1}_t)$.  	
	\begin{align*}
		\Norm{\nabla F_1(\w_t) - \nabla \widehat{F}_1(\w_t)} &= 0,\\
		\Norm{\nabla F_2(\w_t) - \nabla \widehat{F}_2(\w_t)} &= \Norm{\nabla f_1(\w_t) \nabla f_2(\y^1_t) - \nabla f_1(\w_t) \nabla f_2(\u^1_t)}\\
		&\leq  L_f L_J\Norm{\y_t^1 - \u^1_t},	\\
		\Norm{\nabla F_3(\w_t) - \nabla \widehat{F}_3(\w_t)} &= \Norm{\nabla f_1(\w_t) \nabla f_2(\y^1_t) \nabla f_3(\y^2_t) -  \nabla f_1(\w_t) \nabla f_2(\u^1_t) \nabla f_3(\u^2_t)}\\
		&\leq L_f^2 L_J\left(\Norm{\y_t^2 - \u_t^2}  +  \Norm{\y_t^1 - \u^1_t}\right),\\
		& \cdots\\
		\Norm{\nabla F_K(\w_t) - \nabla\widehat{F}_K(\w_t)}  &\leq L_f^{K-1}L_J \sum_{i=1}^{K-1} \Norm{\y_t^i - \u_t^i},
	\end{align*}
	Besides, we also have
	\begin{align*}
		\Norm{\y_t^2 - \u_t^2} &=\Norm{f_2\circ f_1(\w_t) - \u^2_t} \leq \Norm{f_2\circ f_1(\w_t) - f_2(\u^1_t)} + \Norm{f_2(\u^1_t)- \u^2_t}\\
		& \leq L_f \Norm{f_1(\w_t) - \u^1_t} + \Norm{f_2(\u^1_t)- \u^2_t},\\
		\Norm{\y_t^3-\u_t^3} &= \Norm{f_3\circ f_2\circ f_1(\w_t) - \u^3_t} \leq \Norm{f_3\circ f_2 \circ f_1(\w_t) - f_3(\u^2_t)} + \Norm{f_3(\u^2_t) - \u^3_t}\\
		& \leq L_f \Norm{\y_t^2 - \u^2_t}  + \Norm{f_3(\u^2_t) - \u^3_t}\\
		& \leq L_f(L_f \Norm{f_1(\w_t) - \u^1_t} + \Norm{f_2(\u_t^1) - \u_t^2}) +  \Norm{f_3(\u^2_t) - \u^3_t}\\
		& \cdots\\
		\Norm{\y_t^i-\u_t^i} & \leq L_f \Norm{\y_t^{i-1} - \u_t^{i-1}} + \Norm{f_i(\u_t^{i-1}) - \u_t^i}\\
		& \leq \sum_{j=1}^i L_f^{i-j} \Norm{f_j(\u_t^{j-1}) - \u_t^j}
	\end{align*}
	To this end, we can conclude that:
	\begin{align*}
		\Norm{\nabla F(\w_t) -\nabla \widehat{F}_K(\w_t)} & \leq \sum_{i=1}^{K-1} C_i \Norm{f_i(\u^{i-1}_t) - \u_t^i}, 
	\end{align*}
	where $C_i \coloneqq L_f^{K-1} L_J(1+L_f+\dotsc+L_f^{K-i-1})$.
\end{proof}	

\begin{lemma}\label{lem:4} The estimated error of gradient can be bounded as:
	\begin{align*}
			 \E\left[\Norm{\nabla F(\w_t) - \v_t }^2\right] \leq 2K L_f^{2(K-1)}\sum_{i=1}^K \E\left[\Norm{\v_t^i - \nabla f_i(\u_t^{i-1})}^2\right] + 2K L_F^2\sum_{i=1}^{K-1} \E\left[\Norm{f_i(\u_t^{i-1})-\u_t^i}^2\right].
	\end{align*}
\end{lemma}	
\begin{proof}
Consider the definition of $\v_t$:
	\begin{align*}
	 \E\left[\Norm{\v_t - \nabla F(\w_t)}^2\right] \leq 2\underbrace{\E\left[\Norm{\prod_{i=1}^K \v_t^i - \prod_{i=1}^K\nabla f_i(\u_t^{i-1})}^2\right]}_{\coloneqq \diamondsuit} + 2\underbrace{\E\left[\Norm{\nabla \widehat{F}_K(\w_t) - \nabla F(\w_t)}^2\right]}_{\coloneqq \clubsuit}.
	\end{align*}
Based on Lemma~\ref{lem:9}, we have:
\begin{align*}
	\clubsuit &\leq K\sum_{i=1}^{K-1} C_i^2\E\left[\Norm{f_i(\u_t^{i-1})-\u_t^i}^2\right],
\end{align*}
where $C_i \coloneqq L_f^{K-1} L_J\sum_{j=1}^{K-i-1}{L_f^j} \leq L_F$.

The first term $\diamondsuit$ can be upper bounded as:
\begin{align*}
	\diamondsuit &= \E\left[\Norm{\prod_{i=1}^K \v_t^i - \prod_{i=1}^K\nabla f_i(\u_t^{i-1})}^2\right]\\
	&\leq K \E\left[\Norm{\prod_{i=1}^K\nabla f_i(\u_t^{i-1}) - \v_t^1 \prod_{i=2}^K\ \nabla f_i(\u^{i-1}_t) }^2 \right]\\
& \quad\quad + K\E\left[\Norm{\v_t^1\prod_{i=2}^K \nabla f_i(\u^{i-1}_t) -\v_t^1 \v_t^2 \prod_{i=3}^K \nabla f_i(\u^{i-1}_t)}^2\right]\\
& \quad\quad + \dotsc\\
& \quad\quad + K\E\left[\Norm{ \left(\prod_{i=1}^{K-1} \v_t^i\right) \nabla f_K(\u_t^{K-1}) - \prod_{i=1}^K  \v_t^i}^2\right]\\
& \leq K \left(\sum_{i=1}^K L_f^{2(K-1)}\E\left[\Norm{\v_t^i - \nabla f_i(\u_t^{i-1})}^2\right]\right).
\end{align*}
\end{proof}	

\begin{lemma}\label{lem:5}
	The variance of the estimated gradient and function value satisfies the following guarantee:
\begin{align*}
    \E\left[\Norm{\v_t^i - \nabla f_i(\u_t^{i-1})}^2\right] &\leq (1-\beta_{t})\E \left[\Norm{\v_{t-1}^i - \nabla f_i(\u_{t-1}^{i-1})}^2 \right]+ 2\beta_{t}^2 \sigma_J^2 + 2\LL_J^2 \E\left[\Norm{\u_t^{i-1} - \u_{t-1}^{i-1}}^2\right] \\
	\E\left[\Norm{\u_t^i -  f_i(\u_t^{i-1})}^2\right] &\leq (1-\beta_{t}) \E \left[\Norm{\u_{t-1}^i - f_i(\u_{t-1}^{i-1})}^2\right] + 2\beta_{t}^2 \sigma_f^2 + 2\LL_f^2 \E\left[\Norm{\u_t^{i-1} - \u_{t-1}^{i-1}}^2\right]
\end{align*}
\end{lemma}	
\begin{proof}
	Consider the update $\v_t^i = \Pi_{L_f}\left[(1-\beta_{t})\v_{t-1}^i + \beta_{t} \nabla f_i(\u_t^{i-1};\xi_t^i) + (1-\beta_{t})\left(\nabla f_i(\u_t^{i-1};\xi_t^i) - \nabla f_i(\u_{t-1}^{i-1};\xi_t^i)\right)\right]$ and $\Pi_{L_f}[\nabla f_i(\u_t^{i-1})] = \nabla f_i(\u_t^{i-1})$.
	\begin{align*}
		& \E\left[\Norm{\v_t^i - \nabla f_i(\u_t^{i-1})}^2\right]\\
		 = &  \E\left[\Norm{\Pi_{L_f}\left[(1-\beta_{t})\v_{t-1}^i +  \nabla f_i(\u_t^{i-1};\xi_t^i) - (1-\beta_{t})\nabla f_i(\u_{t-1}^{i-1};\xi_t^i)\right]- \Pi_{L_f}[\nabla f_i(\u_t^{i-1})]}^2\right]\\
		 \leq& \E\left[\Norm{(1-\beta_{t})\v_{t-1}^i +  \nabla f_i(\u_t^{i-1};\xi_t^i) - (1-\beta_{t})\nabla f_i(\u_{t-1}^{i-1};\xi_t^i)- \nabla f_i(\u_t^{i-1})}^2\right]\\
		 =& \E\left[\left\|(1-\beta_{t})\left(\v_{t-1}^i - \nabla f_i(\u_{t-1}^{i-1})\right) + \left(\nabla f_i(\u_{t-1}^{i-1}) - \nabla f_i(\u_t^{i-1})  -\left(\nabla f_i(\u_{t-1}^{i-1};\xi_t^i) - \nabla f_i(\u_t^{i-1};\xi_t^i) \right)\right) \right.\right.\\
		&\quad\quad\quad \left.\left. + \beta_{t} \left( \nabla f_i(\u_{t-1}^{i-1};\xi_t^i) - \nabla f_i(\u_{t-1}^{i-1})\right)\right\|^2\right]
	\end{align*}
Note that $\E\left[\nabla f_i(\u_{t-1}^{i-1}) - \nabla f_i(\u_t^{i-1})  -\left(\nabla f_i(\u_{t-1}^{i-1};\xi_t^i) - \nabla f_i(\u_t^{i-1};\xi_t^i)\right)+ \beta_{t} \left( \nabla f_i(\u_{t-1}^{i-1};\xi_t^i) - \nabla f_i(\u_{t-1}^{i-1})\right)\right] = 0$. 
\begin{align*}
	\E\left[\Norm{\v_t^i - \nabla f_i(\u_t^{i-1})}^2\right]	\leq &(1-\beta_{t})\E \Norm{\v_{t-1}^i - \nabla f_i(\u_t^{i-1})}^2 + 2 \beta_{t}^2 \E\left[\Norm{\nabla f_i(\u_{t-1}^{i-1};\xi_t^i) - \nabla f_i(\u_{t-1}^{i-1})}^2\right] \\ & \quad\quad\quad + 2\E\left[\Norm{\nabla f_i(\u_{t-1}^{i-1};\xi_t^i) - \nabla f_i(\u_t^{i-1};\xi_t^i)}^2\right]\\
	 \leq &(1-\beta_{t})\E \Norm{\v_{t-1}^i - \nabla f_i(\u_{t-1}^{i-1})}^2 + 2\beta_{t}^2 \sigma_J^2 + 2\LL_J^2 \E\left[\Norm{\u_t^{i-1} - \u_{t-1}^{i-1}}^2\right]
\end{align*}
Similarly, consider $\u_t^i = (1-\beta_{t})\u_{t-1}^i + \beta_{t} f_i(\u_t^{i-1};\xi_t^i) + (1-\beta_{t})\left(f_i(\u_t^{i-1};\xi_t^i) - f_i(\u_{t-1}^{i-1};\xi_t^i)\right)$:
	\begin{align*}
		& \E\left[\Norm{f_i(\u_t^{i-1}) - \u_t^i}^2\right]\\
		 =&  \E\left[\Norm{(1-\beta_{t})\u_{t-1}^i +  f_i(\u_t^{i-1};\xi_t^i) - (1-\beta_{t}) f_i(\u_{t-1}^{i-1};\xi_t^i)-  f_i(\u_t^{i-1})}^2\right]\\
		 =& \E\left[\left\|(1-\beta_{t})\left(\u_{t-1}^i - f_i(\u_{t-1}^{i-1})\right) + \left( f_i(\u_{t-1}^{i-1}) - f_i(\u_t^{i-1})  -\left(f_i(\u_{t-1}^{i-1};\xi_t^i) -  f_i(\u_t^{i-1};\xi_t^i) \right)\right) \right.\right.\\
		 &\quad\quad\quad \left.\left. + \beta_{t} \left(  f_i(\u_{t-1}^{i-1};\xi_t^i) -  f_i(\u_{t-1}^{i-1})\right)\right\|^2\right]
	\end{align*}
Note that $
	\E\left[f_i(\u_{t-1}^{i-1}) - f_i(\u_t^{i-1})  -\left( f_i(\u_{t-1}^{i-1};\xi_t^i) - f_i(\u_t^{i-1};\xi_t^i)\right)+ \beta_{t} \left(  f_i(\u_{t-1}^{i-1};\xi_t^i) -  f_i(\u_{t-1}^{i-1})\right)\right] = 0 $
\begin{align*}
	& \E\left[\Norm{\u_t^i -  f_i(\u_t^{i-1})}^2\right]\\ \leq &(1-\beta_{t})\E \Norm{\u_{t-1}^i - f_i(\u_{t-1}^{i-1})}^2 + 2 \beta_{t}^2 \E\left[\Norm{f_i(\u_{t-1}^{i-1};\xi_t^i) - f_i(\u_{t-1}^{i-1})}^2\right] + 2\E\left[\Norm{ f_i(\u_{t-1}^{i-1};\xi_t^i) - f_i(\u_t^{i-1};\xi_t^i)}^2\right]\\
	 \leq &(1-\beta_{t}) \E\Norm{\u_{t-1}^i - f_i(\u_{t-1}^{i-1})}^2 + 2\beta_{t}^2 \sigma_f^2 + 2\LL_f^2 \E\left[\Norm{\u_t^{i-1} - \u_{t-1}^{i-1}}^2\right]
\end{align*}
\end{proof}	

\begin{lemma}\label{lem:6}
We have the following guarantee:
	\begin{align*}
    \sum_{i=1}^{K} \E\left[\Norm{\u_{t+1}^{i-1} - \u_{t}^{i-1}}^2\right] \leq \left(\sum_{i=1}^{K}\left(2\LL_f^2\right)^{i-1}\right) \left(\E \left[\eta_{t}^2 \Norm{\v_t}^2\right] + 2\beta_{t+1}^2\sigma_f^2 K   + 2\beta_{t+1}^2 K \sum_{i=1}^{K}  \E\left[\Norm{\u_t^i - f_i(\u_t^{i-1})}^2\right]\right)
\end{align*}
\end{lemma}	
\begin{proof}
First, we discuss two cases:
\begin{itemize}
	\item[1.] ($i=1$): $\E\left[\Norm{\u_{t+1}^{i-1} - \u_{t}^{i-1}}^2\right] = \E\left[\Norm{\w_{t+1} - \w_{t}}^2\right] = \E \left[\eta_{t}^2 \Norm{\v_t}^2\right]$.
	\item[2.] ($2\leq i\leq K$):
	\begin{align*}
		& \E\left[\Norm{\u_{t+1}^{i-1} - \u_{t}^{i-1}}^2\right]\\
		 =&\E\left[\left\|\beta_{t+1}\left(f_{i-1}(\u_{t}^{i-2}) - \u_{t}^{i-1}\right) + (f_{i-1}(\u_{t+1}^{i-2};\xi_{t+1}^{i-1}) - f_{i-1}(\u_{t}^{i-2};\xi_{t+1}^{i-1}))\right.\right.\\
		 &\quad \left.\left.+ \beta_{t+1} \left(f_{i-1} (\u_{t}^{i-2};\xi_{t+1}^{i-1})-f_{i-1}(\u_{t}^{i-2})\right) \right\|^2 \right]\\
		 \leq& 2\E\left[\Norm{\beta_{t+1}\left(f_{i-1}(\u_{t}^{i-2}) - \u_{t}^{i-1}\right) + \beta_{t+1} \left(f_{i-1} (\u_{t}^{i-2};\xi_{t+1}^{i-1})-f_{i-1}(\u_{t}^{i-2})\right)}^2\right] + 2\LL_f^2 \E\left[\Norm{\u_{t+1}^{i-2} - \u_{t}^{i-2}}^2 \right]\\
		 \leq&  2\beta_{t+1}^2 \Norm{f_{i-1}(\u_{t}^{i-2}) - \u_{t}^{i-1}}^2 + 2\beta_{t+1}^2 \sigma_f^2 + 2\LL_f^2 \E\left[\Norm{\u_{t+1}^{i-2} - \u_{t}^{i-2}}^2 \right].
	\end{align*} 
	\end{itemize}
Denote $\Upsilon_t^i \coloneqq \E\left[\Norm{\u_t^i - f_i(\u_t^{i-1})}^2\right]$ and  $\A^{i} \coloneqq \E\left[\Norm{\u_{t+1}^{i-1} - \u_{t}^{i-1}}^2\right]$, we have $A^{i} \leq 2\LL_f^2 A^{i-1} + 2\beta_{t+1}^2 \Upsilon_t^{i-1} + 2\beta_{t+1}^2\sigma_f^2$ for $i\geq 2$. Then we can get:
\begin{align*}
A^{1} &\leq \E \left[\eta_{t}^2 \Norm{\v_t}^2\right] \\
A^{2} &\leq \left(2\LL_f^2\right) \E \left[\eta_{t}^2 \Norm{\v_t}^2\right] +  2\beta_{t+1}^2\sigma_f^2 +2\beta_{t+1}^2 \Upsilon_t^{1} \\
A^{3} &\leq \left(2\LL_f^2\right)^2 \E \left[\eta_{t}^2 \Norm{\v_t}^2\right] +  2\beta_{t+1}^2\sigma_f^2 \left(1+2\LL_f^2\right) + 2\beta_{t+1}^2\left( 2\LL_f^2  \Upsilon_t^{1}+ \Upsilon_t^{2}\right) \\
& \cdots\\
A^{i} &\leq \left(2\LL_f^2\right)^{i-1} \E \left[\eta_{t}^2 \Norm{\v_t}^2\right] + 2\beta_{t+1}^2\sigma_f^2 \sum_{j=1}^{i-1}\left(2\LL_f^2\right)^{j-1}  + 2\beta_{t+1}^2 \sum_{j=1}^{i-1} \left(2\LL_f^2\right)^{i-1-j} \Upsilon_t^{j}\\
&\leq \left(2\LL_f^2\right)^{i-1} \E \left[\eta_{t}^2 \Norm{\v_t}^2\right] + 2\beta_{t+1}^2\sigma_f^2 \sum_{j=1}^{K}\left(2\LL_f^2\right)^{j-1}  + 2\beta_{t+1}^2 \sum_{j=1}^{K} \left(2\LL_f^2\right)^{K-j} \Upsilon_t^{j}
\end{align*}
When summing up, we have:
\begin{align*}
    \sum_{i=1}^{K} A^{i} &\leq \sum_{i=1}^{K}\left(2\LL_f^2\right)^{i-1} \E \left[\eta_{t}^2 \Norm{\v_t}^2\right] + 2\beta_{t+1}^2\sigma_f^2 K \sum_{i=1}^{K}\left(2\LL_f^2\right)^{i-1}  + 2\beta_{t+1}^2K \sum_{i=1}^{K} \left(2\LL_f^2\right)^{K-i} \Upsilon_t^{i} \\
    &\leq \left(\sum_{i=1}^{K}\left(2\LL_f^2\right)^{i-1}\right) \left(\E \left[\eta_{t}^2 \Norm{\v_t}^2\right] + 2\beta_{t+1}^2\sigma_f^2 K   + 2\beta_{t+1}^2K \sum_{i=1}^{K}  \Upsilon_t^{i}\right),
\end{align*}
\end{proof}	

Now we finish the proof of \cref{thm:main}. Denote  that $\Gamma_{t}=F(\w_t)+\frac{1}{c_0 \eta_{t-1}} \sum_{i=1}^{K}\left(\Phi_{t}^{i}+\Upsilon_{t}^{i}\right)$, $\Phi_t^i \coloneqq \Norm{\v_t^i - \nabla f_i(\u_t^{i-1})}^2$ and $\Upsilon_t^i \coloneqq \E\left[\Norm{\u_t^i - f_i(\u_t^{i-1})}^2\right]$.  Based on Lemma~\ref{lem:3}, we have:
	\begin{align*}
	    &\E\left[\Gamma_{t+1}-\Gamma_{t}\right] \\
	    = &\E\left[F(\w_{t+1}) - F(\w_t)+ \frac{1}{c_0 \eta_{t}} \sum_{i=1}^{K}\left(\Phi_{t+1}^{i}+\Upsilon_{t+1}^{i}\right) - \frac{1}{c_0 \eta_{t-1}} \sum_{i=1}^{K}\left(\Phi_{t}^{i}+\Upsilon_{t}^{i}\right)\right]\\
	     \leq & \E\left[\frac{\eta_t}{2}\|\nabla F(\w_{t})-\v_{t}\|^{2}-\frac{\eta_t}{2}\|\nabla F(\w_{t})\|^{2}-\frac{\eta_t}{4}\left\|\v_{t}\right\|^{2} + \frac{1}{c_0 \eta_{t}} \sum_{i=1}^{K}\left(\Phi_{t+1}^{i}+\Upsilon_{t+1}^{i}\right) - \frac{1}{c_0 \eta_{t-1}} \sum_{i=1}^{K}\left(\Phi_{t}^{i}+\Upsilon_{t}^{i}\right)\right]
	\end{align*}
Define constant $L_{1}=\max \left\{1, K L_{f}^{2(K-1)}, K L_{F}^{2}, 2 K\left(\sigma_{J}^{2}+\sigma_{f}^{2}\right),2\left(L_{J}^{2}+L_{f}^{2}\right)\left(1+2 K+2 K \sigma_{f}^{2}\right) \sum_{i=1}^{K}\left(2 \LL_{f}^{2}\right)^{i-1}\right\}$.	By summing up and rearranging, we have:
	\begin{align*}
	 &\E\left[\sum_{t=1}^{T}\frac{\eta_t}{2}\|\nabla F(\w_{t})\|^{2}\right] \\
	 \leq& \E\left[\left(\Gamma_{1}-\Gamma_{T+1}\right) + \sum_{t=1}^{T} \frac{\eta_t}{2}\|\nabla F(\w_{t})-\v_{t}\|^{2}- \sum_{t=1}^{T} \frac{\eta_t}{4}\left\|\v_{t}\right\|^{2} - \sum_{t=1}^{T} \sum_{i=1}^{K}  \frac{1}{c_0 \eta_{t-1}} \left( \Phi_{t}^{i}+\Upsilon_{t}^{i}\right)\right.\\
	 &\quad\quad \left.+ \sum_{t=1}^{T} \sum_{i=1}^{K}  \frac{1}{c_0 \eta_{t}} \left( \Phi_{t+1}^{i}+\Upsilon_{t+1}^{i}\right)\right] \\
	 \leq& \E\left[\left(\Gamma_{1}-\Gamma_{T+1}\right) + L_1 \sum_{t=1}^{T} \sum_{i=1}^{K} \eta_t \left(\Phi_{t}^{i}+\Upsilon_{t}^{i}\right)- \sum_{t=1}^{T} \frac{\eta_t}{4}\left\|\v_{t}\right\|^{2} - \sum_{t=1}^{T} \sum_{i=1}^{K}  \frac{1}{c_0 \eta_{t-1}} \left( \Phi_{t}^{i}+\Upsilon_{t}^{i}\right)  \right.\\
	 & \quad\quad  \left.+\sum_{t=1}^{T} \sum_{i=1}^{K}  \frac{(1-\beta_{t+1})}{c_0 \eta_{t}} \left( \Phi_{t}^{i}+\Upsilon_{t}^{i}\right) + \sum_{t=1}^{T}   \frac{2\beta_{t+1}^2 K \left(\sigma_J^2+\sigma_f^2\right)}{c_0 \eta_{t}}  + \sum_{t=1}^{T} \sum_{i=1}^{K}  \frac{2\left(\LL_J^2+\LL_f^2\right)}{c_0 \eta_{t}}  \Norm{\u_{t+1}^{i-1} - \u_{t}^{i-1}}^2\right]\\
	 \leq& \E\left[\left(\Gamma_{1}-\Gamma_{T+1}\right) + L_1 \sum_{t=1}^{T} \sum_{i=1}^{K} \eta_t \left(\Phi_{t}^{i}+\Upsilon_{t}^{i}\right)- \sum_{t=1}^{T} \frac{\eta_t}{4}\left\|\v_{t}\right\|^{2} - \sum_{t=1}^{T} \sum_{i=1}^{K}  \left( \frac{1}{c_0 \eta_{t-1}}-\frac{1-\beta_{t+1}}{c_0 \eta_{t}}\right) \left( \Phi_{t}^{i}+\Upsilon_{t}^{i}\right) \right.\\
	 & \quad\quad   \left.+ \sum_{t=1}^{T}   \frac{2\beta_{t+1}^2 K \left(\sigma_J^2+\sigma_f^2\right)}{c_0 \eta_{t}} +\sum_{t=1}^{T}  \frac{L_1}{c_0 }\eta_{t} \Norm{\v_t}^2 + \sum_{t=1}^{T}  \frac{L_1 \beta_{t+1}^2}{c_0\eta_{t}} + \sum_{t=1}^{T} \sum_{i=1}^{K}   \frac{L_1 \beta_{t+1}^2}{c_0\eta_{t}} \Upsilon_{t}^{i}\right]\\
	 \leq& \E\left[\left(\Gamma_{1}-\Gamma_{T+1}\right) + \sum_{t=1}^{T} \sum_{i=1}^{K} \left(L_1 \eta_t -  \frac{1}{c_0 \eta_{t-1}}+\frac{L_1 \beta_{t+1}^2}{c_0\eta_{t}}+\frac{(1-\beta_{t+1})}{c_0 \eta_{t}}\right) \left(\Phi_{t}^{i}+\Upsilon_{t}^{i}\right) \right.    \\
	 &\quad\quad \left. - \sum_{t=1}^{T} \left(\frac{1}{4}-\frac{L_1}{c_0 }\right)\eta_t\left\|\v_{t}\right\|^{2} + \sum_{t=1}^{T}  \frac{2 L_1 \beta_{t+1}^2}{c_0\eta_{t}}\right]
	\end{align*}
Set $c_0 = 4L_1$, $\eta_t = \left(a+t \right)^{-1/3}$ and $c=10L_{1}^2$, $a={\left( 20L_1^3\right)}^{3/2}$. Since $(x+y)^{1 / 3}-x^{1 / 3} \leq y x^{-2 / 3} / 3$ and $a \geq 2$, we have:
\begin{align*}
    \frac{1}{\eta_{t}}-\frac{1}{\eta_{t-1}}=\left(a +t \right)^{1 / 3}-\left(a+(t-1) \right)^{1 / 3} \leq \frac{1}{3 \left(a +(t-1) \right)^{2 / 3}} \leq \frac{1}{3 \left(a / 2+t \right)^{2 / 3}} \leq  \frac{2^{2 / 3}}{3 \left(a  +t \right)^{2 / 3}} \leq  \eta_{t}^2 
\end{align*}
Also, we have $\beta_{t+1} = c\eta_t^2 \leq c\eta_0^2 \leq 10L_1^2 a^{-2/3} \leq \left( 2L_1\right)^{-1}$, so:
\begin{align*}
     L_1 \eta_t -  \frac{1}{c_0 \eta_{t-1}}+\frac{L_1 \beta_{t+1}^2}{c_0\eta_{t}}+\frac{(1-\beta_{t+1})}{c_0 \eta_{t}} 
     \leq L_1 \eta_t + \frac{\eta_{t}^2}{c_0} - \frac{ \beta_{t+1}}{2 c_0\eta_{t}} 
     \leq  L_1 \eta_t + \frac{\eta_{t}^2}{4L_1} - \frac{ 5 L_1 \eta_{t}}{4 } \leq 0 
\end{align*}
So, we have:
	\begin{align*}
	 \E\left[\sum_{t=1}^{T}\frac{\eta_t}{2}\|\nabla F(\w_{t})\|^{2}\right] & \E\left[\leq \left(\Gamma_{1}-\Gamma_{T+1}\right) + \sum_{t=1}^{T}  \frac{2 L_1 \beta_{t+1}^2}{c_0\eta_{t}} \right]   \\
	 &\leq \E\left[ F\left(\mathbf{w}_{1}\right)-F_{*} + \frac{1}{c_0 \eta_{0}} \sum_{i=1}^{K}\left(\Phi_{1}^{i}+\Upsilon_{1}^{i}\right) + \sum_{t=1}^{T}  \frac{2 L_1 \beta_{t+1}^2}{c_0\eta_{t}} \right]  \\
	 &\leq \E\left[ \Delta_{F} + \frac{K \left( \sigma_f^2 + \sigma_J^2\right)}{c_0 \eta_0} + 50L_1^4 \sum_{t=1}^{T} \eta_{t}^3 \right] \\
	 &\leq  \Delta_{F} + \frac{K \left( \sigma_f^2 + \sigma_J^2\right)}{c_0 \eta_0} + 50L_1^4 \ln{\left(T+1\right)} 
	\end{align*}
The last inequality holds because $\eta_{t}^3 \leq {\left( a+t\right)}^{-1} \leq {\left(t+1\right)}^{-1}$ and $\sum_{t=1}^{T} {\left(t+1\right)}^{-1} \leq \ln{\left(T+1\right)} $. Since $\eta_t$ is decreasing, we have:
	\begin{align*}
	 \E\left[ \eta_T\sum_{t=1}^{T}\|\nabla F(\w_{t})\|^{2} \right] 
	 &\leq {2\Delta_{F} + \frac{2K \left( \sigma_f^2 + \sigma_J^2\right)}{c_0 \eta_0}} + {100L_1^4 \ln{\left(T+1\right)} }.
	\end{align*}
Similar to the proof of Theorem 1 in STORM \citep{Cutkosky2019MomentumBasedVR}, denote $M=2\Delta_{F} + 2K \left( \sigma_f^2 + \sigma_J^2\right)/\left(c_0 \eta_0\right) + 100L_1^4 \ln{\left(T+1\right)} $. Using Cauchy-Schwarz inequality, we have:
\begin{align*}
     \mathbb{E}\left[\sqrt{\sum_{t=1}^{T}\left\|\nabla F\left(\boldsymbol{\w}_{t}\right)\right\|^{2}}\right]^{2} \leq \mathbb{E}\left[1 / \eta_{T}\right] \mathbb{E}\left[\eta_{T} \sum_{t=1}^{T}\left\|\nabla F\left(\boldsymbol{\w}_{t}\right)\right\|^{2}\right]  \leq \mathbb{E}\left[\frac{M}{\eta_{T}}\right] \leq \mathbb{E}\left[M\left(a+T\right)^{1 / 3}\right],
\end{align*}
which indicate that
\begin{align*}
    \mathbb{E}\left[\sqrt{\sum_{t=1}^{T}\left\|\nabla F\left(\w_{t}\right)\right\|^{2}}\right] \leq \sqrt{M}\left(a+T \right)^{1 / 6}.
\end{align*}
Finally, using Cauchy-Schwarz we have $
\sum_{t=1}^{T}\left\|\nabla F\left(\w_{t}\right)\right\| / T \leq \sqrt{\sum_{t=1}^{T}\left\|\nabla F\left(\w_{t}\right)\right\|^{2}} / \sqrt{T}$ so that:
\begin{align*}
    \mathbb{E}\left[\sum_{t=1}^{T} \frac{\left\|\nabla F\left(\boldsymbol{\w}_{t}\right)\right\|}{T}\right] \leq \frac{\sqrt{M}\left(a+T\right)^{1 / 6}}{\sqrt{T}} \leq \mathcal{O} \left(\frac{a^{1 / 6} \sqrt{ M}}{\sqrt{T}}+\frac{1}{T^{1 / 3}}\right) = \mathcal{O}\left(\frac{1}{T^{1 / 3}} \right),
\end{align*}
where the last inequality is due to $(a+b)^{1 / 3} \leq a^{1 / 3}+b^{1 / 3}$. So, we can achieve the stationary point with $T=\mathcal{O}\left(1 / \epsilon^{3}\right)$.

\section{Proof of Lemma~\ref{lem:222}}
We first calculate the cumulative variance. Denote that $\Upsilon_t^i \coloneqq \Norm{\u_t^i - f_i(\u_t^{i-1})}^2$ and  $\Phi_t^i \coloneqq \Norm{\v_t^i - \nabla f_i(\u_t^{i-1})}^2$. According to Lemma~\ref{lem:5}, we have: 
\begin{align*}
    &\sum_{i=1}^{K} \E\left[\Upsilon_{t+1}^i + \Phi_{t+1}^i \right] \leq (1-\beta)\sum_{i=1}^{K}\E \left[\Upsilon_{t}^i + \Phi_{t}^i\right]+ 2\beta^2 K\left( \sigma_J^2+\sigma_f^2 \right) + 2\left(\LL_J^2+\LL_f^2\right)\sum_{i=1}^{K} \E\left[\Norm{\u_{t+1}^{i-1} - \u_{t}^{i-1}}^2\right] 
\end{align*}
Then, applying Lemma~\ref{lem:6} and setting $\beta \leq \frac{1}{2 L_1}$, we get:
\begin{align*}
   \sum_{i=1}^{K} \E\left[\Upsilon_{t+1}^i + \Phi_{t+1}^i \right]
    \leq &(1-\beta)\sum_{i=1}^{K}\E \left[\Upsilon_{t}^i + \Phi_{t}^i\right]+ \beta^2 L_1 + L_1 \left(\E \left[\eta^2 \Norm{\v_t}^2\right] + \beta^2\  + \beta^2 \sum_{i=1}^{K}  \E\left[\Upsilon_t^i\right]\right)\\ 
    \leq &(1-\frac{\beta}{2})\sum_{i=1}^{K}\E \left[\Upsilon_{t}^i + \Phi_{t}^i\right]+ 2\beta^2 L_1 + L_1 \E \left[\eta^2 \Norm{\v_t}^2\right] 
\end{align*}
By summing up and rearranging, we have:
\begin{align*}
    \frac{\beta}{2}\sum_{t=1}^{T}\sum_{i=1}^{K} \E\left[\Upsilon_{t}^i + \Phi_{t}^i \right] \leq \sum_{i=1}^{K} \E\left[\Upsilon_{1}^i + \Phi_{1}^i \right] + 2\beta^2 L_1 T + L_1 \eta^2 \sum_{t=1}^{T} \E \left[ \Norm{\v_t}^2\right] 
\end{align*}
Denote that $\Gamma_{t}=F(\w_t)+\frac{1}{c_0 \eta} \sum_{i=1}^{K}\left(\Phi_{t}^{i}+\Upsilon_{t}^{i}\right)$. We then try to bound the term $ \sum_{t=1}^T \E\left[\Norm{\v_t}^2\right]$, which is very similar to the proof of \cref{thm:main}:  
	\begin{align*}
	     &\sum_{t=1}^T \E\left[\Norm{\v_t}^2\right]
     \\ \leq	& 2\sum_{t=1}^T \E\left[\Norm{\nabla F(\w_t)}^2\right] +2 \sum_{t=1}^T \E\left[\Norm{\nabla F(\w_t) - \v_t}^2\right] \\
     =	& \frac{4}{\eta} \left(\frac{\eta}{2}\sum_{t=1}^T \E\left[\Norm{\nabla F(\w_t)}^2\right] +\frac{\eta}{2} \sum_{t=1}^T \E\left[\Norm{\nabla F(\w_t) - \v_t}^2\right] \right)\\
	 \leq& \E\left[\frac{4}{\eta}\left(\Gamma_{1}-\Gamma_{T+1}\right) + 4\sum_{t=1}^{T} \|\nabla F(\w_{t})-\v_{t}\|^{2}- \sum_{t=1}^{T} \left\|\v_{t}\right\|^{2} + \frac{4}{\eta}\sum_{t=1}^{T} \sum_{i=1}^{K} \left( \frac{1}{c_0 \eta} \left( \Phi_{t+1}^{i}+\Upsilon_{t+1}^{i}\right) - \frac{1}{c_0 \eta} \left( \Phi_{t}^{i}+\Upsilon_{t}^{i}\right) \right) \right]\\
	 \leq& \E\left[\frac{4}{\eta}\left(\Gamma_{1}-\Gamma_{T+1}\right) + \frac{8L_1}{\eta} \sum_{t=1}^{T} \sum_{i=1}^{K} \eta \left(\Phi_{t}^{i}+\Upsilon_{t}^{i}\right)- \sum_{t=1}^{T} \left\|\v_{t}\right\|^{2} - \sum_{t=1}^{T} \sum_{i=1}^{K}  \frac{4}{c_0 \eta^2} \left( \Phi_{t}^{i}+\Upsilon_{t}^{i}\right) \right. \\
	 & \quad\quad \left. +\sum_{t=1}^{T} \sum_{i=1}^{K}  \frac{4(1-\beta)}{c_0 \eta^2} \left( \Phi_{t}^{i}+\Upsilon_{t}^{i}\right) + \sum_{t=1}^{T}   \frac{8\beta^2 K \left(\sigma_J^2+\sigma_f^2\right)}{c_0 \eta^2}  + \sum_{t=1}^{T} \sum_{i=1}^{K}  \frac{8\left(\LL_J^2+\LL_f^2\right)}{c_0 \eta^2}  \Norm{\u_{t+1}^{i-1} - \u_{t}^{i-1}}^2 \right]\\
	 \leq& \E\left[\frac{4}{\eta}\left(\Gamma_{1}-\Gamma_{T+1}\right) + \frac{8L_1}{\eta} \sum_{t=1}^{T} \sum_{i=1}^{K} \eta \left(\Phi_{t}^{i}+\Upsilon_{t}^{i}\right)- \sum_{t=1}^{T} \left\|\v_{t}\right\|^{2} - \sum_{t=1}^{T} \sum_{i=1}^{K}  \frac{4}{c_0 \eta^2} \left( \Phi_{t}^{i}+\Upsilon_{t}^{i}\right)  \right.\\
	 & \quad\quad \left. +\sum_{t=1}^{T} \sum_{i=1}^{K} \frac{4(1-\beta)}{c_0 \eta^2} \left( \Phi_{t}^{i}+\Upsilon_{t}^{i}\right) + \sum_{t=1}^{T}   \frac{8\beta^2 K \left(\sigma_J^2+\sigma_f^2\right)}{c_0 \eta^2} +\sum_{t=1}^{T}  \frac{4L_1}{c_0 }\Norm{\v_t}^2 + \sum_{t=1}^{T}  \frac{4L_1 \beta^2}{c_0\eta^2} + \sum_{t=1}^{T} \sum_{i=1}^{K}   \frac{L_1 \beta^2}{c_0\eta} \Upsilon_{t}^{i} \right]\\
	 \leq& \E\left[ \frac{4}{\eta}\left(\Gamma_{1}-\Gamma_{T+1}\right) + \sum_{t=1}^{T} \sum_{i=1}^{K} \left(8L_1  -\frac{2\beta}{c_0 \eta^2}\right) \left(\Phi_{t}^{i}+\Upsilon_{t}^{i}\right)- \sum_{t=1}^{T} \left(1-\frac{4 L_1}{c_0 }\right)\left\|\v_{t}\right\|^{2} + \frac{8 L_1 \beta^2 T}{c_0\eta^2} \right]
	\end{align*}
By setting $c_0 = 4L_1$ and $c = 4L_1c_0 = 16 L_1^2$, where $L_1$ is defined in the proof of \cref{thm:main}, we get:
	\begin{align*}
	     \sum_{t=1}^T \E\left[\Norm{\v_t}^2\right] \leq \E\left[ \frac{4}{\eta}\left(\Gamma_{1}-\Gamma_{T+1}\right)  + \frac{8 L_1 \beta^2 T}{c_0\eta^2}\right] \leq \E\left[ \frac{4}{\eta}\left(F(\w_1)-F_{*}\right)  +  \frac{4\sum_{i=1}^{K}\left(\Phi_{1}^{i}+\Upsilon_{1}^{i}\right) }{c_0 \eta^2}+ \frac{8 L_1 \beta^2 T}{c_0\eta^2} \right]
	\end{align*}
When $\tau$ is sampled randomly from $\{1, \ldots, T\}$, we have:
\begin{align*}
    &\sum_{i=1}^K\E\left[\Norm{f_i(\u_{\tau}^{i-1}) - \u_{\tau}^i}^2\right] + \sum_{i=1}^K\E\left[\Norm{\v_{\tau}^i - \nabla f_i(\u_{\tau}^{i-1})}^2\right] \\  
    = &\frac{1}{T} \left(\sum_{i=1}^K\sum_{t=1}^{T}\E\left[\Norm{f_i(\u_t^{i-1}) - \u_t^i}^2\right] + \sum_{i=1}^K\sum_{t=1}^{T}\E\left[\Norm{\v_t^i - \nabla f_i(\u_t^{i-1})}^2\right] \right)\\
    \leq& \E\left[\frac{2}{\beta T} {\sum_{i=1}^{K}\left(\Phi_{1}^{i}+\Upsilon_{1}^{i}\right)}  + 4\beta L_1 + \frac{8L_1\eta}{\beta T}\left(F(\w_1)-F_{*}\right)  + \frac{8L_1\sum_{i=1}^{K}\left(\Phi_{1}^{i}+\Upsilon_{1}^{i}\right) }{c_0 \beta T}+ \frac{16 L_1^2 \beta}{c_0}\right]\\
    \leq& \E\left[\frac{4}{\beta T} {\sum_{i=1}^{K}\left(\Phi_{1}^{i}+\Upsilon_{1}^{i}\right)}  + 8\beta L_1 + \frac{8L_1\eta}{\beta T}\left(F(\w_1)-F_{*}\right) \right]
\end{align*}

Let's consider the first stage, in which $\sum_{i=1}^K \E\left[\Norm{f_i(\u_1^{i-1}) - \u_1^i}^2 +\Norm{\v_1^i - \nabla f_i(\u_1^{i-1})}^2\right] \leq K\left( \sigma_{f}^2+\sigma_{J}^2\right)$. Note that in below the numerical subscripts denote the stage index $\{1, \ldots, S\}$. With the parameter $\beta_{1} =\frac{1}{2L_1}$, and  $T_{1} = \max \left\{ 2 \sqrt{2L_1}\Delta_{F}, 4 L_1 K\left( \sigma_{f}^2+\sigma_{J}^2\right)\right\}$, we have:
\begin{align*}
     &\sum_{i=1}^K\E\left[\Norm{f_i(\u_{1}^{i-1}) - \u_{1}^i}^2\right] + \sum_{i=1}^K\E\left[\Norm{\v_{1}^i - \nabla f_i(\u_{1}^{i-1})}^2\right] 
    \\  \leq &\frac{4}{\beta_{1} T_{1}} K \left( \sigma_{f}^2+\sigma_{J}^2\right)  + 8\beta_{1}L_1   + \frac{8L_1\eta\Delta_{F}}{\beta_1 T_1} 
    \\  \leq &8 L_1
    \\ = & \mu\epsilon_{1}
\end{align*}

Also, according to the PL condition, we have:
\begin{align*}
    \begin{aligned} 
    \E\left[ F\left(\w_{1}\right)-F_{*}\right] &\leq  \E\left[\frac{\left\|\nabla F\left(\w_{1}\right)\right\|^{2}}{2 \mu}\right] \\
    &\leq \frac{4L_1\Delta_{F}}{\mu T_{1} \sqrt{\beta_{1}}}+\frac{4L_1 K \left( \sigma_{f}^2+\sigma_{J}^2\right)}{\mu  \beta_{1} T_{1}}+\frac{8\beta_{1} L^2}{\mu}  \\
    &\leq  \frac{8L_1}{\mu} \\
    &= \epsilon_{1}
    \end{aligned}
\end{align*}

Starting form the second stage, we would prove by induction. Suppose at stage $s-1$, we have $F\left(\w_{s-1}\right)-F_{*} \leq \epsilon_{s-1}$, and $\sum_{i=1}^K\E\left[\Norm{f_i(\u_{s-1}^{i-1}) - \u_{s-1}^i}^2\right] + \sum_{i=1}^K\E\left[\Norm{\v_{s-1}^i - \nabla f_i(\u_{s-1}^{i-1})}^2\right] \leq \mu\epsilon_{s-1}$. Then at $s$ stage, with $\beta_{s}=\frac{\mu \epsilon_{s-1}}{64L_1^2}$ and $T_{s}=\max \left\{\frac{2048L_1^3}{\mu \epsilon_{s-1}}, \frac{256L_1^2}{u^{3 / 2} \sqrt{\epsilon_{s-1}}}\right\}$, we have:
\begin{align*}
     &\sum_{i=1}^K\E\left[\Norm{f_i(\u_{s}^{i-1}) - \u_{s}^i}^2\right] + \sum_{i=1}^K\E\left[\Norm{\v_{s}^i - \nabla f_i(\u_{s}^{i-1})}^2\right] 
    \\ \leq& \frac{4\mu \epsilon_{s-1}}{\beta_s T_s}   + 8\beta_s L_1 + \frac{8L_1\eta_s \epsilon_{s-1}}{\beta_s T_s}
    \\  \leq& \frac{\mu\epsilon_{s-1}}{8 } + \frac{\mu\epsilon_{s-1}}{8} +  \frac{\mu\epsilon_{s-1}}{8} 
    \\ \leq & \frac{\mu\epsilon_{s-1}}{2} 
    \\ = & \mu\epsilon_{s}
\end{align*}
The output of stage $s$ also satisfies the following guarantee:
\begin{align*}
    \E\left[ F\left(\w_{s}\right)-F_{*} \right] &\leq \E\left[ \frac{\left\|\nabla F\left(\w_{s}\right)\right\|^{2}}{2 \mu} \right] \leq \frac{1}{2\mu T_s}\sum_{t=1}^{T_s} \E\left[\Norm{\nabla F(\w_t)}^2\right]\\
    & \leq \frac{\left(\Gamma_{1}-\Gamma_{T+1}\right)}{\mu \eta_s T_s} + \frac{2L_1 \beta_s^2}{\mu c_o \eta_s^2}\\
    & \leq \frac{4L_1\left(F\left(\w_{s-1}\right)-F_{*}\right)}{\mu \sqrt{\beta_s} T_s} + \frac{4L_1 \mu \epsilon_{s-1}}{\mu \beta_s T_s} + \frac{8L_1^2 \beta_s}{\mu}\\
    & \leq \frac{\epsilon_{s-1}}{8 } + \frac{\epsilon_{s-1}}{8} +  \frac{\epsilon_{s-1}}{8}  \\
    &\leq \frac{\epsilon_{s-1}}{2} = \epsilon_{s}
\end{align*}
Combining two cases, we have proved that $\sum_{i=1}^K\E\left[\Norm{f_i(\u_{s}^{i-1}) - \u_{s}^i}^2\right] + \sum_{i=1}^K\E\left[\Norm{\v_{s}^i - \nabla f_i(\u_{s-1}^{i-1})}^2\right] \leq \mu\epsilon_{s}$ and $F\left(\w_{s}\right)-F_{*} \leq \epsilon_{s}$, where $\epsilon_{s} = \epsilon_{1} / 2^{s-1}$. Let $L_2 = 64L_1^2$, we have the results in \cref{lem:222}

\section{Proof of Theorem~\ref{thm:main2}}
We have proved that $F\left(\w_{s}\right)-F_{*} \leq \epsilon_{s}$ in Lemma~\ref{lem:222}. That is to say, $F\left(\w_{S}\right)-F_{*} \leq \epsilon$ when $S = \log _{2}\left(\frac{2\epsilon_{1}}{\epsilon}\right) = \log _{2}\left(\frac{L}{\mu \epsilon}\right)$, and the sample complexity until this stage is computed as:
\begin{align*}
T_{1}+\sum_{s=2}^{S} T_{s} &=\mathcal{O}\left(\sum_{s=2}^{S} T_{s}\right) \\ &=\mathcal{O}\left(\sum_{s=2}^{S} \frac{L_2^{3/2}}{\mu \epsilon_{s-1}}+\frac{L_2}{\mu^{3 / 2} \sqrt{\epsilon_{s-1}}}\right) \\ & \leq \mathcal{O}\left(\frac{L_2^{3/2}}{\mu \epsilon}, \frac{L_2}{\mu^{3 / 2} \sqrt{\epsilon}}\right) \stackrel{\mu \geq \epsilon}{\leq} \mathcal{O}\left(\frac{1}{\mu \epsilon}\right)
\end{align*}

\section{Proof of Theorem~\ref{thm:main3}}
When $F(\w)$ is convex, we define $\hat{F}(\w) = F(\w) + \frac{\mu}{2}\|\w\|^2$. We know that $\hat{F}(\w)$ is $\mu$-strongly convex, which implies $\mu$-PL condition. In Theorem~\ref{thm:main2}, we have proved: for any $\delta > 0$,  there exist $T=\mathcal{O}\left(\frac{1}{\mu \delta}\right)$ such that $\hat{F}(\w_{T}) - \hat{F}_{*} \leq \delta$.  It indicates that $F(\w_{T}) - F_{*} \leq \delta + \frac{\mu}{2} \|\w_{*}\|^2 -  \frac{\mu}{2} \|\w_{T}\|^2 \leq \delta + \frac{\mu}{2}D$. For any $\epsilon > 0$, if we choose $\mu = \frac{\epsilon}{D}$ and $\delta = \frac{\epsilon}{2}$, we get $F(\w_{T}) - F_{*} \leq \epsilon$, for some $T=\mathcal{O}\left(\frac{2D}{\epsilon^2}\right)=\mathcal{O}\left(\frac{1}{\epsilon^2}\right)$.

\section{Proof of Theorem~\ref{thm:T3}}
Note that since the norm of estimated gradient $\left\| \v_{t} \right\|$ is bounded, the value of the learning rate scaling factor $\mathbf{c}=1 /\left(\sqrt{\mathbf{h}_{t}}+\delta\right)$ presented in (\ref{rule2}) is also upper bounded and lower bounded, which can be presented as $c_{l} \leq\left\|\mathbf{c}\right\|_{\infty} \leq c_{u}$. With this property, We can introduce the variant version of Lemma~\ref{lem:3}, which can also be found in Lemma 3 in \citep{guo2022stochastic}.
\begin{lemma}\label{lem:starter_}
    For $\mathbf{w}_{t+1}=\mathbf{w}_{t}-\tilde{\eta}_{t} \mathbf{v}_{t}$, with $\eta_t c_{l} \leq \tilde{\eta}_{t} \leq \eta_t c_{u} $ and $ \eta_t L_F\leq {c_l}/{2 c_u^2 }$, we have following guarantee:
	\begin{align*}
		F(\w_{t+1}) \leq F(\w_t) + \frac{\eta_t c_{u}}{2}\Norm{\nabla F(\w_t) - \v_t}^2 - \frac{\eta_t c_{l}}{2}\Norm{\nabla F(\w_t)}^2 - \frac{\eta_t c_{l}}{4}\Norm{\v_t}^2.
	\end{align*}
\end{lemma}	

\begin{proof}
Due to the smoothness of function $F$ and the definition of $\w_{t+1}$, we have:

\begin{align*}
F\left(\w_{t+1}\right) & \leq F\left(\w_{t}\right)+\left\langle\nabla F(\w_{t}), \w_{t+1}-\w_{t}\right\rangle+\frac{L_{F}}{2}\left\|\w_{t+1}-\w_{t}\right\|^{2} \\ &=F(\w_{t})-\tilde{\eta}_{t}\left\langle\nabla F(\w_{t}), \v_{t}\right\rangle+\frac{\tilde{\eta}_{t}^{2} L_{F}}{2}\|\v_{t}\|^{2}\\ & \leq F(\w_{t})-\tilde{\eta}_{t}\langle\nabla F(\w_{t}), \v_{t}\rangle+\frac{\tilde{\eta}_{t}}{2}\|\nabla F(\w_{t})\|^{2}+\frac{\tilde{\eta}_{t}}{2}\left\|\v_{t}\right\|^{2}-\frac{\tilde{\eta}_{t}}{2}\|\nabla F(\w_{t})\|^{2}-\frac{\tilde{\eta}_{t}}{2}\left\|\v_{t}\right\|^{2} +\frac{\eta_t c_{l}}{4}\left\|\v_{t}\right\|^{2}\\ &=F(\w_{t})+\frac{\tilde{\eta}_{t}}{2}\|\nabla F(\w_{t})-\v_{t}\|^{2}-\frac{\tilde{\eta}_{t}}{2}\|\nabla F(\w_{t})\|^{2}-\frac{\eta_t c_{l}}{4}\left\|\v_{t}\right\|^{2} \\
&\leq F(\w_{t})+\frac{\eta_t c_{u}}{2}\|\nabla F(\w_{t})-\v_{t}\|^{2}-\frac{\eta_t c_{l}}{2}\|\nabla F(\w_{t})\|^{2}-\frac{\eta_t c_{l}}{4}\left\|\v_{t}\right\|^{2} 
\end{align*}
\end{proof}
We set
\begin{align*}
    L_{3}=\max \left\{1, c_{u} K L_{f}^{2(K-1)}, c_{u} K L_{F}^{2}, 2 K\left(\sigma_{J}^{2}+\sigma_{f}^{2}\right),2\left(1+\frac{1}{c_l}\right)\left(L_{f}^{2}+L_{f}^{2}\right)\left(1+2 K+2 K \sigma_{f}^{2}\right) \sum_{i=1}^{K}\left(2 \LL_{f}^{2}\right)^{i-1}\right\}.
\end{align*}
Similar to the proof of \cref{thm:main}, we denote that $\Gamma_{t}=F(\w_t)+\frac{1}{c_0 \eta_{t-1}} \sum_{i=1}^{K}\left(\Phi_{t}^{i}+\Upsilon_{t}^{i}\right)$, $\Phi_t^i \coloneqq \Norm{\v_t^i - \nabla f_i(\u_t^{i-1})}^2$ and $\Upsilon_t^i \coloneqq \E\left[\Norm{\u_t^i - f_i(\u_t^{i-1})}^2\right]$. Based on Lemma~\ref{lem:starter_}, we have:
	\begin{align*}
	    &\E\left[\Gamma_{t+1}-\Gamma_{t} \right]\\
	    = &\E\left[F(\w_{t+1}) - F(\w_t)+ \frac{1}{c_0 \eta_{t}} \sum_{i=1}^{K}\left(\Phi_{t+1}^{i}+\Upsilon_{t+1}^{i}\right) - \frac{1}{c_0 \eta_{t-1}} \sum_{i=1}^{K}\left(\Phi_{t}^{i}+\Upsilon_{t}^{i}\right)\right]\\
	     \leq &\E\left[\frac{\eta_t c_{u}}{2}\|\nabla F(\w_{t})-\v_{t}\|^{2}-\frac{\eta_t c_{l}}{2}\|\nabla F(\w_{t})\|^{2}-\frac{\eta_t c_{l}}{4}\left\|\v_{t}\right\|^{2} + \frac{1}{c_0 \eta_{t}} \sum_{i=1}^{K}\left(\Phi_{t+1}^{i}+\Upsilon_{t+1}^{i}\right) - \frac{1}{c_0 \eta_{t-1}} \sum_{i=1}^{K}\left(\Phi_{t}^{i}+\Upsilon_{t}^{i}\right)\right]
	\end{align*}
	By summing up and rearranging, we have:
	\begin{align*}
	 &\E\left[\sum_{t=1}^{T}\frac{\eta_t c_{l}}{2}\|\nabla F(\w_{t})\|^{2}\right] \\
	 \leq&\E\left[ \left(\Gamma_{1}-\Gamma_{T+1}\right) + \sum_{t=1}^{T} \frac{\eta_t c_{u}}{2}\|\nabla F(\w_{t})-\v_{t}\|^{2}- \sum_{t=1}^{T} \frac{\eta_t c_{l}}{4}\left\|\v_{t}\right\|^{2}\right. \\ &\quad\left.+ \sum_{t=1}^{T} \sum_{i=1}^{K} \left( \frac{1}{c_0 \eta_{t}} \left( \Phi_{t+1}^{i}+\Upsilon_{t+1}^{i}\right) - \frac{1}{c_0 \eta_{t-1}} \left( \Phi_{t}^{i}+\Upsilon_{t}^{i}\right) \right)\right]\\
	 \leq& \E\left[\left(\Gamma_{1}-\Gamma_{T+1}\right) + \sum_{t=1}^{T} \sum_{i=1}^{K} \left(L_3 \eta_t -  \frac{1}{c_0 \eta_{t-1}}+\frac{L_3 \beta_{t+1}^2}{c_0\eta_{t}}+\frac{(1-\beta_{t+1})}{c_0 \eta_{t}}\right) \left(\Phi_{t}^{i}+\Upsilon_{t}^{i}\right)- \sum_{t=1}^{T} \left(\frac{c_{l}}{4}-\frac{L_3 c_{l}}{c_0 }\right)\eta_t\left\|\v_{t}\right\|^{2}\right.\\
	 &\quad \left.+ \sum_{t=1}^{T}  \frac{2 L_3 \beta_{t+1}^2}{c_0\eta_{t}}  \right] 
	\end{align*}
Set $c_0 = 4L_3$, $\eta_t = \left(a+t \right)^{-1/3}$ and $c=10L_{3}^2$, $a={\left( 20L_3^3\right)}^{3/2}$, we have:
	\begin{align*}
	 \E\left[\sum_{t=1}^{T}\frac{\eta_t c_{l}}{2}\|\nabla F(\w_{t})\|^{2}\right] \leq  \Delta_{F} + \frac{K \left( \sigma_f^2 + \sigma_J^2\right)}{c_0 \eta_0} + 50L_3^4 \ln{\left(T+1\right)}
	\end{align*}
Since $\eta_t$ is decreasing, we have:
	\begin{align*}
	 \eta_T\sum_{t=1}^{T}\|\nabla F(\w_{t})\|^{2}  
	 \leq \frac{1}{c_l}\left(2\Delta_{F} +\frac{2K \left( \sigma_f^2 + \sigma_J^2\right)}{c_0 \eta_0} + 100L_3^4 \ln{\left(T+1\right)} \right).
	\end{align*}
Very similar to the proof of \cref{thm:main}, set $M=\frac{1}{c_l}\left(2\Delta_{F} +\frac{2K \left( \sigma_f^2 + \sigma_J^2\right)}{c_0 \eta_0} + 100L_3^4 \ln{\left(T+1\right)} \right)$ and we have:
\begin{align*}
    \mathbb{E}\left[\sum_{t=1}^{T} \frac{\left\|\nabla F\left(\boldsymbol{\w}_{t}\right)\right\|}{T}\right] \leq \frac{\sqrt{M}\left(a+T\right)^{1 / 6}}{\sqrt{T}} \leq \mathcal{O} \left(\frac{a^{1 / 6} \sqrt{ M}}{\sqrt{T}}+\frac{1}{T^{1 / 3}}\right) = \mathcal{O}\left(\frac{1}{T^{1 / 3}} \right).
\end{align*}
So, Adaptive SMVR enjoys the same $\mathcal{O}\left(1 / \epsilon^{3}\right)$ sample complexity as the SMVR method.
\end{document}